\documentclass[twoside,11pt]{article}

\usepackage[utf8]{inputenc} 
\usepackage[T1]{fontenc}    
\usepackage{amsmath}        
\usepackage{amsthm}
\usepackage{mathtools}
\usepackage{caption}
\usepackage{subcaption}

\PassOptionsToPackage{hyphens}{url}
\usepackage{jmlr2e}
\hypersetup{hidelinks}
\usepackage{url}            
\usepackage{booktabs}       
\usepackage{multicol}
\usepackage{multirow}
\usepackage{microtype}      

\usepackage{enumitem}
\setlist{nosep}
\usepackage{xspace}
\usepackage{lipsum}
\usepackage{xfrac}
\usepackage{scalerel}

\usepackage{tikz}
\usetikzlibrary{shapes.geometric, shapes, positioning, arrows.meta}
\tikzset{>={Stealth[scale=1.2]}}
\tikzset{Ferrers bullet/.style={circle, fill, inner sep=0pt, minimum width=2mm, draw=none}}
\AtBeginDocument{\colorlet{normalcolor}{.}}

\usepackage{pgfplots}
\usepackage{pgfplotstable}
\usepackage{etoolbox} 
\usepackage{enumitem}
\setlist{nosep}

\usepackage[cal=dutchcal, scr=boondox]{mathalfa}
\usepackage{dsfont}
\usepackage{bm} 

\usepackage[algo2e, plain, vlined, linesnumbered]{algorithm2e}
\SetArgSty{textnormal}
\makeatletter
\renewcommand{\@algocf@capt@plain}{above}
\makeatother
\DontPrintSemicolon
\SetKwInOut{Input}{input}
\SetKwInOut{Output}{output}
\newcommand{\ie}{\emph{i.e., \xspace}}
\newcommand{\eg}{\emph{e.g. \xspace}}

\newcommand{\eqdef}{
    \overset{{\mbox{\tiny \textup{def}}}}{=}
}
\let\oldforall=\forall
\renewcommand{\forall}{\hspace{1pt}\oldforall\hspace{1pt}}
\let\oldexists=\exists
\renewcommand{\exists}{\hspace{1pt}\oldexists\hspace{1pt}}

\DeclareMathOperator{\vcdim}{VCdim}

\DeclareMathOperator*{\argmax}{argmax}
\DeclareMathOperator*{\argmin}{argmin}

\DeclareMathOperator*{\Expectation}{\mathds{E}}
\newcommand{\exv}[2]{\Expectation_{#1}\left[#2\right]}
\newcommand{\Id}[1]{\mathds{1}\hspace{-2.5pt}\left[#1\right]}
\newcommand{\IdSign}[1]{\mathds{S}\hspace{-2.5pt}\left[#1\right]}

\newcommand{\pr}[1]{\left( #1 \right)}
\newcommand{\floor}[1]{\left\lfloor #1 \right\rfloor}
\newcommand{\ceil}[1]{\left\lceil #1 \right\rceil}
\newcommand{\cb}[1]{\left\{ #1 \right\}}
\newcommand{\abs}[1]{\left|#1\right|}
\newcommand{\norm}[1]{\left\lVert#1\right\rVert}

\newcommand{\reals}{\mathds{R}}
\newcommand{\naturals}{\mathds{N}}

\renewcommand{\a}{\mathbf{a}}

\newcommand{\x}{\mathbf{x}}
\newcommand{\y}{\mathbf{y}}
\newcommand{\z}{\mathbf{z}}

\newcommand{\A}{\mathcal{A}}
\newcommand{\C}{\mathcal{C}}
\newcommand{\D}{\mathcal{D}}

\renewcommand{\H}{{\mathcal{H}}}

\renewcommand{\O}{\mathscr{O}}
\renewcommand{\P}{\mathcal{P}}
\newcommand{\Q}{\mathcal{Q}}
\newcommand{\R}{\mathcal{R}}
\renewcommand{\S}{\mathcal{S}}
\newcommand{\X}{\mathcal{X}}
\newcommand{\Y}{\mathcal{Y}}

\newcommand{\Ll}{L_{T_l}}
\newcommand{\Lr}{L_{T_r}}

\newcommand{\Kb}{\mathbf{K}}
\newcommand{\Khat}{\widehat{K}}
\newcommand{\Kbhat}{\mathbf{\Khat}}

\newcommand{\Nb}{\mathbf{N}}
\newcommand{\Nhat}{\widehat{N}}
\newcommand{\Nbhat}{\mathbf{\Nhat}}

\newcommand{\Ob}{\mathbf{O}}
\newcommand{\Obbar}{\widebar{\Ob}}
\newcommand{\Obar}{\widebar{O}}
\newcommand{\Ohat}{\widehat{O}}
\newcommand{\Obhat}{\mathbf{\Ohat}}

\newcommand{\parti}[1]{\bar{#1}}
\def\plambda{\parti{\lambda}}
\def\palpha{\parti{\alpha}}
\def\pbeta{\parti{\beta}}
\def\pgamma{\parti{\gamma}}

\newcommand{\floormtwo}{{\floor{\frac{m}{2}}}}

\newcommand{\smallstirling}[2]{\ensuremath{{\begin{Bsmallmatrix}
  #1 \\ #2
\end{Bsmallmatrix}}}}
\newcommand{\stirling}[2]{\ensuremath{{\begin{Bmatrix}
  #1 \\ #2
\end{Bmatrix}}}}

\definecolor{color2}{rgb}{0.5,0.7,0.9}
\definecolor{color1}{rgb}{0.15,0.35,0.75}
\definecolor{color3}{rgb}{0.8,0.2,0.3}

\makeatletter
\let\save@mathaccent\mathaccent
\newcommand*\if@single[3]{%
  \setbox0\hbox{${\mathaccent"0362{#1}}^H$}%
  \setbox2\hbox{${\mathaccent"0362{\kern0pt#1}}^H$}%
  \ifdim\ht0=\ht2 #3\else #2\fi
  }
\newcommand*\rel@kern[1]{\kern#1\dimexpr\macc@kerna}
\newcommand*\widebar[1]{\@ifnextchar^{{\wide@bar{#1}{0}}}{\wide@bar{#1}{1}}}
\newcommand*\wide@bar[2]{\if@single{#1}{\wide@bar@{#1}{#2}{1}}{\wide@bar@{#1}{#2}{2}}}
\newcommand*\wide@bar@[3]{%
  \begingroup
  \def\mathaccent##1##2{%
    \let\mathaccent\save@mathaccent
    \if#32 \let\macc@nucleus\first@char \fi
    \setbox\z@\hbox{$\macc@style{\macc@nucleus}_{}$}%
    \setbox\tw@\hbox{$\macc@style{\macc@nucleus}{}_{}$}%
    \dimen@\wd\tw@
    \advance\dimen@-\wd\z@
    \divide\dimen@ 3
    \@tempdima\wd\tw@
    \advance\@tempdima-\scriptspace
    \divide\@tempdima 10
    \advance\dimen@-\@tempdima
    \ifdim\dimen@>\z@ \dimen@0pt\fi
    \rel@kern{0.6}\kern-\dimen@
    \if#31
      \overline{\rel@kern{-0.6}\kern\dimen@\macc@nucleus\rel@kern{0.4}\kern\dimen@}%
      \advance\dimen@0.4\dimexpr\macc@kerna
      \let\final@kern#2%
      \ifdim\dimen@<\z@ \let\final@kern1\fi
      \if\final@kern1 \kern-\dimen@\fi
    \else
      \overline{\rel@kern{-0.6}\kern\dimen@#1}%
    \fi
  }%
  \macc@depth\@ne
  \let\math@bgroup\@empty \let\math@egroup\macc@set@skewchar
  \mathsurround\z@ \frozen@everymath{\mathgroup\macc@group\relax}%
  \macc@set@skewchar\relax
  \let\mathaccentV\macc@nested@a
  \if#31
    \macc@nested@a\relax111{#1}%
  \else
    \def\gobble@till@marker##1\endmarker{}%
    \futurelet\first@char\gobble@till@marker#1\endmarker
    \ifcat\noexpand\first@char A\else
      \def\first@char{}%
    \fi
    \macc@nested@a\relax111{\first@char}%
  \fi
  \endgroup
}
\makeatother

\usepackage{lastpage}
\jmlrheading{1}{2022}{1-\pageref{LastPage}}{10/22}{N/A}{leboeuf22a}{Jean-Samuel Leboeuf and Frédéric LeBlanc and Mario Marchand}

\ShortHeadings{Generalization Properties of Decision Trees}{Leboeuf, LeBlanc and Marchand}
\firstpageno{1}

\begin{document}


\title{Generalization Properties of Decision Trees\\ on Real-valued and Categorical Features}


\author{\name Jean-Samuel Leboeuf \email jean-samuel.leboeuf.1@ulaval.ca \\
        \addr Department of Computer Science\\
        Universit\'e Laval\\
        Qu\'ebec, QC, Canada
        \AND
        \name Fr\'ed\'eric LeBlanc \email frederic.leblanc.12@ulaval.ca\\
        \addr Faculty of Education\\
        Universit\'e Laval\\
        Qu\'ebec, QC, Canada
        \AND
        \name Mario Marchand \email mario.marchand@ift.ulaval.ca \\
        \addr Department of Computer Science\\
        Universit\'e Laval\\
        Qu\'ebec, QC, Canada
        }

\editor{TBD}

\maketitle

\begin{abstract}
We revisit binary decision trees from the perspective of partitions of the data.
We introduce the notion of partitioning function, and we relate it to the growth function and to the VC dimension.
We consider three types of features: real-valued, categorical ordinal and categorical nominal, with different split rules for each.
For each feature type, we upper bound the partitioning function of the class of decision stumps before extending the bounds to the class of general decision tree (of any fixed structure) using a recursive approach.
Using these new results, we are able to find the exact VC dimension of decision stumps on examples of $\ell$ real-valued features, which is given by the largest integer $d$ such that $2\ell \ge \binom{d}{\floor{\frac{d}{2}}}$.
Furthermore, we show that the VC dimension of a binary tree structure with $L_T$ leaves on examples of $\ell$ real-valued features is in $O(L_T \log(L_T\ell))$.
Finally, we elaborate a pruning algorithm based on these results that performs better than the cost-complexity and reduced-error pruning algorithms on a number of data sets, with the advantage that no cross-validation is required.
\end{abstract}


\section{Introduction}
\label{sec:introduction}

The recent years have been marked by the advent of deep neural networks (DNN), whose prowesses have impressed the general public.
DNNs perform exceptionally well on ``natural'' data, such as pictures, time series and natural language, but require an equally exceptional amount of data and compute power.
One might even argue that they perform unreasonably well, as they still defy the classical generalization theories \citep{zhang2021understanding}.
The absence of a theory explaining how and why DNNs generalize hinders greatly their interpretability, \ie the possibility to explain their predictions.

In constrast, decision trees complement nicely DNNs as they excel on unorganized and disparate data, requiring only a small sample to be trained, and being inherently interpretable by the way they try to mimic human decision processes.
These advantages make them particularly useful in areas where ethics, safety, health or well-being is at stake, like in the insurance and financial sectors.
In fact, decision trees at the basis of several popular and powerful boosting algorithms, such as those provided by XGBoost \citep{chen2016xgboost}, LightGBM \citep{ke2017lightgbm}, and CatBoost \citep{dorogush2018catboost}.

Nevertheless, despite decision trees being one of the oldest learning models \citep{breiman1984classification,quinlan1986induction} and being commonly used, their generalization properties are still poorly characterized.
To the best of our knowledge, there currently exists no upper bound on the combinatorial properties that are the VC dimension and the growth function of binary decision trees with real-valued and/or categorical features that share a common structure.
Hence, the main goal of this paper is to provide such bounds\footnote{This paper is an extended version of \citet{leboeuf2020decision}'s conference paper titled \textit{Decision trees as partitioning machines to
characterize their generalization properties
} published at NeurIPS 2020. While this previous work focused only on real-valued features, the present paper adds further analyses for decision trees on two types of categorical features as well as on a mixture of feature types. We also adapt and improve the proposed pruning algorithm based on these results.}.

To do so, we introduce the idea of a \emph{realizable partition} and define the notion of partitioning function, a concept closely related to the growth function and the VC dimension.
We proceed to bound tightly the partitioning function of the class of decision stumps that can be constructed from a set of real-valued or categorical features.
This leads us to find an \emph{exact} expression of its VC dimension for the case of real-valued features---a long-standing open problem.
We then extend our bound of the partitioning function to general binary decision tree structures applied on real-valued features, categorical features or a mixture of both, from which we derive the asymptotic behavior of the VC dimension of a tree with $L_T$ leaves.

Taking a step further, as a secondary goal, we also investigate and support the relevance of these generalization properties of decision trees by considering some practical applications.

Due to their expressive power and the traditionally greedy way they are trained, decision trees are prone to overfitting.
To handle this problem, algorithms usually apply a pruning step, where unnecessary branches are replaced or deleted.
Most popular pruning algorithms are guided by some heuristics that are not fully theoretically justified, and resort to practical techniques, such as cross-validation, which often increase the running time and impairs generalization when the number of training examples is small.

As an alternative, we propose to use the results derived here and combine them with a structural risk minimization generalization bound for VC classes to produce a new theoretically-sound pruning algorithm.
The idea of using generalization bounds to guide pruning is not new; for instance, \cite{drouin2019interpretable} successfully prune decision trees learned on a genomic data set by optimizing a sample-compression-based bound.
However, this technique is limited by the fact that it is task-specific due to the need for a reconstruction function.
Even as early as \citeyear{kearns1998fast}, \citeauthor{kearns1998fast} proposed a pruning algorithm based on a bound for VC classes---with the obvious caveat that previous to our work, it could only be used for trees learned on examples with binary features exclusively.
Thanks to our work on the growth function of decision trees, VC-bound-based pruning algorithms can now be used for any type or combinations of features.

In our experiments, we compare two bound-based pruning algorithms with the popular \citet{breiman1984classification} CART's cost-complexity pruning and \citet{quinlan1986induction}'s reduced-error pruning algorithms on 25 data sets of various sizes, with different number and types of features, and different number of classes.
Our findings show that both bound-based algorithms clearly outperform the other two, suggesting that the VC dimension is indeed an appropriate quantifier of the generalization capacity of a binary decision tree.

The paper is divided as follows.
In Section~\ref{sec:related_work}, we discuss the work related to the analysis of the VC dimension and the complexity of decision trees and how our work inserts itself naturally among those.
Next, in Section~\ref{sec:mathematical_setting}, we present the mathematical setting, the definitions, and we justify in detail our choices regarding the treatment of categorical features.
Section~\ref{sec:paritions_as_a_framework} is dedicated to introducing the new mathematical framework based on partitions of the data.
Then, in Section~\ref{sec:analysis_decision_stumps}, we sequentially apply this framework to the case of a decision stumps (a single node with two leaves) on real-valued features, ordinal features and nominal features.
We also state the exact VC dimension of decision stumps on real-valued feature examples.
Section~\ref{sec:analysis_of_decision_trees} respectively extends these results to general decision trees before combining them into a theorem that applies to decision trees on a mixture of feature types.
Finally, in Section~\ref{sec:experiments}, we develop a novel pruning algorithm that we compare to other pruning approaches in an experiment.

\section{Related Work}
\label{sec:related_work}

For the case of binary features, \cite{simon1991vapnik} has shown that the VC dimension of binary decision trees of rank at most $r$ with $\ell$ features is given by $\sum_{i=0}^r \tbinom{\ell}{i}$.
However, the set of decision trees with rank at most $r$ includes multiple tree structures that clearly possess different individual generalization properties.
Later, \cite{mansour1997pessimistic} claimed that the VC dimension of a binary decision tree with $N$ nodes and $\ell$ \emph{binary} features is between $\Omega(N)$ and $O(N \log \ell)$, but did not provide the proof. 
Then, \cite{maimon2002improving} provided a bound on the VC dimension of oblivious decision trees, which are trees such that all the nodes of a given layer make a split on the same feature.

In 2009, \citeauthor{aslan2009calculating} proposed an exhaustive search algorithm to compute the VC dimension of decision trees with binary features. Results were obtained for all trees of height at most 4. Then, they used a regression approach to \emph{estimate} the VC dimension of a tree as a function of the number of features, the number of nodes, and the VC dimension of the left and right subtrees.

More recently, \cite{yildiz2015vc} found the exact VC dimension of the class of decision stumps (\ie trees with a single node) that can be constructed from a set of $\ell$ \emph{binary} features, which is given by $\floor{\log_2(\ell+1)} + 1$, and proved that this is a lower bound for the VC dimension of decision stumps with $\ell$ real-valued features.
They then used these expressions as base cases to develop a recursive lower bound on the VC dimension of decision trees with more than one node.
Furthermore, they were able to extend their results to trees with more than two branches.
However, they did not provide an upper bound for the VC dimension of decision trees.

On a related topic, \cite{gey2018vapnik} found the exact VC dimension of axis-parallel cuts on $\ell$ real-valued features, which are a kind of one-sided decision stumps.
They showed that the VC dimension of this class of functions is given by the largest integer $d$ such that $\ell \geq \tbinom{d}{\floor{\frac{d}{2}}}$.
As a corollary of their result, one has that the largest integer $d$ that satisfies $2\ell \geq \tbinom{d}{\floor{\frac{d}{2}}}$ is an upper bound for the VC dimension of a decision stump, an observation they however do not make.
Using a completely different approach, we here show that this bound is in fact exact.
We discuss the difference between our results and theirs in Section~\ref{ssec:decision_stumps_on_rl_feat}.

Our work distinguishes itself from previous work by providing an upper bound for the VC dimension of any binary decision tree class on real-valued features, categorical features, or a mixture of both.
Our framework also extends to the multiclass setting, and we show that bound-based pruning algorithms are viable and useful alternatives to other popular methods.

\section{Mathematical Setting}
\label{sec:mathematical_setting}

In this paper, an example is a vector denoted by $\x$ and lives in a domain space $\X$ that will be specified when needed.
The entries of $\x$ are called \emph{features}, and are indexed by an integer $i$ such that the $i$-th feature of $\x$ is denoted by $x_i$.
We consider the multiclass setting with $y \in [n]$ where $n$ is the number of classes.
Here, we use the notation $[n]$ to represent the set $\cb{1,\dots,n}$.
We also reserve the symbols $S$ for a sample of examples, and $m$ for the number of examples in $S$.

Machine learning algorithms often deal with various types of data.
In this work, we consider two main types of features: \emph{real-valued features}, which take value in $\reals$, and \emph{categorical feature} (also called \emph{discrete-valued features}), which take value in some finite set whose elements are called \emph{categories}.
Furthermore, we subdivide categorical features into two subtypes, depending on whether or not the categories are endowed with a natural order.
In the first case, we say that the features are \emph{ordinal} and otherwise they are \emph{nominal}.

Having made these distinctions, we will denote by $\ell$ the number of real-valued features, by $\omega$ the number of ordinal features, and by $\nu$ the number of nominal features.
We further refine these definitions by introducing the vectors $\Ob \in \naturals^\omega$ and $\Nb\in\naturals^\nu$, which we call the ordinal and nominal \emph{feature landscapes}, where the components $O_i$ of $\Ob$ (and $N_i$ of $\Nb$) indicate the number of categories that the $i$-th ordinal feature can take.
For convenience, without loss of generality, we hereafter always assume that the $i$-th ordinal feature takes value in $[O_i]$ and that the $i$-th nominal feature takes value in $[N_i]$; if they do not, simply relabel the categories.
With these tools in hand, we can write quite generally the domain space as $\X = \reals^\ell \times \bigtimes_{i=1}^\omega [O_i] \times \bigtimes_{i=1}^\nu [N_i]$, where the cross denotes the Cartesian product.

Our analysis of decision trees requires us to make the distinction between three kinds a mathematical objects with different properties to avoid confusion in the terminology.
First, we define a \emph{tree} as a graph with two types of nodes: \emph{internal nodes}, which have two or more children, and \emph{leaves} which do not have any children.
For simplicity, internal nodes will be referred to as \emph{nodes} (in contrast to leaves).
Second, we say that a \emph{decision tree} is a function $t : \X \to \Y$ defined on a given tree graph, where each leaf is associated with a class label while each node is endowed with a function, called a \emph{decision rule} $\phi$, which redirects incoming examples to its children.
Third, the space of all allowed decision rules, called a \emph{rule set} $\Phi$, accompanied by some fixed tree graph, defines a class of functions called a \emph{decision tree class} that we always denote by the symbol $T$.
In that case, the number of nodes and leaves and the underlying graph are fixed, but the parameters of the decision rules at the nodes and the class labels at the leaves are free parameters.
When all nodes have exactly two children, we say that we have a \emph{binary tree}, in which case we label (arbitrarily) the two children as the \emph{left} and \emph{right} subtrees, denoted respectively by $T_l$ and $T_r$.
For reasons justified later, the present work deals exclusively with binary decision trees.
Moreover, we say that a decision tree with only one node and two leaves is a \emph{decision stump}.

The rule set is arbitrary, and different choices will lead to different learning algorithms.
In fact, the rule set is directly related to the expressiveness of decision tree classes; therefore it is important to define clearly the decision rules that will be treated in the present work.
In the literature, the rule set varies according to the type of features at hand.
However, most decision rules are assumed to be feature-aligned, \ie to depend on a single feature at a time.
This is partly due to the fact that many learning problems act on a mixture of real-valued and categorical features, and it is not clear how one should mix different feature types to produce a useful (and interpretable) decision rule.

For real-valued and ordinal features, there exists an obvious ``natural'' decision rule that we refer to as the \emph{threshold split}, which consists to redirect the examples to the left and right subtrees according to whether they satisfy $x_i \le \theta$ or not, $x_i$ is the $i$-th feature of $\x$ and $\theta$ is some threshold taken in $\reals$ for real-valued features or in $[O_i-1]$ for ordinal features (we do not allow $\theta=O_i$ to be a valid choice as it does not split the sample into two parts).
Clearly, the threshold split decision rule as described is only valid for binary decision trees.
Of course, one could consider more complex rules where we could have multiple thresholds in order to redirect the examples to more than two subtrees.
However, this type of rules is completely equivalent to simply having deeper binary trees with more internal nodes.
Therefore, it does not make sense to consider non-binary trees for real-valued and ordinal features.

The situation for categorical features is more subtle.
We have identified multiple strategies employed for nominal features depending on the algorithm; here are some of the most popular solutions encountered.

The CART algorithm, proposed by \citet{breiman1984classification}, proceeds via the simple \emph{global comparison} rule set, which considers rules of the form: if $x_i\in A$, where $A \subseteq [N_i]$, then $\x$ is sent to the left subtree, and is sent to the right one otherwise.
In that case, there are $2^{N_i}-2$ possible rules to consider for each feature $i$.
This becomes quickly unmanageable for large number of categories, and is much more prone to overfitting due to the large number of rules.

Another rule set is the \emph{$N_i$-ary tree} approach, as adopted by \citet{quinlan1986induction}'s C4.5 algorithm, where each node has $N_i$ subtrees, one for each possible category.
This choice also increases the chances of overfitting due to the fact that very few examples are redirected in each subtree when $N_i$ is large.

On the other hand, the LightGBM library \citep{ke2017lightgbm} employs the \emph{ordinal conversion} technique, where one assigns an arbitrary order to the categories so that the natural rule for ordinal features can apply (the actual implementation of LightGBM is more complex than a simple ordinal conversion, but the intricacies of the algorithm are beyond the scope of this paper).

Finally, the CatBoost algorithm \citep{dorogush2018catboost} offers multiple ways to deal with categorical features, with more flexibility than the previous algorithms (as its name suggests).
Two main approaches are proposed to handle categories.
The first and simplest is the \emph{binary conversion} (also known as \emph{one-hot encoding}), which transforms the single feature $i$ into $N_i$ binary features (for binary features, the rule set is trivial as there is only two rules: for $x_i \in [2]$, $\x$ is redirected to the left subtree if $x_i = 1$ and to the right subtree otherwise, or the other way around).
The second method proposed is the ordinal conversion, available in several variations. 
A big weakness of the ordinal conversion is that it is difficult to choose the best order to give to the categories: indeed, there are $N_i!$ such possible orders---a number that grows very quickly.
Nevertheless, there are some empirical methods that have been proposed that allow to cleverly order the categories using some statistics of the sample.
In fact, when the classification problem is binary, it is possible to find in polynomial time the order that achieves the best gain for a single split (\citeauthor{breiman1984classification}, \citeyear{breiman1984classification}, see also Section 9.2.4 of \citeauthor{hastie2009elements}, \citeyear{hastie2009elements}).
Unfortunately, this result breaks down when the output is multiclass, even though many approximations have been proposed and implemented into CatBoost.
From a practical point of view, these approaches have their merits, but they are hard to handle from a theoretical one, since the order chosen is data-dependent---a property incompatible with the assumptions made in this work.


In light of this discussion, in addition to the threshold split rule set for real-valued and ordinal features, we consider in this paper the \emph{unitary comparison} decision rules for nominal features, where decision rules send to either the left or the right subtree every examples $\x$ satisfying $x_i = C$ for some category $C \in [N_i]$ and the rest to the other one.
Note that for all practical purposes, this rule set is equivalent to the binary conversion of CatBoost; it is however different on a theoretical level, since the binary features created must be mutually exclusive (\ie one and only one of the $N_i$ binary features can be a 1), while binary features in general need not be mutually exclusive.

In accordance with our choice of decision rules, we are here only concerned with binary decision trees.
Hence, let us now provide a mathematical implementation of the rule sets that will be used throughout this work.
For convenience, we define the signed indicator function $\mathds{S}$ such that $\IdSign{A} \eqdef \Id{A} - \Id{\neg A}$, so that rather than being 0 when the argument is false, it is equal to $-1$.
This allows us to route the examples to the subtrees according to the sign of the decision rule.
With this notation, we define formally the threshold split rule set (for real-valued features) as $\Phi = \cb{\phi(\x) = s\cdot \IdSign{x_i \le \theta } : s \in \cb{\pm 1}, i \in [\ell], \theta \in \reals}$, where $s$ is a sign parameter that gives the flexibility to redirect the example to the left or the right subtree as we wish.
The rule set for ordinal features is the same with $i$ taken in $[\omega]$ and $\theta$ taken in $[O_i-1]$ instead.
For nominal features, we write $\Phi = \cb{ \phi(\x) = s \cdot \IdSign{x_i = C} : s \in \cb{\pm 1}, i \in [\nu], C \in [N_i] }$.

Equipped with such a rule set $\Phi \subset \cb{-1, 1}^\X$, and given a tree class $T$ endowed with $\Phi$, we can formalize the output $t(\x)$ of a decision tree $t\in T$ on an example $\x$ recursively as follows.
\begin{definition}[Output of a binary decision tree]
\label{def:binary_decision_tree}
If the tree $t$ is a leaf, the output $t(\x)$, on example $\x$, is given by the class label associated with the leaf. 
Otherwise, if the tree $t$ is rooted at a node having a left subtree $t_l$, a right subtree $t_r$, and a decision rule $\phi \in \Phi$, then the output $t(\x)$ is given by 
\begin{equation*}
  t(\x) \eqdef \left\{ \!\!\!
    \begin{array}{cl}
      t_l(\x) & \text{if } \phi(\x) = 1\\
      t_r(\x) & \text{if } \phi(\x) = -1.
    \end{array} \right.
\end{equation*}
\end{definition}


The predictive power of a decision tree class is heavily influenced by the tree graph, but also by the choice of rule sets.
In this paper, we aim to quantify the expressivity of decision trees from the perspective of VC theory.
In this framework, the key combinatorial tools to evaluate are the growth function, and, for the binary classification problem, the VC dimension, both of which are defined below.

\begin{definition}[Growth function]
\label{def:growth_function}
We define the \emph{growth function $\tau_H$} of a hypothesis class $H \subseteq [n]^\X$ as the largest number of distinct functions that $H$ can realize on a sample $S$ of $m$ examples, \ie
\begin{equation*}
  \tau_H(m) \eqdef \max_{S:\left|S\right|=m} \left| \cb{ h|_S : h \in H} \right|,
\end{equation*}
where $h|_S \eqdef (h(\x_1), h(\x_2), \dots, h(\x_m))$, for $\x_j \in S$, is the restriction of $h$ to $S$.
\end{definition}
Note that in general the growth function can depend on other parameters than the number of examples such as the number of features and the nature of the domain space.
However, to alleviate the notation, we will explicitly write the dependencies only when confusion is possible.
In the binary classification setting, one can use the VC dimension to bound the growth function with Sauer-Shelah's Lemma.
\begin{definition}[VC dimension]
\label{def:VC dimension}
Let $H$ be a class of binary classifiers.
A sample $S=\cb{\x_1, \dots, \x_m}$ is shattered by $H$ iff all possible Boolean functions on $S$ can be realized by functions $h\in H$. 
The \textit{VC dimension} of $H$, $\vcdim H$, is defined as the maximal cardinality of a set $S$ shattered by $H$. 
In particular, the VC dimension of $H$ is the largest integer $d$ such that $\tau_H(d) = 2^d$.
\end{definition}

\section{Partitions as a Framework}
\label{sec:paritions_as_a_framework}

The main goal of this paper is to upper bound the generalization properties of decision trees with a fixed structure.
Our research have led us to abandon the concept of decision trees as a class of functions and to instead inspect them through the lens of partition theory.
These intrinsically combinatorial objects, defined formally below, have the advantageous property of being agnostic to the labeling of the tree leaves and combine nicely with the recursive nature of decision trees.

\begin{definition}[Partition]
\label{def:partition}
Given some finite set $A$, an \emph{$a$-partition} $\palpha(A)$ is a set of $a\in\mathds{N}$ disjoint and non-empty subsets $\alpha_j \subseteq A$, called \emph{parts}, such that $\bigcup_{j=1}^a \alpha_j = A$.
\end{definition}
The total number of $a$-partitions that exist on a set of $m$ elements is given by the Stirling number of the second kind, denoted $\smallstirling{m}{a}$ \citep{graham1989concrete}.
In particular, for $m\geq 1$, we have that $\smallstirling{m}{1} = 1$ and $\smallstirling{m}{2} = 2^{m-1}-1$.
For the purpose of this work, we also define $\smallstirling{m}{2}_k$ as the number of 2-partitions of $m$ with at least a part of size $k$, which is equal to $\binom{m}{k}$ when $k \neq \frac{m}{2}$ and is equal to $\frac{1}{2}\binom{m}{\frac{m}{2}}$ when $k = \frac{m}{2}$.

As a convention, we will use capital Latin letters for sets of integers or examples (such as $S$ and $A$), overlined lower case Greek letters for partitions (such as $\palpha$, $\pbeta$, $\pgamma$ and $\plambda$), and script Latin letters for sets of partitions (such as $\P$, $\Q$, $\R$ and $\S$).

To see how binary decision trees can be represented as some kind of ``partitioning machines'', consider a set $S$ of examples that is sieved through some tree so that all examples are distributed among the leaves.

Setting aside the labels, the set of non-empty leaves exactly satisfies the definition of a partition of $S$.
Then, when the leaves are labeled using one of the $n$ classes, if some leaves share the same label, the union of the identically labeled leaves forms a new single part, still disjoint from the other parts.
Hence, we end up with another valid partition of $S$ containing an equal or smaller number of parts.
We say that any partition of the data that can be realized through this process is a \emph{realizable partition} of $S$.
The following definition formally defines this concept. 
\begin{definition}[Realizable partition]
\label{def:realizable_partition}
Let $T$ be a binary decision tree class (of a fixed structure).
An $a$-partition $\palpha(S)$ of a sample $S$ is \emph{realizable} by $T$ iff there exists some tree $t \in T$ such that
\begin{itemize}[topsep=0pt]
    \item For all parts $\alpha_j \in \palpha(S)$, and for all examples $\x_1, \x_2 \in \alpha_j$, we have that $t(\x_1) = t(\x_2)$;
    \item For all distinct $\alpha_j, \alpha_k \in \palpha(S)$, and for all $\x_1 \in \alpha_j, \x_2 \in \alpha_k$, we have that $t(\x_1) \neq t(\x_2)$.
\end{itemize}
\end{definition}
Hence, the set $\P^a_T(S)$ of all distinct $a$-partitions a tree class $T$ can realize on $S$ is obtained by considering all possible rules that we can use at each node of $T$ and all possible labelings in $[a]$ that we can assign to the leaves of $T$. 
We can link the growth function $\tau_T(m)$ of $T$ to $|\P^a_T(S)|$ as follows. 
Given some realizable $a$-partition $\palpha(S)$, we have $n$ choices of label for any one part, then we have $n-1$ choices for the next one, because assigning it the same label would effectively create an $(a-1)$-partition.
This process continues until no more parts or labels are left.
Therefore, for any $a$-partition with $a \in [n]$, one can produce $(n)_a$ distinct functions, where $(n)_a \eqdef n (n-1) \cdots (n-a+1)$ is the falling factorial.
Consequently, the growth function $\tau_T(m)$ can be written as
\begin{align}\label{eq:growth_func_partitions_set}
    \tau_T(m) = \max_{S:\abs{S}=m} \,\sum_{a=1}^{\mathclap{\min \cb{m,n,L_T}}}\,\,\, (n)_a \abs{\P^a_T(S)},
\end{align}
where $L_T$ denotes the number of leaves of the tree class $T$ and where the sum goes up to $\min \cb{m,n,L_T}$ so that every term in the sum stays well defined.
This hints us to an important property of a tree class, that we call the \emph{partitioning functions}.

\begin{definition}[Partitioning functions]
\label{def:partitioning_function}
The \emph{$a$-partitioning function $\pi^a_T$} of a tree class $T$ is defined as the largest number of distinct $a$-partitions that $T$ can realize on a sample $S$ of $m$ examples, \ie
\begin{equation*}\label{eq:def_a-partitioning_func}
  \pi^a_T(m) \eqdef \max_{S:\abs{S}=m} \abs{ \P^a_T(S) }\, .
\end{equation*}
Moreover, we refer to the set of all possible $a$-partitioning functions of $T$ for all integers $a \in [L_T]$, with $L_T$ being the number of leaves of $T$, as the \emph{partitioning functions} of the tree class $T$. 
\end{definition}
Since the maximum of a sum is less than or equal to the sum of the maxima of its summands, we have that 
\begin{equation}\label{eq:ub_growth_func}
    \tau_T(m) \leq \sum_{a=1}^{L_T}(n)_a \pi^a_T(m).
\end{equation}
Moreover, we have equality whenever $n=2$ or $L_T=2$ since the first term of the sum of Equation~\eqref{eq:growth_func_partitions_set} is always $\abs{\P_T^1(S)}=1$ for any $S$ with $m>0$.
As for the growth function, the partitioning functions are dependent on the number of features and on the feature landscapes of the domain space, but we explicitly write these dependencies only when it is relevant.

Having linked the partitioning functions to the growth function, we can relate them to the VC dimension in the following way.
In the binary classification setting, each of the realizable partitions yields exactly 2 distinct functions by labeling the parts with the two available classes.
Thus, $T$ can realize $2^m$ binary functions iff $T$ realizes every 1- and 2-partition on $S$ (of which there are respectively exactly 1 and $2^{m-1}-1$).
On the other hand, Definition~\ref{def:VC dimension} implies that a tree $T$ shatters a sample $S$ iff it can realize all $2^m$ functions on $S$.
Therefore, since any tree class $T$ can realize the single 1-partition, we have that $T$ shatters a sample $S$ iff it realizes every 2-partition on $S$.
Hence, the VC dimension of any tree class $T$ having at least one internal node is given by
\begin{equation}\label{eq:vcdim_def_partition_func}
    \vcdim T = \max \cb{d: \pi^2_T(d) = 2^{d-1} - 1}.
\end{equation}

\section{Analysis of Decision Stumps}
\label{sec:analysis_decision_stumps}

We first inspect the theoretical aspect of the base case, that is, a tree with a single node and two leaves, also known as a decision stump.
This will prove useful in the next section when we tackle the analysis of any general decision tree.

The goal of this section is to obtain an expression for the 2-partitioning function of decision stump.
In order to do this, we consider three separate cases.
First, we examine the case when the examples are composed of only real-valued features.
We are able to derive a tight upper bound on the 2-partitioning function, which in turn allows us to recover the \emph{exact} VC dimension of decision stumps on real-valued features, a long standing open problem.
Second, we inspect the case when features are ordinal.
The analysis, while similar to real-valued features, is complicated by the discrete nature of the domain space.
Nevertheless, we can find a greedy algorithm which returns a tight upper bound on the 2-partitioning function.
Finally, we deal with nominal features, where we employ a different approach to also provide an upper bound.

\subsection{The Class of Decision Stumps on Real-valued Features}
\label{ssec:decision_stumps_on_rl_feat}

In this section, the domain space $\X$ is assumed to contain $\ell$ real-valued features.
As such, the rule set is the threshold split $\Phi \eqdef \cb{\phi(\x) = s\cdot\IdSign{x_i \le \theta } : s \in \cb{\pm 1}, i \in [\ell], \theta \in \reals}$.
As the class $T$ of decision stumps has only one root node and two leaves, the only non-trivial $a$-partitioning function of $T$ is $\pi^2_T(m)$, the maximum number of 2-partitions achievable on $m$ examples.
The following theorem gives a tight upper bound of this quantity.

\begin{theorem}[Upper bound on the 2-partitioning function of decision stumps on real-valued features]\label{thm:ub_partitioning_func_decision_stumps_rl_feat}
Let $T$ be the class of decision stumps on examples of $\ell$ real-valued features. Then
\begin{align}\label{eq:decision_stump_partitioning_function_rl_feat}
    \pi^2_T(m, \ell) \leq \frac{1}{2} \sum_{k=1}^{m-1} \min \cb{ 2\ell, \binom{m}{k} },
\end{align}
and this is an equality for $2\ell \leq m$, for $2\ell \geq \binom{m}{\floor{\frac{m}{2}}}$, and for $1 \leq m \leq 7$.
\end{theorem}
The proof is presented in Appendix~\ref{app:proof_stump_rl}, and relies on a permutation representation of the decision rules as well as on graph-theoretical arguments to prove the equality for $2\ell \geq \binom{m}{\floor{\frac{m}{2}}}$.
We conjecture that the bound is an equality for all $m$, but it is not clear how this can be shown.

Let us compare the theorem with the \emph{trivial bound} that is often used for decision stumps. The trivial bound consists in exploiting the fact that for each available feature, a stump can realize at most $m - 1$ different $2$-partitions, which gives $\pi^2_T(m) \le \ell(m-1) = (1/2)\sum_{k=1}^{m-1} 2\ell$. This yields $\tau_T(m) \le 2 + 2\ell (m-1)$ for the growth function (in binary classification). Comparing the trivial bound with Theorem~\ref{thm:ub_partitioning_func_decision_stumps_rl_feat}, we see that the trivial bound becomes an equality for $2\ell \le m$ and becomes strictly larger than the bound of Theorem~\ref{thm:ub_partitioning_func_decision_stumps_rl_feat} for $2\ell > m$. Also, the trivial bound exceeds the bound of Theorem~\ref{thm:ub_partitioning_func_decision_stumps_rl_feat} by $\ell(m-1) + 1 - 2^{m-1}$ for $2\ell \ge \binom{m}{\lfloor m/2\rfloor}$ --- a gap which is at least 
\begin{equation*}
    \frac{1}{2}\sum_{k=1}^{m-1} \left[ \binom{m}{\lfloor \frac{m}{2}\rfloor} - \binom{m}{k}\right].
\end{equation*}
Each term of the sum being positive, the trivial bound can be \emph{much larger} than the proposed bound.

Now that we have a tight upper bound on the 2-partitioning function of decision stumps, it is straightforward to find the \emph{exact} VC dimension of decision stumps.

\begin{corollary}[VC dimension of decision stumps on real-valued features]
\label{thm:vcdim_stump}
Let $T$ be the hypothesis class of decision stumps on examples of $\ell$ real-valued features.
Then, the VC dimension of $T$ is implicitly given by solving for the largest integer $d$ that satisfies $\displaystyle 2\ell \geq \tbinom{d}{\floor{\frac{d}{2}}}$.
\end{corollary}

\begin{proof}
According to Equation~\eqref{eq:vcdim_def_partition_func}, the VC dimension is given by the largest integer $m$ such that $\pi^2_T(m) = 2^{m-1}-1$.
Theorem~\ref{thm:ub_partitioning_func_decision_stumps_rl_feat} gives an upper bound on the 2-partitioning function of decision stumps.
Notice that for $2\ell \geq \binom{m}{\floor{\frac{m}{2}}}$, this theorem simplifies to $\pi^2_T(m) = 2^{m-1}-1$, while for $2\ell < \binom{m}{\floor{\frac{m}{2}}}$, it implies $\pi^2_T(m) < 2^{m-1}-1$. Since $\binom{m}{\floor{\frac{m}{2}}}$ is a strictly increasing function of $m$, the largest integer $m$ such that $\pi^2_T(m) = 2^{m-1}-1$ is the largest $m$ that satisfies $2\ell \geq \binom{m}{\floor{\frac{m}{2}}}$.
\end{proof}

\noindent\textbf{Remark }
Let us mention the similarities with the result of \citet{gey2018vapnik}, where they find the VC dimension of axis-parallel cuts.
They define axis-parallel cuts as some kind of asymmetric stump, where the left leaf is always labeled~0 and the right leaf is always labeled~1.
The main difference is that the VC dimension of axis-parallel cuts is given by the largest integer $d$ that satisfies $\ell \geq \binom{d}{\floor{\frac{d}{2}}}$ (the factor 2 is absent).
Their approach is a set theoretic one, and we expect it would be hard to extend it to decision stumps, specifically for the case where $m$ is odd.
Moreover, the graph theoretic approach used here (see Appendix~\ref{app:proof_part_3_vcdim_stump}) allows us to recover a tight upper bound for the growth function (and therefore applies to the multiclass setting), while theirs does not.

\subsection{The Class of Decision Stumps on Ordinal Features}
\label{ssec:decision_stumps_on_ordinal_feat}

We here consider the domain space $\X$ is to be composed exclusively of $\omega$ ordinal features satisfying some feature landscape $\Ob$.
We assume that for every feature $i$, $x_i$ takes value in $[O_i]$ for $O_i \in \naturals$; if it does not, simply relabel the categories.
As discussed previously, the rule set is again the threshold split $\Phi \eqdef \cb{\phi(\x) = s \cdot \IdSign{x_i \le \theta } : s \in \cb{\pm 1}, i \in [\omega], \theta \in [O_i-1]}$.
The discrete nature of the features introduces some complexity not present when dealing with real-valued features.
To be able to leverage this fact in the partitioning function, we develop a greedy procedure attribution procedure of the available decision rules.
The output of the algorithm is summarized in the following theorem.

\begin{theorem}[Upper bound on the $2$-partitioning function of decision stumps on ordinal features]
\label{thm:ub_partitioning_func_decision_stumps_ordinal_feat}
Let $T$ be the class of decision stump on $\omega$ ordinal features and denote the feature landscape by $\Ob$.
Let $\Obbar^0$ be the feature landscape conjugate of $\Ob$ with components defined by $\Obar_C^0 = \sum_{i=1}^\omega \Id{O_i-1 \ge C}$, where $C\in[\Omega-1]$ for $\Omega \eqdef \max_i O_i$.
Then, the number of realizable 2-partitions is bounded by
\begin{equation*}
    \pi^2_T(m, \Ob) \le \sum_{k=1}^{\floormtwo} R_k,
\end{equation*}
with
\begin{equation*}
    R_k = \begin{cases}
        \min\cb{\Obar^{k-1}_1 + \Obar^{k-1}_2, \binom{m}{k}} & \text{if } k < \frac{m}{2}\\
        \min\cb{\Obar^{k-1}_1, \frac{1}{2}\binom{m}{k}} & \text{if } k = \frac{m}{2},
    \end{cases}
\end{equation*}
where, for $k<\frac{m}{2}$,
\begin{equation*}
    \Obar^k_C =
    \begin{cases}
        \Obar^{k-1}_{C} & \text{if } C < {\Gamma_k},\\
        \Obar^{k-1}_{\Gamma_k} + \Obar^{k-1}_{{\Gamma_k}+1} + \Obar^{k-1}_{{\Gamma_k}+2} - R_k & \text{if } C = {\Gamma_k}\\
        \Obar^{k-1}_{C+2} & \text{otherwise,}
    \end{cases}
\end{equation*}
with $\Obar^k_\Omega = \Obar^k_{\Omega+1} \eqdef 0$ and $\Gamma_k \eqdef \max \cb{ 1 \le C \le \Omega-1 : \Obar^{k-1}_C + \Obar^{k-1}_{C+1}\ge R_k }$. (No update rules for $\Obbar^k$ are needed when $k=\frac{m}{2}$.)
\end{theorem}
The proof of this theorem is presented in Appendix~\ref{app:proof_stump_ordinal}.

Note that this theorem is an upper bound and not an equality, because the greedy procedure from which this theorem stems does not verify if the construction can be realized by an actual sample; indeed, it could be possible that for some feature landscape, there is no sample that achieves this upper bound.
However, equality would follow as a consequence of the equality of the bound for stumps on real-valued features.
Therefore, we can conjecture with the same degree of certainty that this new theorem is actually an equality.

Note that this result can be simplified at the cost of tightness to the more manageable form
\begin{equation*}
    \pi^2_T(m, \Ob) \le \frac{1}{2} \sum_{k=1}^{m-1} \min\cb{2\omega, \binom{m}{k}}.
\end{equation*}
This expression is identical to the bound for decision stumps on real-valued features (where $\ell$ is changed for $\omega$), as one should expect since ordinal features can only be less expressive than real-valued ones.

\subsection{The Class of Decision Stumps on Nominal Features}
\label{ssec:decision_stumps_on_nominal_feat}

In this section, we handle the case where the domain space consists exclusively in $\nu$ nominal features with feature landscape $\Nb$.
For convenience, we again assume that every feature $i$ takes value in $[N_i]$ for $N_i \in \naturals$.
The rule set is the unitary comparison, which can be written formally as $\Phi \eqdef \cb{ \phi(\x) = s \cdot \IdSign{x_i = C} : s \in \cb{\pm 1}, i \in [\nu], C \in [N_i] }$.
Then, the following theorem on the expressivity of a decision stump holds.

\begin{theorem}[Upper bound on the 2-partitioning function of decision stumps on nominal features]
\label{thm:ub_partitioning_func_decision_stumps_nominal_feat}
Let $T$ be the class of decision stumps on $\nu$ nominal features and denote the feature landscape by $\Nb$.
Then, the of realizable 2-partitions is bounded by
\begin{align*}
    \pi_T^2(m) \le \min \cb{
        \sum_{i=1}^\nu \min \bigg\{ R_{N_i}, \sum_{k=1}^{\floormtwo} R_{N_i,k}\bigg\},
        \sum_{k=1}^{\floormtwo} \min \bigg\{ \stirling{m}{2}_k\!\!,\, \sum_{i=1}^\nu R_{N_i,k}\bigg\}
    },
\end{align*}
where
\begin{equation*}
    R_{N} \eqdef \begin{cases}
        N-1 & \text{if } N = 1 \text{ or } 2,\\
        \min\cb{N, m} & \text{otherwise,}
    \end{cases}
\end{equation*}
and
\begin{equation*}
    R_{N,k} \eqdef \begin{cases}
    1 & \text{if } \frac{m}{k}=2 \text{ and } N \neq 1 ,\\
    N-1 & \text{if } \frac{m}{k} > N,\\
    \floor{\frac{m}{k}} & \text{if } \frac{m}{k} \le N.
    \end{cases}
\end{equation*}
\end{theorem}

The proof of this theorem relies on combinatorial arguments and on the union bound.
It is presented in Appendix~\ref{app:proof_stump_nominal}.
Two corollaries easier to understand can be obtain from this bound depending on which part of the minimum is kept.
The first result that can be derived is
\begin{equation*}
    \pi_T^2(m) \le \sum_{i=1}^\nu R_{N_i} \le \sum_{i=1}^\nu N_i.
\end{equation*}
where $\sum_{i=1}^\nu N_i$ is the total number of categories across all features.
This expression highlights that one cannot have more partitions than the total number of categories available.
The second result that can be recovered can be simplified to the form
\begin{equation*} \label{eq:R_k_stump}
    \pi_T^2(m) \le \frac{1}{2}\sum_{k=1}^{m-1} \min\cb{ \binom{m}{k},\, \nu\floor{\frac{m}{\min\cb{k,m-k}}}}.
\end{equation*}
This latter corollary is reminiscent of the bound for stump on $\ell$ real-valued features (Theorem~\ref{thm:ub_partitioning_func_decision_stumps_rl_feat}), where $\nu \floor{\frac{m}{\min\cb{k,m-k}}}$ is changed for $2\ell$.
Note also that $\floor{\frac{m}{\min\cb{k,m-k}}} \ge 2$ for the considered values of $k$, hence we can expect that decision trees using the unitary comparison rule set are slightly more expressive than the threshold split.

\section{Analysis of General Decision Trees}
\label{sec:analysis_of_decision_trees}

In the previous section, we have presented three results on the 2-partitioning functions of decision stumps.
As a consequence from the fact that a decision tree can be defined recursively from its left and right subtrees, we can use these theorems as base cases for general tree classes.
Therefore, we here derive upper bounds for the $c$-partitioning function of any tree when the features are exclusively real-valued, ordinal or nominal, and we finally combine these results in a theorem that applies for any mixture of feature types.

The proofs of these results require some careful examinations for each type of feature.
Therefore, in this section we only present the theorems and discuss some of their implications, while we delegate the mathematical details and derivation to Appendix~\ref{app:proof_of_ub_partitioning_functions}.
All the results that are presented below share some common steps that relies on a recursive decomposition of $\P^c_T(S)$, the set of $c$-partitions realizable by a tree $T$ on sample $S$.
This is first exposed in Appendix~\ref{app:decision_trees_as_partitioning_machines}; subsequent appendices handle, in order, real-valued features, categorical features and a mixture of both.

\subsection{The Class of Decision Trees on Real-valued Features}
\label{ssec:decision_trees_on_rl_feat}

We now provide an extension of Theorem~\ref{thm:ub_partitioning_func_decision_stumps_rl_feat} that applies to any binary decision tree class on real-valued features, from which we derive the asymptotic behavior of the VC dimension of these classes.

\begin{theorem}[Upper bound on the $c$-partitioning function of decision trees on real-valued features]
\label{thm:ub_partitioning_functions_decision_trees_rl_feat}
Let $T$ be a binary decision tree class endowed with the threshold split rule set, let $T_l$ and $T_r$ be the hypothesis classes of its left and right subtrees, and let $L_T$ denote the number of leaves of $T$.
Let the examples be made of $\ell$ real-valued features.
Then, for $m \le L_T$, we have $ \pi^c_T(m) = \smallstirling{m}{c}$, whereas for $m > L_T$, the $c$-partitioning function must satisfy
\begin{equation}\label{eq:ub_partitioning_function_tree}
    \pi^c_T(m)
    \leq
    2^{-\delta_{lr}}
    \hspace{-7pt}
    \sum_{k=L_{T_l}}^{m-L_{T_r}}
    \!\!
    \min\cb{ 2\ell, \tbinom{m}{k} }
    \hspace{-7pt}
    \sum_{\substack{1 \leq a, b \leq c \\ a + b \geq c}}
    \!\!
    \tbinom{a}{c-b} \tbinom{b}{c-a} (a + b - c)!\;
    \pi^a_{T_l}(k) \pi^b_{T_r}(m-k)\, ,
\end{equation}
where $\delta_{lr} = \Id{T_l = T_r}$ is the Kronecker delta.
\end{theorem}

The proof is provided in Appendix~\ref{app:proof_decision_trees_rl}.
Note that the inequality~\eqref{eq:ub_partitioning_function_tree} of Theorem~\ref{thm:ub_partitioning_functions_decision_trees_rl_feat} reduces to the inequality~\eqref{eq:decision_stump_partitioning_function_rl_feat} of Theorem~\ref{thm:ub_partitioning_func_decision_stumps_rl_feat} when $T$ is the class of decision stumps.

From this Theorem, one can find the asymptotic behavior of the VC dimension of a binary decision tree class on examples with real-valued features.
It is stated in the following corollary.

\begin{corollary}[Asymptotic behavior of the VC dimension]\label{coro:asymptotic_VC_tree}
Let $T$ be a class of binary decision trees with a structure containing $L_T$ leaves on examples of $\ell$ real-valued features.
Then, $\vcdim T \in O\pr{ L_T \log(L_T\ell) }$.
\end{corollary}
\begin{proof}
Letting $c=2$ in Theorem~\ref{thm:ub_partitioning_functions_decision_trees_rl_feat}, using the fact that $2^{-\delta_{lr}} \le 1$, $\min\cb{2\ell,\binom{m}{k}} \le 2\ell$ and $\pi^c_{T}(k) \le \pi^c_T(m)$ for $k \le m$, we have
\begin{equation*}
    \pi^2_T(m) \le 2\ell (m-L_T) \pr{1 + 2\pi^2_{T_l}(m) + 2\pi^2_{T_r}(m) +2\pi^2_{T_l}(m)\pi^2_{T_r}(m) }.
\end{equation*}
Observe that in a binary tree, the number of internal nodes $N$ is equal to the number of leaves minus one.
Hence, let us show by induction that $\pi^2_T(m) \in O((m\ell)^N)$.
Assume $\pi^2_T(m) \le (Cm\ell)^N$ for some constant $C\ge 1$, and let $N_l$ and $N_r$ be the number of nodes in the left and right subtrees respectively, so that $N_l+N_r+1=N$.
The previous equation becomes (with $m-L_T<m$)
\begin{align*}
    \pi^c_T(m) &\le 2m\ell\pr{1 + 2(Cm\ell)^{N_l} + 2(Cm\ell)^{N_r} + 2(Cm\ell)^{N_l} (Cm\ell)^{N_r}}\\
    &\le 14m\ell (Cm\ell)^{N_l+N_r},
\end{align*}
which proves our claim for $C \ge 14$.
Then, Equation~\eqref{eq:vcdim_def_partition_func} implies
\begin{equation*}
    \vcdim T \le \max \cb{ m : (Cm\ell)^N \ge 2^{m-1}-1}.
\end{equation*}
One can solve for the inequality $(Cm\ell)^N \ge 2^m$ instead, since this implies $(Cm\ell)^N \ge 2^{m-1}-1$ is true too.
The Lambert $W$ function \citep{corless96lambertw} can give us an exact solution, which is $m \le - \frac{N}{\ln 2} W_{-1} \pr{ -\frac{\ln 2}{C N \ell} }$.
Since $-W_{-1}(-z^{-1}) \in O\pr{\log z}$, we have that $\vcdim T \in O\pr{N\log(N\ell)}$, which directly implies the corollary since $N = L_T-1$.
\end{proof}

\subsection{The Class of Decision Trees on Ordinal Features}
\label{ssec:decision_trees_on_ordinal_feat}

We here give an extension of Theorem~\ref{thm:ub_partitioning_func_decision_stumps_ordinal_feat} that applies to any binary decision tree class on ordinal features.
The following theorem bounds in two ways the partitioning functions of such trees.

\begin{theorem}[Upper bound on the $c$-partitioning function of decision trees on ordinal features]
\label{thm:ub_partitioning_functions_decision_trees_ordinal_feat}
Let $T$ be a binary decision tree class endowed with the threshold split rule set, and let $T_l$ and $T_r$ be the hypothesis classes of its left and right subtrees.
Let the examples be made of $\omega$ ordinal features following feature landscape $\Ob$.
The $c$-partitioning function must satisfy
\begin{equation*}
    \pi^c_T(m,\Ob) \le
    2^{-\delta_{lr}}
    \sum_{i=1}^\omega
    \sum_{k=1}^{m-1}
    R_{i,k}
    \hspace{-7pt}
    \sum_{\substack{1 \leq a, b \leq c\\ a+b\ge c}}
    \hspace{-7pt}
    \tbinom{a}{c-b}\tbinom{b}{c-a}(a+b-c)!\,
    \pi^a_{T_l}(k,\Ob^{k,i}) \pi^b_{T_r}(m-k,\Ob^{m-k,i}).
\end{equation*}
where $\delta_{lr} = \Id{T_l = T_r}$, $R_{i,k} = \min\cb{2^{\delta_{k,\frac{m}{2}}}(O_i-1), 2}$, and $O^{k,i}_j = \min\cb{O_j - \delta_{i,j}, k}$.
Furthermore, it also satisfies
\begin{equation*}
    \pi^c_T(m,\Ob) \le
    2^{-\delta_{lr}}
    \sum_{k=1}^{m-1}
    R_k'
    \hspace{-6pt}
    \sum_{\substack{1 \leq a, b \leq c\\ a+b\ge c}}
    \hspace{-7pt}
    \tbinom{a}{c-b}\tbinom{b}{c-a}(a+b-c)!\,
    \pi^a_{T_l}(k,\Ob) \pi^b_{T_r}(m-k,\Ob),
\end{equation*}
with
\begin{equation*}
    R_k' \eqdef \begin{cases}
        \min\cb{\Obar_1 + \Obar_2, \binom{m}{k}} & \text{if } k < \frac{m}{2}\\
        \min\cb{2\Obar_1, \binom{m}{k}} & \text{if } k = \frac{m}{2},
    \end{cases}
\end{equation*}
where $\Obar_1$ is the number of features with 2 categories or more and $\Obar_2$ is the number of features with 3 categories or more.
\end{theorem}

The proof, given in Appendix~\ref{app:proof_decision_trees_ordinal}, explains how to modify the bound for real-valued features to be able to integrate the dependence on the feature landscape.

Remark that using the fact that $\Obar_2 \le \Obar_1 \le \omega$ (by definition), $R_k'$ simplifies and we obtain
\begin{equation*}
    \pi^c_T(m) \le
    2^{-\delta_{lr}}
    \sum_{k=1}^{m-1}
    \min\cb{2\omega, \tbinom{m}{k}}
    \hspace{-6pt}
    \sum_{\substack{1 \leq a, b \leq c\\ a+b\ge c}}
    \hspace{-7pt}
    \tbinom{a}{c-b}\tbinom{b}{c-a}(a+b-c)!
    \pi^a_{T_l}(k) \pi^b_{T_r}(m-k),
\end{equation*}
where the bound no longer depend on the feature landscape $\Ob$.
This expression is nearly identical to the result on real-valued features (with $\ell$ replaced by $\omega$), as expected since the decision rules are essentially the same in both cases.

\subsection{The Class of Decision Trees on Nominal Features}
\label{ssec:decision_trees_on_nominal_feat}

Having considered decision trees on real-valued and ordinal features, we are left with nominal features.
The following theorem bounds from above in two different ways the partitioning functions of binary decision trees on examples made of nominal features, and thus extends Theorem~\ref{thm:ub_partitioning_func_decision_stumps_nominal_feat} that only applies to decision stumps.

\begin{theorem}[Upper bound on the $c$-partitioning function of decision trees on nominal features]
\label{thm:ub_partitioning_functions_decision_trees_nominal_feat}
Let $T$ be a binary decision tree class endowed with the unitary comparison rule set, and let $T_l$ and $T_r$ be the hypothesis classes of its left and right subtrees.
Let the examples be made of $\nu$ nominal features following feature landscape $\Nb$.
The $c$-partitioning function must satisfy
\begin{align*}
    \pi^c_T(m,\Nb) &
        \hspace{-1.5pt} \le \hspace{-1.5pt}
        2^{-\delta_{lr}}
        \hspace{-1.5pt}
        \sum_{i=1}^\nu
        \sum_{k=1}^{m-1}
        \hspace{-2pt}
        R'_{N_i,\min\cb{k,m-k}} 
        \hspace{-9pt}
        \sum_{\substack{1 \leq a, b \leq c\\ a+b\ge c}}
        \hspace{-9pt}
        \tbinom{a}{c-b}\tbinom{b}{c-a}(a\!+\!b\!-\!c)!
        \pi^a_{T_l}(k,\Nb^{k,i}) \pi^b_{T_r}(m\!-\!k,\Nb^{k,i})
\end{align*}
where $\delta_{lr} = \Id{T_l = T_r}$, $N^{k,i}_j = \min\cb{N_j - \delta_{i,j}, k}$ and
\begin{equation*}
    R'_{N,k} \eqdef \begin{cases}
    N-1 & \text{if } \frac{m}{k} > N,\\
    \floor{\frac{m}{k}} & \text{if } \frac{m}{k} \le N.
    \end{cases}
\end{equation*}
Furthermore, it also satisfies
\begin{equation*}
    \pi^c_T(m, \Nb) \le 
        2^{-\delta_{lr}}
        \hspace{-1.5pt}
        \sum_{k=1}^{m-1}
        R_k
        \hspace{-7pt}
        \sum_{\substack{1 \leq a, b \leq c\\ a+b\ge c}}
        \hspace{-7pt}
        \tbinom{a}{c-b}\tbinom{b}{c-a}(a+b-c)!
        \pi^a_{T_l}(k,\Nb^k) \pi^b_{T_r}(m-k,\Nb^k),
\end{equation*}
where $R_k \eqdef \min\cb{\smallstirling{m}{2}_k,\, \nu\! \floor{\frac{m}{\min\cb{k,m-k}}}}$, and where $N^k_i = \min\cb{N_i, k}$.
\end{theorem}

The proof of this theorem is presented in Appendix~\ref{app:proof_decision_trees_nominal} following a similar reasoning to the proofs for ordinal features.

Again, one can observe that the second bound is nearly identical to that of Theorem~\ref{thm:ub_partitioning_functions_decision_trees_rl_feat}, with the coefficient $2\ell$ replaced with $\floor{\frac{m}{\min\cb{k,m-k}}}\!\nu$ (where $\ell$ is the number of real-valued features).

\subsection{The Class of Decision Trees on a Mixture of Feature Types}
\label{ssec:decision_trees_on_mixture}

Once upper bounds on the partitioning functions of decision trees for real-valued, ordinal and nominal features have been worked out, it is straightforward to combine them into a general theorem that applies to domain spaces which present a mixture of these three types of features.
The following result gives an example of a theorem that can be derived as a corollary of the previous bounds.

\begin{theorem}[Upper bound on the $c$-partitioning function of decision trees on a mixture of real-valued and categorical features]
\label{thm:ub_partitioning_functions_decision_trees_mixture_feat}
Let $\X$ a the domain space with examples having $\ell$ real-valued features, $\omega$ ordinal features following feature landscape $\Ob$ and $\nu$ nominal features following feature landscape $\Nb$.
Let $T$ be a binary decision tree class that can use threshold splits as decision rules for real-valued and ordinal features and unitary comparisons as decision rules for nominal features.
Furthermore, let $T_l$ and $T_r$ be the hypothesis classes of its left and right subtrees.
The $c$-partitioning function must satisfy
\begin{align*}
    \pi^c_T(m,\ell, \Ob, \Nb) &
        \hspace{-1.5pt} \le \hspace{-1.5pt}
        2^{-\delta_{lr}}
        \hspace{-1.5pt}
        \sum_{k=1}^{m-1}
        \hspace{-2pt}
        R_k
        \hspace{-9pt}
        \sum_{\substack{1 \leq a, b \leq c\\ a+b\ge c}}
        \hspace{-9pt}
        \tbinom{a}{c-b}\tbinom{b}{c-a}(a\!+\!b\!-\!c)!
        \pi^a_{T_l}(k,\ell, \Ob^k, \Nb^k) \pi^b_{T_r}(m-k, \Ob^k, \Nb^k),
\end{align*}
where $\delta_{lr} = \Id{T_l = T_r}$, $O^k_i = \min\cb{O_i, k}$, $N^k_i = \min\cb{N_i, k}$, and
\begin{equation*}
    R_k \eqdef \min \cb{2\ell + 2\omega +  \floor{\frac{m}{\min\cb{k,m-k}}}\!\nu}.
\end{equation*}
\end{theorem}

The proof is supplied in Appendix~\ref{app:proof_decision_trees_mixture} and largely borrows from the proofs of the previous Theorems.
To keep things simple, we have used some specific parts of the proofs of Theorems~\ref{thm:ub_partitioning_functions_decision_trees_rl_feat}, \ref{thm:ub_partitioning_functions_decision_trees_ordinal_feat} and \ref{thm:ub_partitioning_functions_decision_trees_nominal_feat} to produce this new more general theorem, but one could, if he desired, generate other similar, but slightly more complex, results using different parts of these same proofs.

Theorem~\ref{thm:ub_partitioning_functions_decision_trees_mixture_feat} can be used recursively to compute an upper bound on the VC dimension of decision trees.
Indeed, starting with $m = 1$, one can evaluate the bound on $\pi^2_T(m)$ incrementally until it is less than $2^{m-1} - 1$, according to Equation~\eqref{eq:vcdim_def_partition_func}.
The algorithm is presented in Appendix~\ref{app:algorithms}.

\section{Experiments}
\label{sec:experiments}

Decision trees are known to be computationally hard to learn.
As a way to circumvent this problem, the community has come up with a non-optimal but sufficiently good and efficient two-step approach.
In the first step, the tree is greedily trained, adding node after node into an overgrown tree that either fits perfectly the data or has reached an arbitrarily fixed large number of leaves.
This tree is very prone to overfit, hence a second step is applied where unecessary branches are pruned from the tree in order to improve generalization.

This section aims to demonstrate how our framework can be useful in practice by applying our results to the task of pruning a greedily learned decision tree.
We consider two applications of the new theorems: in the first, we develop a novel pruning algorithm based on the VC generalization bound of \citet{shawe1998structural}, while in the second, we use the algorithm of \citet{kearns1998fast} in conjunction with our work to compute the VC dimension of the decision trees.
These two algorithms clearly outperform the most popular approaches---cost-complexity pruning and reduced-error pruning---when tested on 25 different data sets.

The section is divided as follows.
We start by describing the new pruning algorithm.
Then, we carefully explain the methodology, the choices that were made and the differences between the four considered algorithms.
Finally we present and discuss the results.

\subsection{A New Pruning Algorithm}
\label{ssec:pruning_algorithm}

A pruning algorithm must, given an overgrown decision tree, determine which nodes should be removed to optimize the error rate of the tree on unseen examples.
Therefore, a starting point is simply to minimize a generalization bound.
This is where our work shines: having an upper bound on the partitioning functions of decision trees---and by the same token on their growth function---allows us to use one of the many generalization bounds for VC classes without using a validation set; something which was not possible before.

We base our new pruning algorithm on Theorem~2.3 of \citet{shawe1998structural}, a choice we justify later.
Let us restate it here with some minors adjustments for the sake of completeness.
\begin{theorem}[\citet{shawe1998structural}]\label{thm:shawe-taylor}
Let $\X \times \Y$ be the domain space which contains all possible examples $(\x,\y)$.
Assume there exists some unknown distribution $\D$ over $\X \times \Y$, from which a sample $S$ of $m\in\naturals$ examples are each identically and independently drawn from $\D$.
Furthermore, let $\H \eqdef \bigcup_{d\in\naturals} H_d$ be any countable set of hypothesis classes $H_d \subset \Y^\X$ (with growth function $\tau_{H_d}$) indexed by an integer $d$, and let $p_d$ be any discrete probability distribution on $\naturals$.
For any classifier $h \in H_d$, let $k \eqdef \sum_{(\x,\y)\in S} \Id{h(\x) \ne \y}$ be the empirical error of $h$ on sample $S$.
Then, for any discrete probability distribution $q_k$ on $[m]$, with probability at least $1-\delta \in (0,1)$, the true risk $R_\D(h) \eqdef \exv{(\x,\y)\sim\D}{\Id{h(\x) \ne \y}}$ of any classifier $h\in H_d$ is at most  
\begin{equation}\label{eq:shawe-taylor_bound}
    \epsilon(m, k, d, \delta) \eqdef  \frac{1}{m} \pr{ 2k + 4 \ln \pr{ \frac{4 \tau_{H_d}(2m) }{\delta q_k p_d} }}.
\end{equation}
\end{theorem}
Although this theorem was originally stated for binary classification and for a sequence of nested hypothesis classes $H_d$ indexed by their VC dimension, it is also valid in the multiclass setting with zero-one loss if one uses the growth function directly instead of the upper bound provided by Sauer-Shelah's lemma, which can be much looser (take for instance the class $T$ of decision stump on $\ell$ real-valued features, where one has in reality $\log \tau_T(m) \in \O(d + \log m)$, with $d$ the VC dimension, while Sauer-Shelah's lemma simply implies $\log \tau_T(m) \in \O(d \log m/d)$).
Furthermore, it is not necessary to have nested hypothesis classes, since the main argument of the proof uses the union bound which applies for any countable set of classes.

The goal of our pruning algorithm is to minimize the true risk $R_\D(t)$ of a given overgrown tree $t$ by minimizing the upper bound $\epsilon$, where we assume that $\H$ is the set of all decision tree classes, and $H_d$ is the $d$-th decision tree class, where every decision tree class is arbitrarily and uniquely indexed by an integer $d$ that we call a \emph{complexity index}.
It goes as follows.
Given an overgrown decision tree $t$, any predetermined distributions $q_k$ and $p_d$, and a fixed confidence parameter $\delta$, we compute the bound $\epsilon$ associated to this tree.
Then, for each internal node of the tree, we create a candidate pruned tree by replacing the subtree rooted at this node either by
\begin{enumerate}
  \item its left subtree;
  \item its right subtree;
  \item or the leaf whose label minimizes the empirical error on $S$ of the newly pruned tree.
\end{enumerate}
For all three cases, we compute the bound associated with the resulting tree.
Among all such pruned trees, let $t'$ be the one that has the minimum bound value.
If the bound of $t'$ is less than or equal to the bound of $t$, we discard $t$ and we keep $t'$ instead.
Repeat this process until pruning the tree doesn't decrease the bound.
The formal version of the algorithm is presented in Algorithm~\ref{algo:prune_tree_with_bound} below.

\begin{algorithm2e}[t]
\caption{PruneTreeWithBound$(t, \epsilon, \delta, m)$}\label{algo:prune_tree_with_bound}
\DontPrintSemicolon
\SetAlgoVlined

\KwIn{A fully grown tree $t$, a bound function $\epsilon$ on the true risk, a confidence parameter $\delta$, the number of examples $m$.}

Let $T_d$ be the tree class of the tree $t$ with complexity index $d$.

Let $k_t$ be the number of errors made by $t$.

Let $b \leftarrow \epsilon(m,k_t,d,\delta)$ according to Equation~\eqref{eq:shawe-taylor_bound}.

Let $B \leftarrow b$ be the final bound.

\While{$t$ is not a leaf}{
  \For{every internal node $n$ of the tree $t$}{
    Let $t_n^1$ be the tree $t$ with node $n$ replaced by the left subtree $t_l$.

    Let $t_n^2$ be the tree $t$ with node $n$ replaced by the right subtree $t_r$.
    
    Let $t_n^3$ be the tree $t$ with node $n$ replaced by a leaf. The label of the leaf corresponds to the majority label of the examples of $S$ reaching node $n$.

    \For{every pruned tree $t_n^j$}{
      Let $T_{d_n}^j$ be the tree class of the tree $t_n^j$ with complexity index $d_n^j$. 

      Let $k_{t_n}^j$ be the number of errors made by $t_n^j$.
      
      \If{$\epsilon(m, k_{t_n}^j, d_n^j, \delta) \le b$}{
        Let $b \leftarrow \epsilon(m, k_{t_n}^j, d_n^j, \delta)$ be the new best bound.
        
        Let $t' \leftarrow t_n^j$ be the new best tree.
      }
    }
  }
  \uIf{$b \le B$}{
    Let $t \leftarrow t'$.
    
    Let $B \leftarrow b$.
  }
  \Else{
    \textbf{break}
  }
}

\KwOut{The pruned tree $t$, the associated bound $B$.}
\end{algorithm2e}

A key characteristic of the proposed algorithm is that it takes into account global information about the whole tree via its growth function.
This is to be contrasted with most pruning algorithms, which often only consider local information, \ie they make their decision to prune a node from its subtree only.



\subsection{Methodology}
\label{ssec:methodology}

We describe four aspects of the methodology with enough details in order to reproduce the experiments.
We first discuss the data selection and preparation, then we explain how we grow the decision tree.
Third, we present the six pruning models compared, and finally we examine some computational considerations.

\subsubsection{Data Selection and Preparation}
\label{sssec:data_selection_and_preparation}

We benchmark our pruning algorithm on 25 data sets taken from the UCI Machine Learning Repository \citep{Dua:2019}.
We chose data sets suited to a \emph{multiclass classification} task with either real-valued features only or a mixture of types, and no missing entries.
Furthermore, we limited ourselves to data sets with 10 or less classes, as Equation~\eqref{eq:ub_growth_func} becomes computationally expensive for a large number of classes.
All in all, 6 of the 25 data sets contain a mixture of feature types and the rest are real-valued only. In addition, for 10 out of the 25 data sets, the task is multiclass, while the others are binary classification tasks.

These data sets do not come with a defined train/test split.
As such, we chose to randomly split each data set so that the models are trained on 85\% of the examples and tested on the remaining 15\%.
To limit the effect of the randomness of the splits, we run each experiment 25 times with different seeds and we report the mean test accuracy and the standard deviation.

\subsubsection{Growing the Decision Tree}
\label{sssec:growing_the_decision_tree}

In all situations, decision trees are trained in a greedy fashion following \citet{breiman1984classification}'s CART algorithm, where we progressively grow the tree by replacing each leaf by the decision stump that optimizes the Gini index until it reaches a maximum of 75 leaves or until it perfectly fits the data.
This learning algorithm is implemented in the popular Python package \texttt{scikit-learn} but only for real-valued and ordinal features.
To accomodate the unitary comparison rule for nominal features, we have programmed our own implementation of the algorithm.
Note that this only requires a small modification to the original algorithm, and it does not require to convert each nominal feature $i$ into $N_i$ binary features, contrary to what is done in many pieces of software, such as the CatBoost library (when using the binary conversion and not the ordinal conversion).
This makes finding the optimal decision rule for a single node in $\O(\nu m \log m)$ instead of $\O(\norm{\Nb}_1 m \log m)$ where $\norm{\Nb}_1 = \sum_{i=1}^\nu N_i$ is the total number of categories, which can be quite faster when the $N_i$ values are large.
The algorithms used to grow the trees and the running complexity analyses are presented in Appendix~\ref{app:algorithms_to_grow_tree}.

A Python implementation of this algorithm as well as all other algorithms described in this paper are offered on GitHub at \url{https://github.com/jsleb333/paper-decision-trees-as-partitioning-machines}.
This repository also contains all the code to reproduce the experiments and the tables of results.

\subsubsection{Pruning Models}
\label{sssec:pruning_models}

For the experiment, we consider six ``pruning models'', which we describe below.

\noindent \textbf{The overgrown (OG) tree.}
The first model considered is the original unpruned overgrowned tree.
It serves as a baseline for all other models.

\noindent \textbf{CART cost-complexity (CC) pruning.}
Cost-complexity pruning, proposed initially in Chapter 3 of \cite{breiman1984classification}, is one of the most popular pruning algorithm. In fact, CC pruning is the pruning algorithm implemented in \texttt{scikit-learn}.
The idea is to assume that the true risk of a tree can be approximated via its empirical risk by adding a complexity term of the form $\alpha L_T$ to it, where $\alpha$ is a constant to be determined and $L_T$ is the number of leaves of the tree $T$.
Then, for every node $n$ in the tree, we ask if $k_n + \alpha \le k_{t_n} + \alpha L_{t_n}$, where $k_n$ is the number of errors made by the node $n$ (on the part of the examples that reach it), $t_n$ is the subtree rooted at node $n$ and $k_{t_n}$ is the number of errors made by $t_n$.
If the answer is yes, then we replace $t_n$ by the leaf with minimal empirical risk.

To find the value of $\alpha$, \citet{breiman1984classification} have proposed a cross-validation procedure, where, for each fold of the training set, one fully grows a decision tree and identifies a set of values that $\alpha$ can take as follows.
For each node $n$ of the tree, define $\alpha^*_n \eqdef \frac{k_n-k_{t_n}}{L_{t_n}-1}$ as a critical value such that if $\alpha^*_n \le \alpha$, the tree $t_n$ rooted at $n$ is pruned into a leaf, and otherwise it is kept.
Consider the set of pruned trees obtained by continuously increasing $\alpha$.
Observe that each time $\alpha$ is exactly equal to a critical value, a new node is pruned while all previously pruned nodes stay that way, thus producing at most $L_T-1$ pruned trees (where $L_T-1$ is the number of internal nodes).
Hence, one can see that there are really only $L_T-1$ values of $\alpha$ to be considered in the cross-validation in order to find the optimal value.

In our experiment, we find the value of $\alpha$ by proceeding with a 5-fold cross-validation on the training set.
For each fold, we find the critical value of $\alpha$ that minimizes the validation risk, and at the end we take the average of these five values to prune the decision tree grown on the whole training set.


Note that this algorithm, while being ingenious, only follows an heuristic and is not theoretically grounded.
Furthermore, the cross-validation step is computationally expensive and can slow down significantly the learning process.

\noindent \textbf{Reduced-error (RE) pruning.}
Reduced-error (RE) pruning, introduced with C4.5 by \citet{quinlan1986induction}, is also popular due to its simplicity and its efficiency.
The idea is simple: prune the subtrees such that the empirical error on some unseen data is minimal.
To achieve this, we split the training set into a new training set and a validation set, such that the validation set is the same size as the test set.
The overgrown tree is thus learned on $70\%$ of the data (instead of $85\%$), pruned on $15\%$ and tested on the remaining $15\%$ as before.
There are multiples ways to implement RE pruning; we choose to use the exact same algorithm presented Algorithm~\ref{algo:prune_tree_with_bound} with the exception that we compute the validation error of the tree instead of the upper bound on the generalization error.

This algorithm has the advantage of being extremely fast, but with the obvious inconvenient of reducing the number of examples that can be used to grow the original tree.
Note that for this same reason, this model is the only one of the six considered for which the overgrown tree is different.

\noindent \textbf{Kearns-Mansour (KM) pruning.}
The algorithm of \citet{kearns1998fast} is a single-pass bottom-up pruning algorithm, similar to that of \citet{breiman1984classification}'s CC pruning, but with the advantage of being theoretically sound.
Instead of using some \textit{ad hoc} formula to estimate the risk of a subtree, \citeauthor{kearns1998fast} propose to prune subtree $t_n$ rooted at node $n$ if $k_{t_n} + \alpha_\text{KM}(m, t, n, \delta) \ge k_n$ (Equation~1 and 2 of their paper), with
\begin{equation*}
  \alpha_\text{KM}(m, t, n, \delta) \eqdef C \sqrt{\frac{\ln(\tau_{P_n}(m_n)) + \ln(\tau_{T_n}(m_n)) + \ln(m/\delta)}{m}},
\end{equation*}
where $C>1$ is a universal constant, $m_n$ is the number of examples that reach node $n$, $P_n$ is the tree class that consists of a ``path'' from the root of $t$ to node $n$ (\ie the decision list that isolates the subset of the sample that reaches the node $n$), $T_n$ is the tree class of the subtree $t_n$, and $\delta$ is a (more or less inconsequential) confidence parameter between $0$ and $1$ that we set to $0.05$.
It is interesting to see that the term involving $\tau_{P_n}$ takes into account the depth of the subtree inside the overgrown tree, while $\tau_{T_n}$ considers the complexity of the subtree only.
Remark that this algorithm requires the growth function of decision trees to be worked out, for which a bound was only known for binary features prior to our work.

Because the derivations of this equation relies on multiple approximations, the exact expression for $\alpha_\text{KM}$ is simply unusable in this state as it grossly overestimates the true risk.
Thus, for our experiments, we decided to make a 5-fold cross-validation of the value of $C$ between $\cb{10^{-20}, 10^{-19},\dots, 10^0}$, which made the algorithm competitive.
We use Equation~\eqref{eq:ub_growth_func} and Theorem~\ref{thm:ub_partitioning_functions_decision_trees_mixture_feat} to compute $\tau_{P_n}$ and $\tau_{T_n}$, albeit with some modifications to accelerate the computation described in the following section.

\noindent \textbf{Proposed pruning algorithm (Ours).}
The fifth model considered is the pruning algorithm proposed in the previous section.
The hyperparameters to handle are $\delta$, $p_d$ and $q_k$.
As for the KM model, we fix $\delta = 0.05$.
The choices of distributions $p_d$ and $q_k$ are arbitrary and should reflect our prior knowledge of the problem.
We would like $p_d$ to go to $0$ slowly as $d$ grows in order not to penalize large trees too severely.
As mentioned below Theorem~\ref{thm:shawe-taylor}, it is not necessary to use the VC dimension to index the hypothesis classes.
In fact, we have several deterrent to using this quantity as index: it is defined for binary classification while we work with multiclass problems, the exact VC dimension is still unknown despite our work providing an upper bound, and it is not easily computable for large trees.
As an alternative, we opt to use the number of leaves $L_T$ of the tree $T$ as a complexity index, and we give the same probability $p_d$ to every tree with the same number of leaves.
We thus choose to let $p_d = \frac{6}{\pi^2 L_{T_d}^2} \frac{1}{\text{WE}(L_{T_d})}$ so that $\sum_d p_d = 1$, where $\text{WE}(L_T)$ denotes the $L_T$-th Wedderburn-Etherington number \citep{bona2015handbook}, which counts the number of structurally different binary trees with $L_T$ leaves (the first ten are 1, 1, 1, 2, 3, 6, 11, 23, 46, 98).

We observed that, in Shawe-Taylor's bound~\eqref{eq:shawe-taylor_bound}, the penalty given to the complexity of the tree is disproportionately larger that the penalty given to the number of errors.
This is because much of the looseness of the bound comes from the growth function.
Indeed, it is already an upper bound for the annealed entropy, and our bound of the growth function adds even more looseness on top of that.
Fortunately, Shawe-Taylor's bound offers us a chance to compensate this fact by choosing an appropriate distribution $q_k$ which will introduce a large penalty for the number of errors $k$.
We chose $q_k$ of the form $(1-r) r^{k}$ for some $r < 1$, such that $\sum_k q_k$ is a geometric series summing to $1$.
We made a 5-fold cross-validation of $r$ on a single data set and we stuck with this value of $r$ for all others to avoid the slowdown incurred by this approach.
We tried inverse powers of $2$ for $r$ and we took the geometric mean of $10$ draws as the final value.
The Wine data set from the UCI Machine Learning Repository \citep{Dua:2019} gave a value of $r = 2^{-10.5} \approx \frac{1}{1448}$.
This choice makes the value of the bound $\epsilon$ larger. However, it allows to correct the gap between the complexity dependence and the dependence of the bound on the number of errors, which gives better results in practice.

We compute $\tau_{H_d}$ the same way we do for the KM model, with the same modifications.

\noindent \textbf{Oracle pruning.}
The last model we consider the oracle pruned tree, where the overgrown tree is pruned such that it minimizes the error on the \emph{test} set.
This way, we are able to observe the maximum gain that can be achieved and compare it to the other models.
The algorithm we use is simply the reduced-error pruning on the test set rather than on the validation set.

\subsubsection{Computing the Growth Function}

When running the experiments, we have been confronted to some minor encumbrances in regard to the computation of the growth function of decision trees, which can easily be solved with some precautions.

First, we observed that Theorem~\ref{thm:ub_partitioning_functions_decision_trees_mixture_feat} is computationally expensive because of the sum over $k$.
Hence, we used the following upper bound for the partitioning functions instead
\begin{equation*}
    \pi^c_T(m)
      \leq 2^{-\delta_{lr}} (m-1) R_{m-1}
      \hspace{-7pt} \sum_{\substack{1 \leq a, b \leq c \\ a + b \geq c}}
      \hspace{-7pt}
      \tbinom{a}{c-b} \tbinom{b}{c-a} (a + b - c)!\;
      \pi^a_{T_l}(m-1) \pi^b_{T_r}(m-1),
\end{equation*}
which simply replaces the sum over $k$ by $m-1$ times the greatest term of the sum.
This modified expression is much faster to compute and had only a small impact on the bound $\epsilon$ because of the logarithmic dependence on the growth function.

Furthermore, to avoid overflows, we considered a logarithmic version of the bound, where we made use of the general formula for summations:
\begin{equation*}
  \log \pr{\sum_{i=1}^n A_i} = \log A_j  + \log \pr{\sum_{i=1}^n \frac{A_i}{A_j}} = \log A_j + \log\pr{ \sum_{i=1}^n \exp \pr{ \log A_i - \log A_j } },
\end{equation*}
where choosing $j \eqdef \argmax_i A_i$ ensures us that every term in the last sum will be less than or equal to 1, and thus numerical errors will be minimal.
We also used this formula to implement an algorithm to compute the logarithm of the growth function according to Equation~\eqref{eq:ub_growth_func}.
Algorithms~\ref{algo:log_partition_func_upper_bound} and \ref{algo:growth_func_upper_bound} of Appendix~\ref{app:algorithms} presents the exact computation steps executed to recover the bound.

\subsection{Results and Discussion}

Table~\ref{table:results} presents the results\footnote{The results of the present paper differs to that of the short version published at NeurIPS~2020 for multiple reasons, but mainly because: 1) the newly proposed pruning algorithm is slightly different, 2) the train/test split ratio is different to accomodate the RE algorithm which requires a validation set, and 3) the maximum number of leaves is set to 75 instead of 40. The maximum number of leaves was changed because we found out that some overgrown trees were in fact underfitting the training set.} of the six models compared on the 25 data sets.
The top part shows the mean test accuracy across 25 random splits, while the bottom part gives some global model-specific statistics.
The five last lines of the table show the number of times a model outperformed all other models, the mean test accuracy across all data sets, the mean fraction (in percent) of the oracle accuracy achieved by each model, the mean time required to prune the OG tree, and the mean number of leaves at the end of the pruning process.
Remark that the RE model prunes a different tree than the OG tree because it is trained on a smaller training set.
This implies that the Oracle only applies for the other three pruning models---in particular, it is possible for the RE model to outperform the Oracle, as it does for the Cardiotocography10 data set. 
The standard deviation is omitted for conciseness, but is included along with more statistics about each models in Appendix~\ref{app:more_stats}.

\begin{table}[h!]
\centering
\caption{Mean test accuracy (in percent) on 25 random splits of 25 data sets obtained from the UCI Machine Learning Repository \citep{Dua:2019}.
Starred data sets contain a mixture of feature types (the others are real-valued only).
The total number of examples followed by the number of classes of the data set is indicated in parenthesis.
The best performances up to a $0.1\%$ accuracy gap are highlighted in bold.
}  
\label{table:results}
\vspace{6pt}
\small
\begin{tabular}{l@{\hspace{6pt}}c@{\hspace{6pt}}c@{\hspace{6pt}}c@{\hspace{6pt}}c@{\hspace{6pt}}c@{\hspace{6pt}}cl@{\hspace{6pt}}c@{\hspace{6pt}}c@{\hspace{6pt}}c@{\hspace{6pt}}c@{\hspace{6pt}}c@{\hspace{6pt}}cl@{\hspace{6pt}}c@{\hspace{6pt}}c@{\hspace{6pt}}c@{\hspace{6pt}}c@{\hspace{6pt}}c@{\hspace{6pt}}cl@{\hspace{6pt}}c@{\hspace{6pt}}c@{\hspace{6pt}}c@{\hspace{6pt}}c@{\hspace{6pt}}c@{\hspace{6pt}}cl@{\hspace{6pt}}c@{\hspace{6pt}}c@{\hspace{6pt}}c@{\hspace{6pt}}c@{\hspace{6pt}}c@{\hspace{6pt}}cl@{\hspace{6pt}}c@{\hspace{6pt}}c@{\hspace{6pt}}c@{\hspace{6pt}}c@{\hspace{6pt}}c@{\hspace{6pt}}cl@{\hspace{6pt}}c@{\hspace{6pt}}c@{\hspace{6pt}}c@{\hspace{6pt}}c@{\hspace{6pt}}c@{\hspace{6pt}}c}        
\toprule
\multirow{2}{*}[-3pt]{Data set} & \multicolumn{6}{c}{Pruning Model}\\
\cmidrule{2-7}
 & OG & CC & RE & KM & Ours & Oracle\\
\midrule
Acute Inflammation* (120, 2) & $\mathbf{100.00}$ & $91.56$ & $97.11$ & $\mathbf{100.00}$ & $\mathbf{100.00}$ & $100.00$\\
Amphibians* (189, 2) & $59.14$ & $59.71$ & $\mathbf{61.71}$ & $56.00$ & $58.86$ & $81.43$\\
Breast Cancer Wisconsin Diagnostic (569, 2) & $92.14$ & $90.78$ & $92.38$ & $\mathbf{93.18}$ & $\mathbf{93.27}$ & $95.25$\\
Cardiotocography10 (2126, 10) & $57.05$ & $56.95$ & $\mathbf{60.83}$ & $57.79$ & $58.16$ & $60.26$\\
Climate Model Simulation Crashes (540, 2) & $89.73$ & $90.86$ & $90.67$ & $\mathbf{91.31}$ & $\mathbf{91.36}$ & $93.98$\\
Connectionist Bench Sonar (208, 2) & $71.35$ & $69.68$ & $68.26$ & $72.52$ & $\mathbf{72.77}$ & $83.48$\\
Diabetic Retinopathy Debrecen (1151, 2) & $62.24$ & $59.24$ & $60.67$ & $61.66$ & $\mathbf{63.38}$ & $70.38$\\
Fertility (100, 2) & $76.27$ & $\mathbf{88.00}$ & $84.53$ & $84.27$ & $\mathbf{88.00}$ & $89.87$\\
Habermans Survival (306, 2) & $65.39$ & $\mathbf{72.78}$ & $70.00$ & $\mathbf{72.78}$ & 
$71.57$ & $80.96$\\
Heart Disease Cleveland Processed* (303, 5) & $46.76$ & $50.67$ & $49.96$ & $51.11$ & $\mathbf{51.47}$ & $63.82$\\
Image Segmentation (210, 7) & $86.12$ & $84.00$ & $81.50$ & $\mathbf{86.38}$ & $85.62$ & $88.25$\\
Ionosphere (351, 2) & $88.30$ & $79.62$ & $88.15$ & $89.13$ & $\mathbf{89.58}$ & $93.43$\\
Iris (150, 3) & $92.91$ & $88.00$ & $92.73$ & $92.91$ & $\mathbf{93.64}$ & $96.00$\\    
Mushroom* (8124, 2) & $\mathbf{100.00}$ & $99.57$ & $\mathbf{99.97}$ & $\mathbf{100.00}$ & $\mathbf{100.00}$ & $100.00$\\
Parkinson (195, 2) & $84.97$ & $82.21$ & $83.72$ & $82.48$ & $\mathbf{85.24}$ & $90.76$\\
Planning Relax (182, 2) & $56.74$ & $68.44$ & $69.63$ & $\mathbf{72.59}$ & $57.04$ & $78.96$\\
Qsar Biodegradation (1055, 2) & $79.97$ & $78.99$ & $80.46$ & $81.34$ & $\mathbf{82.20}$ & $87.90$\\
Seeds (210, 3) & $\mathbf{91.12}$ & $87.75$ & $90.50$ & $90.88$ & $\mathbf{91.12}$ & $94.88$\\
Spambase (4601, 2) & $85.78$ & $84.68$ & $\mathbf{88.19}$ & $85.93$ & $86.07$ & $88.54$\\
Statlog German* (1000, 2) & $65.47$ & $\mathbf{70.16}$ & $68.03$ & $\mathbf{70.24}$ & $69.65$ & $74.00$\\
Vertebral Column 3C (310, 3) & $77.48$ & $72.96$ & $79.30$ & $78.35$ & $\mathbf{80.61}$ 
& $87.83$\\
Wall Following Robot24 (5456, 4) & $\mathbf{99.52}$ & $\mathbf{99.47}$ & $99.26$ & $\mathbf{99.51}$ & $99.39$ & $99.63$\\
Wine (178, 3) & $\mathbf{92.44}$ & $90.37$ & $85.93$ & $\mathbf{92.44}$ & $91.11$ & $94.52$\\
Yeast (1484, 10) & $45.72$ & $35.19$ & $\mathbf{49.00}$ & $46.91$ & $47.80$ & $49.69$\\ 
Zoo* (101, 7) & $\mathbf{94.93}$ & $92.80$ & $86.40$ & $\mathbf{94.93}$ & $92.80$ & $96.53$\\
\midrule
Number of best & 6 & 4 & 5 & 11 & 14 & N/A\\
Mean accuracy (\%) & 78.46 & 77.78 & 79.16 & 80.19 & 80.03 & 85.61\\
Mean fraction of Oracle (\%) & 91.01 & 90.14 & 92.21 & 93.23 & 93.04 & 100.00\\
Mean pruning time (s) & N/A & 9.98 & 0.95 & 6.57 & 2.87 & 0.95\\
Mean number of leaves & 37.02 & 6.19 & 7.29 & 15.56 & 12.00 & 7.08\\
\bottomrule
\end{tabular}
%
\end{table}

Comparing the four pruning algorithms, one can see that Ours is the best most of the time, performing equally well as or better than other models on 14 out of 25 data sets, whereas the KM algorithm dominates for 11 data sets.
However, when comparing the global mean accuracy, KM pruning outperforms slightly (0.16\% more) our model, so it is difficult to determine the better of the two.

On the other hand, we can observe that without a doubt, these models, both of which rely on the novel work derived in this paper, outperform significantly the popular CC and RE pruning algorithms for mean accuracy and number of best.

Notice that CC pruning actually worsen the test accuracy of the overgrown tree.
This seems correlated with the number of leaves that are kept by the algorithm, as trees resulting from CC pruning have in average only $6.19$ leaves.
This suggests that one would be better to not use CC, except if they really need to reduce the size of tree.
Not only CC pruning is the worst of all models considered, but it is also the slowest technique, being in average $10$ times slower than RE, and about $3.5$ times slower than Ours, due to the fact that it requires cross-validation.
KM pruning is also slow because of this reason, being more than twice as slower than Ours.

It is interesting to see that KM and Ours, which perform the best, have about twice the number of leaves of CC and RE.
This hints that the latters prune more aggressively than necessary the OG trees.



\section{Conclusion}
By considering binary decision trees as partitioning machines, and introducing the set of partitioning functions of a tree class, we have found that the VC dimension of a tree class is given by the largest integer $d$ such that $\pi^2_T(d) = 2^{d-1}-1$.
Then, we found a tight upper bound on the 2-partitioning function of the class of decision stumps on $\ell$ real-valued features.
This bound allowed us to find the exact VC dimension of decision stumps, which is given by the largest $d$ such that $2\ell \geq \tbinom{d}{\floor{\frac{d}{2}}}$.
Moreover, a greedy algorithm was developped to compute a tight bound for the partitioning function of a decision stump on $\omega$ ordinal features, and an upper bound was derived when the examples were made from $\nu$ nominal features.
Then, we successfully extend each of these results into a recursive upper bound of the $c$-partitioning functions of any class of binary decision tree on real-valued, ordinal, nominal or a combination of the three.
As a corollary, we were able to found that the VC dimension of a tree class on $\ell$ real-valued features with $L_T$ leaves is of order $\O(L_T\log(L_T\ell))$.

Based on our findings, we proposed a new pruning algorithm based on Shawe-Taylor's generalization bound, and we applied our theorems to the pruning algorithm of \citet{kearns1998fast}.
Both of these algorithms performed better than the popular \citet{breiman1984classification}'s cost-complexity and \citet{quinlan1986induction}'s reduced-error pruning algorithms on the majority of the 25 data sets, showing that bound-based algorithms are applicable and useful alternatives to these non-theoretically grounded pruning methods.

\acks{This work was supported by NSERC Discovery grant RGPIN-2016-05942, by NSERC ES D scholarship PGSD3–505004–2017 and by NSERC BRPC scholarship BRPC-540188-2019.}

\newpage

\appendix

\section{Proof of the bound for stumps on real-valued features}
\label{app:proof_stump_rl}

Before proceeding with the proof, we introduce a convenient way to think about a node's decision rule.
Recall that a node is associated with a threshold split rule described by a feature $i \in [\ell]$, a threshold $\theta \in \mathds{R}$, and a sign $s \in \cb{\pm 1}$.
The sample $S = \{\x_1, \hdots, \x_m\}$ may be represented by a collection $\Sigma$ of $\ell$ permutations of $[m]$ representing the ordering of its data points according to their values for each feature, since that relative ordering encapsulates all relevant information on the sample, from the perspective of decision trees. To be more precise, for each $i = 1, \hdots, \ell$, let $\sigma^{i}$ be a permutation of $[m]$ satisfying 
\begin{equation*}
x_{\sigma_1^{i}}^{i} \leq x_{\sigma_2^{i}}^{i} \leq \cdots \leq x_{\sigma_m^{i}}^{i} \, .
\end{equation*}
In general, unless the data points all have different values for a given feature, there may be many such permutations; just pick one arbitrarily. 

Any node in a decision tree splits the data points in two according to a rule of the form 
\begin{equation*}
t(\x) = \begin{cases} t_l(\x) &\text{if} \, s \cdot x^{i} \le \theta \\
t_r(\x) &\text{otherwise.}
\end{cases}
\end{equation*}
This corresponds to splitting the permutation 
\begin{equation*}
\sigma^{i} = \begin{bmatrix}
\sigma_1^{i} & \sigma_2^{i} & \cdots & \sigma_m^{i}
\end{bmatrix}
\end{equation*}
in two parts, sending examples $\x_{\sigma_j^i}$ to one subtree for $j \leq J$, and sending the rest of the examples to the other subtree, where $J$ is determined by $\theta$ and $s$. In fact, as long as the inequalities 
\begin{equation*}
x_{\sigma_1^{i}}^{i} < x_{\sigma_2^{i}}^{i} < \cdots < x_{\sigma_m^{i}}^{i} 
\end{equation*}
are strict (all data points have different values for each feature), then all the different ways of splitting the (now unique) permutation $\sigma^{i}$ induce a split on the sample $S$ according to which it was defined. This situation could be called the worst-case scenario, because it allows for more distinct 2-partitions to be realized on the sample. 

We split the proof in 4 parts: 1) the bound itself, 2) the equality for $2\ell \leq m$, 3) the equality for $2\ell \geq \binom{m}{\floor{\frac{m}{2}}}$, and 4) the equality for $1 \leq m \leq 7$.

\subsection{Proof of part 1 of Theorem~\ref{thm:ub_partitioning_func_decision_stumps_rl_feat}}
\label{app:partitions_node}

We want to show that 
\begin{equation*}
\pi^2_T(m) \leq \frac{1}{2} \sum_{k=1}^{m-1} \min \cb{ 2\ell, \binom{m}{k} }.
\end{equation*}
where $T$ is the class of decision stumps on $\ell$ real-valued features.

\begin{proof}
First, let $\R(S)$ be the set of 2-partitions of $S$ realizable by a single node, and notice that bounding the cardinality of $\R(S)$ directly gives a bound on $\pi^2_T(m)$ if the bound does not depend directly on $S$.

Let $\R_k(S) \subset \R(S)$ be the subset of $2$-partitions with a part of size $k$, and notice $\R_k(S) = \R_{m-k}(S)$.
Therefore, we can decompose $\R(S)$ into the disjoint union
\begin{align*}
    \R(S) = \bigcup_{k=1}^{\floor{\frac{m}{2}}} \R_k(S).
\end{align*}

To bound $\abs{\R_k(S)}$, first consider $k < \frac{m}{2}$. Every partition in $\R_k(S)$ is determined by a set of $k$ data points, so that $\abs{\R_k(S)} \leq \binom{m}{k}$, the number of $k$-subsets of $S$. 
On the other hand, given a feature $i \in [\ell]$, in the worst-case scenario, we can split the permutation $\sigma^{i}$ after the $k$ first points or before the $k$ last points to induce 2 distinct elements of $\R_k(S)$. Since there are $\ell$ features, this makes a total of at most $2\ell$ realizable 2-partitions with a part of size $k$. We conclude that, for $k < \frac{m}{2}$, we have $\abs{\R_k(S)} \leq \min \left\{ 2\ell, \binom{m}{k} \right\}$. 

Now let $k = \frac{m}{2}$. Then the same arguments apply, except that the number of $2$-partitions with a part of size $k$ is $\frac{1}{2} \binom{m}{k}$ because each such partition contains two subsets of the same size $k$. Moreover, for the same reason, the node can produce at most only one $2$-partition with a part of size $k$ for each feature. 
Thus, $\abs{\R_k(S)} \leq \min \left\{ \ell, \frac{1}{2} \binom{m}{k} \right\}$.

Combining our results, we have
\begin{align}\label{eq:bound_on_RkS}
    \abs{\R_k(S)} \eqdef \begin{cases}
        \min \cb{\ell, \frac{1}{2}\binom{m}{k} } & \text{if} \, k = \frac{m}{2} \\
        \min \cb{2\ell, \binom{m}{k} } & \text{otherwise.}
    \end{cases}
\end{align}

Using Inequality~\eqref{eq:bound_on_RkS}, the symmetry $\R_k(S) = \R_{m-k}(S)$ yields, for $m$ odd,
\begin{align*}
\abs{\R(S)}
    &= \sum_{k=1}^{\floormtwo} \abs{\R_k(S)}\\
    &= \frac{1}{2}\pr{\sum_{k=1}^{\floormtwo} \abs{\R_k(S)} + \sum_{k=1}^{\floormtwo} \abs{\R_{m-k}(S)}}\\
    &= \frac{1}{2}\pr{\sum_{k=1}^{\floormtwo} \abs{\R_k(S)} + \sum_{k=\floormtwo+1}^{m-1} \abs{\R_{k}(S)}}\\
    &\leq \frac{1}{2} \sum_{k=1}^{m-1} \min \left\{ 2\ell, \binom{m}{k} \right\},
\end{align*}
and, for $m$ even,
\begin{align*}
\abs{\R(S)}
    &= \sum_{k=1}^{\frac{m}{2}} \abs{\R_k(S)}\\
    &= \frac{1}{2}\pr{\sum_{k=1}^{\frac{m}{2}} \abs{\R_k(S)} + \sum_{k=1}^{\frac{m}{2}} \abs{\R_{m-k}(S)}}\\
    &= \frac{1}{2}\pr{\sum_{k=1}^{\frac{m}{2}} \abs{\R_k(S)} + \sum_{k=\frac{m}{2}}^{m-1} \abs{\R_{k}(S)}}\\
    &= \frac{1}{2}\pr{\abs{\R_{\frac{m}{2}}(S)} + \sum_{k=1}^{m-1} \abs{\R_{k}(S)}}\\
    &\leq \frac{1}{2} \sum_{k=1}^{m-1} \min \left\{ 2\ell, \binom{m}{k} \right\},
\end{align*}
which is the same for both parity, hence concluding the proof.
\end{proof}

\subsection{Proof of part 2 of Theorem~\ref{thm:ub_partitioning_func_decision_stumps_rl_feat}}
\label{app:proof_part_2_vcdim_stump}

\begin{proof}
We want to show that the bound of Theorem~\ref{thm:ub_partitioning_func_decision_stumps_rl_feat} is an equality for $2\ell \leq m$. To this end, we want to show the existence of a sample $S$ such that 
\begin{equation*}
\lvert \mathcal{R}_k(S) \rvert = \begin{cases} \ell &\text{if} \, k = \frac{m}{2} \\
2\ell &\text{otherwise.}
\end{cases}
\end{equation*}
Since $2\ell \leq m$ implies $2\ell \leq \binom{m}{k}$ for all $k$, we will have
\begin{equation*}
\lvert \mathcal{R}(S) \rvert = \sum_{k = 1}^{\floor{\frac{m}{2}}} \lvert \mathcal{R}_k(S) \rvert = \ell (m - 1) = \frac{1}{2} \sum_{k = 1}^{m - 1} 2\ell = \frac{1}{2} \sum_{k = 1}^{m - 1} \min\left\{2\ell, \binom{m}{k}\right\}
\end{equation*}
which establishes that the bound of Theorem~\ref{thm:ub_partitioning_func_decision_stumps_rl_feat} is an equality. 

Let us construct a suitable sample $S$. Consider the permutations $\sigma^1, \hdots, \sigma^\ell$ given by the rows of the following permutation representation of $S$:
\begin{equation*}
\Sigma =
\left[
\begin{array}{c c c c | c c c c | c c c c}
1 & 2 & \hdots & l & 2l + 1 & 2l + 2 & \hdots & m & 2l & 2l - 1 & \hdots & l + 1\\
2 & 3 & \hdots & l + 1 & 2l + 1 & 2l + 2 & \hdots & m & 1 & 2l & \hdots & l + 2\\
3 & 4 & \hdots & l + 2 & 2l + 1 & 2l + 2 & \hdots & m & 2 & 1 & \hdots & l + 3\\
\vdots & \vdots & \ddots & \vdots & \vdots & \vdots & \ddots & \vdots & \vdots & \vdots & \ddots & \vdots\\
l & l + 1 & \hdots & 2l - 1 & 2l + 1 & 2l + 2 & \hdots & m & l - 1 & l - 2 & \hdots & 2l
\end{array} 
\right] .
\end{equation*}
$\Sigma$ is built up from an $\ell \times \ell$ matrix on the left, an $\ell \times (m - 2\ell)$ matrix in the middle, and an $\ell \times \ell$ matrix on the right. In the remainder of this paragraph, a shift is a shift in the sequence $1, 2, \hdots, 2\ell$. The first row of the left matrix is $1, 2, \hdots, \ell$; subsequent rows are obtained by shifting one position to the right. The middle matrix has identical rows running from $2\ell + 1$ to $m$. The first row of the right matrix is $2\ell, 2\ell - 1, \hdots, \ell + 1$; subsequent rows are obtained by shifting one position to the left. For example, if $\ell = 3$ and $m = 9$, we have
\begin{equation*} 
\Sigma =
\left[
\begin{array}{c c c | c c c | c c c}
1 & 2 & 3 & 7 & 8 & 9 & 6 & 5 & 4\\
2 & 3 & 4 & 7 & 8 & 9 & 1 & 6 & 5\\
3 & 4 & 5 & 7 & 8 & 9 & 2 & 1 & 6
\end{array} 
\right] .
\end{equation*}
It is clear that, for $k = 1, \hdots, \floor{\frac{m}{2}}$, splitting any of these permutations after the first $k$ points or before the last $k$ points always induces different 2-partitions with a part of size $k$ on the sample, as long as the sample is chosen so that the strict inequalities 
\begin{equation*}
x_{\sigma_1^{i}}^{i} < x_{\sigma_2^{i}}^{i} < \cdots < x_{\sigma_m^{i}}^{i} 
\end{equation*}
hold; it suffices to choose $x_{\sigma_j^{i}}^{i} = j$ for $i = 1, \hdots, \ell$ and $j = 1, \hdots, m$. This gives us a total of $\ell$ distinct 2-partitions if $k = \frac{m}{2}$ (with even $m$), and a total of $2\ell$ distinct permutations if $k < \frac{m}{2}$, as required. 
\end{proof}

\subsection{Proof of part 3 of Theorem~\ref{thm:ub_partitioning_func_decision_stumps_rl_feat}}
\label{app:proof_part_3_vcdim_stump}

We prove part 3 of Theorem~\ref{thm:ub_partitioning_func_decision_stumps_rl_feat} by showing that for $2\ell \geq \tbinom{m}{\floormtwo}$ \big(so that $2\ell \geq \binom{m}{k}$ for all $k$\big), there exists a sample $S$ such that 
\begin{equation*}
\abs{\R(S)} = \sum_{k = 1}^{m - 1} \binom{m}{k} = \sum_{k = 1}^{m - 1} \min \left\{ 2\ell, \binom{m}{k} \right\}.
\end{equation*}
Part 3 of the Theorem is crucial in order to obtain the exact VC dimension of decision stumps on real-valued features.

We proceed in two steps. First, we show that there exists a sample $S$ of $m$ examples on which every 2-partition with a part of size $\floormtwo$ is realized by a stump, when $2\ell \geq \binom{m}{\floormtwo}$. 
Second, we use induction from this base case to establish the proof for all part sizes.
More precisely, we show that if there exists a sample $S_k$ such that a stump can realize every $2$-partition with a part of size $2 \leq k \leq \frac{m}{2}$, then there also exists a sample $S_{k - 1}$ of the same size such that a stump can realize every $2$-partition with a part of size $k$ \emph{and} every $2$-partition with a part of size $k-1$.

Let $\Sigma$ be the permutation representation of $S$, as explained at the beginning of Appendix~\ref{app:proof_stump_rl}. Furthermore, assume we are in the worst-case scenario where
\begin{equation*}
x_{\sigma_1^{i}}^{i} < x_{\sigma_2^{i}}^{i} < \cdots < x_{\sigma_m^{i}}^{i} 
\end{equation*}
for all $i \in [\ell]$. In this case, showing that every 2-partition of $S$ is realizable by a decision stump is equivalent to showing that every $k$-subset of $[m]$ is attainable by splitting a permutation of $\Sigma$ in two, either by splitting after the first $k$ elements or before the last $k$ elements for every possible $k$. Moreover, we only need to consider $k$-subsets for $1 \leq k \leq \frac{m}{2}$ since $\R_k(S) = \R_{m-k}(S)$.

\textbf{Step 1.}
We want to show that there exists a sample $S_\floormtwo$ of $m$ examples on which every $2$-partition with a part of size $\floormtwo$ is realized by a stump when $2\ell \geq \tbinom{m}{\floormtwo}$, \ie when $\ell \geq \ceil{\frac{1}{2}\tbinom{m}{\floormtwo}}$. Let $\Sigma_{\floor{\frac{m}{2}}}$ be its permutation representation. Our problem is then equivalent to finding a matrix $\Sigma_{\floor{\frac{m}{2}}}$ whose rows are permutations of $[m]$ such that each $\floor{\frac{m}{2}}$-subset of $[m]$ may be found as the first $\floor{\frac{m}{2}}$ elements or the last $\floor{\frac{m}{2}}$ elements of a row of $\Sigma_{\floor{\frac{m}{2}}}$. 

This is easy for even $m$. Given that $\ell \geq \frac{1}{2}\tbinom{m}{\frac{m}{2}}$ and that there are exactly $\frac{1}{2}\binom{m}{\frac{m}{2}}$ different 2-partitions of $[m]$ with a part of size $\frac{m}{2}$, we can fit them all the first $\ell$ rows of the matrix $\Sigma_{\frac{m}{2}}$ with the first $\frac{m}{2}$ elements of each row being the elements of the first part of each 2-partition. Then, $\Sigma_{\frac{m}{2}}$ induces a sample $S_{\frac{m}{2}}$ on which every 2-partition is realizable by a stump. If $S_{\frac{m}{2}} = \{ \x_1, \hdots, \x_m \}$, choosing $x_{\rho_j^{i}}^{i} = j$, where the $\rho_j^{i}$ are 
the elements of the matrix $\Sigma_{\frac{m}{2}}$, suffices. 

Now, let's see what happens when $m$ is odd. Consider the minimal case $\ell = \ceil{\frac{1}{2}\tbinom{m}{\floormtwo}}$. We rephrase our problem as a graph problem. 
Let the vertices of the graph $G = (V, E)$ be the $\floormtwo$-subsets of $[m]$ and only place edges between disjoint $\floor{\frac{m}{2}}$-subsets. Now, pairs of $\floor{\frac{m}{2}}$-subsets with an edge connecting them are exactly the pairs of $\floor{\frac{m}{2}}$-subsets of $[m]$ whose elements can occur in the same row of $\Sigma_{\floor{\frac{m}{2}}}$ (since each row is a permutation and therefore contains each element of $[m]$ exactly once). 
The problem of constructing a suitable matrix $\Sigma_{\floor{\frac{m}{2}}}$ becomes equivalent to showing that there exists a subset of edges $M\subseteq E$ such that no two edges $e_1, e_2 \in M$ are incident to the same vertex, with cardinality $\abs{M} = \ell$ if $\binom{m}{\floormtwo}$ is even and $\abs{M} = \ell - 1$ if $\binom{m}{\floormtwo}$ is odd (since in this case, one $\floor{\frac{m}{2}}$-subset of $[m]$ will have its own row in the matrix $\Sigma_{\floor{\frac{m}{2}}}$). Such problems are called \emph{matching} problems in the field of graph theory.

As it turns out, the graph $G$ is known as the \textit{Odd Graph} $O_n$ with $n=\floor{\frac{m}{2}}$ (since $m = 2\floor{\frac{m}{2}} + 1$ when $m$ is odd).
According to \cite{mutze2018sparse}, $O_n$ has at least one Hamiltonian cycle for $n=1$ and for every $n\geq 3$, a Hamiltonian cycle being a cycle which goes through every vertex exactly once. In particular, it has a Hamiltonian path as long as $n \neq 2$. 
This implies that for $n \neq 2$, there exists a matching of size $\floor{ \frac{1}{2} \binom{m}{ \floor{\frac{m}{2}} } }$.
Indeed, it suffices to take one such Hamiltonian path, add the first edge to $M$, skip the next one, and continue adding every other edge to $M$ as we follow along the path. This ensures that every vertex is incident to exactly one of the selected edges, except when the number of vertices is odd, in which case one vertex is left out (thus accounting for the floor function).
The case $n=2$ (which only occurs when $m = 5$) is exceptional and $O_2$ corresponds to the Petersen Graph, which has no Hamiltonian cycle.
However, from Figure~\ref{fig:petersen_graph}, we can see that there still exists a matching of size $\ell=\frac{1}{2} \binom{5}{ \floor{\frac{5}{2}} }=5$.

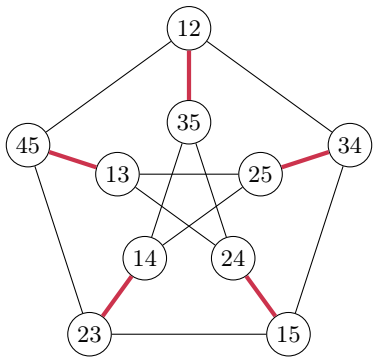
\begin{figure}
\centering
\begin{tikzpicture}
\def\outerradius{2.25}
\def\innerradius{1}
\foreach \angle/\tag in {0/12, 72/45, 144/23, 216/15, 288/34}{
    \node[circle, draw, inner sep=2pt](\tag) at (90+\angle:\outerradius) {\footnotesize\tag};
}
\foreach \angle/\tag in {0/35, 72/13, 144/14, 216/24, 288/25}{
    \node[circle, draw, inner sep=2pt](\tag) at (90+\angle:\innerradius) {\footnotesize\tag};
}
\draw (12) -- (34);
\draw (12) -- (45);
\draw[ultra thick, color3] (12) -- (35);

\draw[ultra thick, color3] (34) -- (25);
\draw (34) -- (15);

\draw (35) -- (24);
\draw (35) -- (14);

\draw[ultra thick, color3] (45) -- (13);
\draw (45) -- (23);

\draw (25) -- (13);
\draw (25) -- (14);

\draw[ultra thick, color3] (15) -- (24);
\draw (15) -- (23);

\draw (24) -- (13);

\draw[ultra thick, color3] (14) -- (23);

\end{tikzpicture}
\caption{The odd graph $O_2$, also commonly known as the Petersen Graph. One matching of size $5$ is shown in bold red.}
\label{fig:petersen_graph}
\end{figure}

We can easily construct a set of $\ell$ permutations which separates all $\floor{\frac{m}{2}}$-subsets from such a matching.
Pair off the $\floor{\frac{m}{2}}$-subsets which are joined by an edge in the chosen matching $M$. 
Since $m$ is odd, choosing a pair of disjoint $\floor{\frac{m}{2}}$-subsets of $[m]$ fixes $m-1$ elements of a permutation, leaving exactly one possible element to complete it.
Hence, sandwich the missing elements between each pair of $\floor{\frac{m}{2}}$-subsets to construct the rows of $\Sigma_{\floor{\frac{m}{2}}}$. 
If $\binom{m}{ \floor{\frac{m}{2}} }$ is even, we are done. 
Otherwise, put the last $\floor{\frac{m}{2}}$-subset at the beginning of the $\ell$-th row of $\Sigma_{\floor{\frac{m}{2}}}$. 
Lastly, if $\ell > \ceil{\frac{1}{2}\tbinom{m}{\frac{m}{2}}}$, then build the first $\ell$ rows of $\Sigma_{\floor{\frac{m}{2}}}$ as described above and fill in the rest with arbitrary permutations.
With this configuration, just like in the even case, $\Sigma_{\floor{\frac{m}{2}}}$ induces a sample $S_{\frac{m}{2}}$ on which every 2-partition is realizable by a stump. 

\textbf{Step 2.}
We want to prove that given a sample $S_\floormtwo$ on which every $2$-partition with a part of size $\floormtwo$ is realizable by a stump, we can construct a sample $S_1$ on which every $2$-partition is realizable by a stump.
We proceed inductively, showing that given a sample $S_k$ with $1 < k \leq \floor{\frac{m}{2}}$ on which every $2$-partition with a part of size $k, k+1, \dots, \floormtwo$ is realizable by a stump, there exists a sample $S_{k-1}$ of the same size as $S_k$ on which every $2$-partition with a part of size $k-1, k, k+1, \dots, \floormtwo$ is realizable by a stump.

Let $S_k$ be a sample such that no two of its instances have the same value for any feature, and let $\Sigma_k$ be its permutation representation. 

Then, Lemma~\ref{lem:injective_mapping_from_Ak_to_Ak+1}, proved below, assures us that there exists an injective map $\phi$ from the set $\binom{[m]}{k-1}$ of all $(k-1)$-subsets of $[m]$ to the set $\binom{[m]}{k}$ of all $k$-subsets of $[m]$ such that for every $(k-1)$-subset $a$, we have $a \subset \phi(a)$.

By assumption, $\Sigma_k$ separates all $k$-subsets of $[m]$, that is all $k$-subsets of $[m]$ appear either as the first $k$ elements or the last $k$ elements of a row of $\Sigma_k$. For some $a \in \binom{[m]}{k-1}$, reorder the elements of $\phi(a)$ appearing at the beginning or the end of a row of $\Sigma_k$ so that the elements of $a$ are either at the beginning or at the end of this row (according to whether the elements of $\phi(a)$ are at the beginning or at the end of the row). Notice that after this procedure, the new matrix $\Sigma'_k$ that is obtained still separates $\phi(a)$; moreover, it also separates $a$. Since the map $\phi$ is injective, we can continue this process without ever needing to reorder the same half-row twice, applying the same steps for each $a \in \binom{[m]}{k-1}$. This yields a final matrix $\Sigma_{k - 1}$ which induces the desired sample $S_{k - 1}$. 

Now, since Lemma~\ref{lem:injective_mapping_from_Ak_to_Ak+1} is valid for $2 \leq k \leq \floor{\frac{m}{2}}$, and because $S_\floormtwo$ is a set on which every $2$-partition with a part of size $\floormtwo$ can be realized by a stump, one can repeat the process above until $k=2$ so that $\R(S_1)$ contains every 2-partition.
Thus, $S_1$ is the set needed to conclude the proof.

\begin{lemma}
\label{lem:injective_mapping_from_Ak_to_Ak+1}
Let $\binom{[m]}{k} \eqdef \left\{ a \subseteq [m] : |a| = k \right\}$ be the set of all $k$-subsets of $[m]$.
Then, for $1 \leq k < \frac{m}{2}$ there exists an injective mapping $\phi : \binom{[m]}{k} \to \binom{[m]}{k+1}$ such that $a \subset \phi(a)$ for all $a \in \binom{[m]}{k}$.
\end{lemma}
\begin{proof}
Let $k$ be such that $1 \leq k < \frac{m}{2}$. Consider the bipartite graph $G = (V, E)$ whose set of vertices is $V = \binom{[m]}{k} \cup \binom{[m]}{k + 1}$, with an edge connecting $a \in \binom{[m]}{k}$ and $b \in \binom{[m]}{k + 1}$ if and only if $a \subset b$, and no other edges. The lemma is equivalent to finding a matching of $G$ which covers $\binom{[m]}{k}$ in the sense that each vertex in $\binom{[m]}{k}$ is incident to an edge of the matching. We show the existence of such a matching using Hall's marriage theorem (see \cite{HallMarriage}). 

Let $W \subseteq \binom{[m]}{k}$ and consider the set $N(W)$ containing all the vertices in $\binom{[m]}{k + 1}$ which are adjacent to a vertex in $W$, that is all $(k + 1)$-subsets of $[m]$ which contain a $k$-subset of $[m]$ from $W$. 

Given $a \in W$, we can make $m - k$ different $(k + 1)$-subsets containing $a$ by adding one of the $m - k$ elements of $[m]$ not present in $a$ to it. Since we can do this for each $a \in W$, we obtain $(m - k)\abs{W}$ (not necessarily all distinct) $(k + 1)$-subsets. In fact, in the worst case, when all $\binom{k + 1}{k} = k + 1$ different $k$-subsets of some $b \in \binom{[m]}{k + 1}$ are present in $W$, $b$ will be counted $k + 1$ times. Therefore $(k + 1)\abs{N(W)} \geq (m - k)\abs{W}$. Moreover, since $1 \leq k < \frac{m}{2}$, we have $m - k \geq k + 1$. This means
\begin{equation*}
\abs{N(W)} \geq \frac{m - k}{k + 1}\abs{W} \geq \abs{W}.
\end{equation*}
Since this inequality holds for all $W \subseteq \binom{[m]}{k}$, a straightforward application of Hall's marriage theorem yields a matching of $G$ which covers $\binom{[m]}{k}$ and proves the lemma. 
\end{proof}

\subsection{Proof of part 4 of Theorem~\ref{thm:ub_partitioning_func_decision_stumps_rl_feat}}
\label{app:proof_m=1to7}

We now prove that
\begin{align}\label{eq:decision_stump_partitioning_function_rl_feat_app}
    \pi^2_T(m) = \frac{1}{2} \sum_{k=1}^{m-1} \min \cb{ 2\ell, \binom{m}{k} }
\end{align}
when $1 \leq m \leq 7$.
To do so, consider the permutation representation $\Sigma$ of a sample $S$ as described at the beginning of the Appendix.
We explicitly define $\Sigma$ which induces a sample $S$ that shows Equation~\eqref{eq:decision_stump_partitioning_function_rl_feat_app} is satisfied. 

One must understand the following matrices as follows.
If $\ell$ is less than or equal to the total number of rows of the matrix, build $\Sigma$ from the first $\ell$ rows.
If $\ell$ is greater than the number of rows of the matrix, add arbitrary permutations to fill out the rest of the rows of $\Sigma$; these do not matter because $\Sigma$ already separates all subsets of $[m]$ with its first $\ell$ rows.

\begin{itemize}
  \item $m=1$:
\begin{align*}
    \begin{bmatrix}
        1
    \end{bmatrix}
\end{align*}

\item $m=2$:
\begin{align*}
    \begin{bmatrix}
        1 & 2
    \end{bmatrix}
\end{align*}

  \item $m=3$:

\begin{align*}
    \begin{bmatrix}
        1 & 2 & 3\\
        1 & 3 & 2
    \end{bmatrix}
\end{align*}

  \item $m=4$:

\begin{align*}
    \begin{bmatrix}
        1 & 2 & 4 & 3\\
        2 & 3 & 1 & 4\\
        1 & 3 & 2 & 4
    \end{bmatrix}
\end{align*}

  \item $m=5$:

\begin{align*}
\begin{bmatrix}
1 & 2 & 3 & 5 & 4\\
2 & 3 & 4 & 1 & 5\\
3 & 4 & 1 & 2 & 5\\
1 & 3 & 5 & 2 & 4\\
1 & 4 & 2 & 3 & 5
\end{bmatrix}
\end{align*}

  \item $m=6$:

\begin{align*}
\begin{bmatrix}
1 & 2 & 3 & 6 & 5 & 4\\
2 & 3 & 4 & 1 & 6 & 5\\
3 & 4 & 5 & 2 & 1 & 6\\
1 & 3 & 6 & 5 & 4 & 2\\
3 & 5 & 2 & 1 & 6 & 4\\
5 & 1 & 4 & 3 & 2 & 6\\
1 & 4 & 3 & 6 & 2 & 5\\
3 & 6 & 5 & 1 & 2 & 4\\
1 & 2 & 5 & 3 & 4 & 6\\
1 & 3 & 5 & 2 & 4 & 6
\end{bmatrix}
\end{align*}

  \item $m=7$:
\begin{align*}
\begin{bmatrix}
1 & 2 & 3 & 4 & 5 & 6 & 7\\
2 & 3 & 4 & 7 & 1 & 5 & 6\\
3 & 4 & 7 & 6 & 2 & 1 & 5\\
4 & 7 & 6 & 2 & 5 & 1 & 3\\
1 & 4 & 3 & 7 & 6 & 2 & 5\\
5 & 7 & 4 & 3 & 2 & 1 & 6\\
3 & 7 & 5 & 6 & 1 & 2 & 4\\
2 & 7 & 4 & 1 & 6 & 3 & 5\\
2 & 6 & 3 & 7 & 1 & 4 & 5\\
1 & 7 & 3 & 5 & 2 & 4 & 6\\
3 & 6 & 7 & 1 & 2 & 4 & 5\\
1 & 4 & 7 & 6 & 2 & 3 & 5\\
1 & 2 & 7 & 3 & 4 & 5 & 6\\
1 & 5 & 7 & 2 & 3 & 4 & 6\\
1 & 6 & 7 & 2 & 3 & 4 & 5\\
2 & 3 & 7 & 5 & 1 & 4 & 6\\
2 & 5 & 7 & 4 & 3 & 6 & 1\\
2 & 6 & 7 & 1 & 3 & 4 & 5
\end{bmatrix}
\end{align*}

\end{itemize}

\section{Proof of the bound for stumps on ordinal features}
\label{app:proof_stump_ordinal}

The proof of Theorem~\ref{thm:ub_partitioning_func_decision_stumps_ordinal_feat} relies on a greedy procedure which considers every threshold split possible and attributes to each one a 2-partitions in an optimal fashion.
We divide the proof in three parts.
In the first one, we motivate the algorithm and proceed to give a textual description of what it does, then the second part proves that it is optimal.
Finally the third one develops the mathematical framework needed to derive a useful implementation of the algorithm used to state the theorem.

\subsection{Motivation and description of the algorithm}

Given a sample $S$ of $m$ examples, define $\R(S)$ as the set of 2-partitions realizable by a single node from threshold split rule set.
Our goal is to upper bound the cardinality of $\R(S)$ independently of $S$ as it will directly imply a bound on the partitioning function, since $\pi^2_T(m,\Ob) = \max_{S:\abs{S}=m} \abs{\R(S)}$.

First of all, notice that the number of times a single feature $i$ can be used to realize a split is limited by the number of categories $O_i$.
More precisely, each feature $i$ yields at most $O_i-1$ distinct 2-partitions of $S$.
As such, an obvious upper bound is 
\begin{equation*}
  \abs{\R(S)} \le \sum_{i=1}^\omega O_i-1.
\end{equation*}

However, this naive analysis is oblivious to the fact that, as for the case of real-valued feature, one still has that each feature can yield at most two distinct 2-partitions with a part of size $k < \frac{m}{2}$ and only one 2-partition with a part of size $k=\frac{m}{2}$.
Hence, this upper bound will most of the time be loose.
We therefore have the additional challenge (compared to the case of real-valued features) to keep track of the used rules in order to better estimate $\abs{\R(S)}$.
We explain the procedure in generality, but we provide an example along to facilitate the comprehension.

Consider a sample with feature landscape $\Ob = (6,6,6,6,5,5,5,2,2)$ (in this section, we assume without loss of generality that the features have been prospectively rearranged so that $O_i \ge O_{i+1}$ for all $i$), which means there are 2 binary features, 3 features with five categories and 4 with six categories.
We can represent the available rules that can be used to make splits by a Ferrers diagram as shown in Figure~\ref{fig:Ferrers_diag} by listing the features from the largest number of categories to the smallest.
Here, each row corresponds to a feature $i$, and each bullet of this row corresponds to a possible split that can be realized using this feature (hence, each row is one less than $O_i$).
Note that Ferrers diagrams are usually used to represent integer partitions.

The numbers on top of the columns correspond to the number of bullets in each column.
We refer to them as the feature landscape conjugate $\Obbar$ of $\Ob$ and its components are defined by 
\begin{equation*}
    \Obar_C = \sum_{i=1}^\omega \Id{O_i-1 \ge C},
\end{equation*}
where $C$ ranges from $1$ to $\Omega-1$ with $\Omega \eqdef \max_i O_i$.
This quantity coincides with the conjugate of the integer partition defined by $\Ob - (1,\dots,1)$, whence the terminology was borrowed.

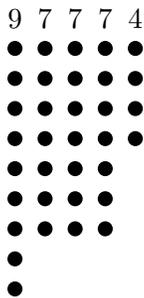
\begin{figure}[h!]
\centering
\begin{tikzpicture}
\def\circleDistance{.4}
\def\featDistConj{9,7,7,7,4}
\foreach \C [count=\i] in \featDistConj {
  \node at (\circleDistance*\i,0) {$\C$};
  \foreach \j in {1,...,\C}{
    \node[Ferrers bullet](\i-\j) at (\circleDistance*\i,-\circleDistance*\j) {};
  }
}




\end{tikzpicture}

\caption{Ferrers diagram of the feature landscape $\Ob = (6,6,6,6,5,5,5,2,2)$. The numbers in the top row count the number of bullets in each column, they form the feature landscape conjugate $\protect\Obbar=(9,7,7,7,4)$.}
\label{fig:Ferrers_diag}
\end{figure}

We now want to attribute a distinct 2-partition to each bullet, respecting the constraint that each feature can only be used twice for each $k< \frac{m}{2}$ (where $k$ is the size of a part) and only once if $k=\frac{m}{2}$.
Since the number of features is limited, they must be attributed with a particular attention in order to maximize the number of 2-partitions that can be realized.
Indeed, we should attribute to each bullet a 2-partition such that the number of available features stays the largest possible.
In the Ferrers diagram, this translates into keeping the number of rows as large as possible at each attribution.

This means that we can develop a greedy procedure that achieves what we want.
Starting from $k = 1$, attribute a 2-partition with a part of size $k$ to each bullet in the last column starting from the bottom going up until all the column is used or all 2-partitions are attributed.
If there are 2-partitions left, go to the second to last column from the bottom going up, respecting the limit of one or two 2-partitions per row (depending on $k$).
Repeat until there is no 2-partition or bullets left, then repeat for every $k$ up to $\floormtwo$.
Figure~\ref{fig:Ferrers_colored} presents the procedure for feature landscape $\Ob$ assuming a sample of $m=6$ examples. 

Remark that the optimality of the solution does not depend on the order of the 2-partitions are attributed; the optimality depends only on the fact that no feature is ``wasted'' (i.e. the rightmost columns are always used first).
This implies that we could also, for example, start by attributing 2-partitions with a part of size $k=\floormtwo$ first, with a part of size $k=\floormtwo-1$ second, down to $k=1$.
This would change the actual attribution, but the total number of 2-partitions attributed would be the same, as depicted in Figure~\ref{fig:Ferrers_colored}.

Note that at each step of the procedure, the Ferrers diagram is valid (\ie still corresponds to some integer partition), and optimally attributes the 2-partitions in order to maximize the number of available rows at all times.

\begin{figure}[h!]
\centering
\begin{tikzpicture}
\node(1) {
    \begin{tikzpicture}
    \def\circleDistance{.4}
    \def\featDistConj{9,7,7,7,4}
    \foreach \C [count=\i] in \featDistConj {
      \node at (\circleDistance*\i,0) {$\C$};
      \foreach \j in {1,...,\C}{
        \node[Ferrers bullet](\i-\j) at (\circleDistance*\i,-\circleDistance*\j) {};
      }
    }
    
    
    
    
    \foreach \ij in {5-1,5-2,5-3,5-4,4-6,4-7}{
      \node[Ferrers bullet, fill=green!80!black] at (\ij) {};
    }
    
    \foreach \ij in {3-1,3-2,3-3,3-4,4-1,4-2,4-3,4-4,4-5,2-6,2-7,3-5,3-6,3-7,1-9}{
      \node[Ferrers bullet, fill=blue!60] at (\ij) {};
    }
    
    \foreach \ij in {1-6,1-7,1-8,2-1,2-2,2-3,2-4,2-5}{
      \node[Ferrers bullet, fill=red!80] at (\ij) {};
    }

    \draw[->, rounded corners=3pt] (2-5.center) -- (2-1.center) -- (1-8.center) -- (1-6.center);
    
    \draw[->, rounded corners=3pt] (4-5.center) -- (4-1.center) -- (3-7.center) -- (3-1.center) -- (2-7.center) -- (2-6.center) -- (1-9.center);
    
    \draw[->, rounded corners=3pt] (5-4.center) -- (5-1.center) -- (4-7.center) -- (4-6.center);
    \end{tikzpicture}
};
\node[below=5mm of 1] {\small Procedure starting at $k=1$.};
\end{tikzpicture}
\hspace{7mm}
\begin{tikzpicture}
\node(1) {
    \begin{tikzpicture}
    \def\circleDistance{.4}
    \def\featDistConj{9,7,7,7,4}
    \foreach \C [count=\i] in \featDistConj {
      \node at (\circleDistance*\i,0) {$\C$};
      \foreach \j in {1,...,\C}{
        \node[Ferrers bullet](\i-\j) at (\circleDistance*\i,-\circleDistance*\j) {};
      }
    }
    
    
    
    
    \foreach \ij in {5-1,5-2,5-3,5-4,4-5,4-6,4-7,1-8,1-9}{
      \node[Ferrers bullet, fill=red!80] at (\ij) {};
    }
    
    \foreach \ij in {3-1,3-2,3-3,3-4,4-1,4-2,4-3,4-4,2-5,2-6,2-7,3-5,3-6,3-7}{
      \node[Ferrers bullet, fill=blue!60] at (\ij) {};
    }
    
    \foreach \ij in {1-6,1-7,2-1,2-2,2-3,2-4}{
      \node[Ferrers bullet, fill=green!80!black] at (\ij) {};
    }
    
    \draw[->, rounded corners=3pt] (5-4.center) -- (5-1.center) -- (4-7.center) -- (4-5.center) -- (1-9.center) -- (1-8.center);

    \draw[->, rounded corners=3pt] (4-4.center) -- (4-1.center) -- (3-7.center) -- (3-1.center) -- (2-7.center) -- (2-5.center);
    
    \draw[->, rounded corners=3pt] (2-4.center) -- (2-1.center) -- (1-7.center) -- (1-6.center);
    \end{tikzpicture}
};
\node[below=5mm of 1] {\small Procedure starting at $k=3$.};
\end{tikzpicture}
\hspace{7mm}
\begin{tikzpicture}
\node(1) {
    \begin{tikzpicture}
    \def\circleDistance{.4}
    \def\featDistConj{9,7,7,7,4}
    \foreach \C [count=\i] in \featDistConj {
      \node at (\circleDistance*\i,0) {$\C$};
      \foreach \j in {1,...,\C}{
        \node[Ferrers bullet](\i-\j) at (\circleDistance*\i,-\circleDistance*\j) {};
      }
    }
    
    
    
    
    \foreach \ij in {3-1,3-2,3-3,4-4,4-5,4-6,4-7}{
      \node[Ferrers bullet, fill=red!80] at (\ij) {};
    }
    
    \foreach \ij in {1-1,1-2,1-3,2-1,2-2,2-3,2-4,2-5,2-6,2-7,3-4,3-5,3-6,3-7}{
      \node[Ferrers bullet, fill=blue!60] at (\ij) {};
    }
    
    \foreach \ij in {1-4,1-5,1-6,1-7,1-8,1-9}{
      \node[Ferrers bullet, fill=green!80!black] at (\ij) {};
    }
    \end{tikzpicture}
};
\node[below=5mm of 1] {\small Non-optimal attribution};
\end{tikzpicture}

\caption{Greedy attribution procedure of the splits for the 2-partitions.
The left diagram presents the procedure as described, starting from $k=1$: All of the 6 partitions with a part of size 1 can be realized (green), all of the 15 partitions with a part of size 2 can be realized (blue), and only 8 out of the 10 partitions with a part of size 3 can be realized (red).
Note that even though the total number of possible splits is larger than the number of possible 2-partitions of $S$, they cannot be all realized because of the constraints.
On the middle, the same procedure is used but starting with $k=3$.
The total number of 2-partitions is the same in both cases, but in this situation, 9 out of 10 partitions with a part of size 3 can be realized and 14 out of the 15 partitions with a part of size 2 can be realized.
Remark that all attribution procedures are not optimal.
The rightmost diagram illustrates that using both binary features (\ie the last two bullets of the first column) for the partitions with a part of size 1 allows the realization of two partitions fewer than the proposed procedure. }
\label{fig:Ferrers_colored}
\end{figure}

It is worth noticing that the attribution procedure corresponds to a method that constructs a hypothetical worst-case sample $S$ which would maximize the number of possible splits made by a stump.

\subsection{Proof of optimality}

To show that the greedy attribution procedure actually achieves an optimal solution to the problem, we need to introduce some notation.
Since there is a \emph{very} large amount of possible attribution, we regroup them in equivalence classes labeled by a vector $\a \in \naturals^\floormtwo$, where each component $a_k$ is the number of 2-partitions with a part of size $k$ present in this attribution.
Furthermore, let $A$ be the set of all vectors $\a$ which corresponds to attributions that satisfies the constraints of the problem.
The number of 2-partitions realizable by a stump is thus bounded by $\max_{\a\in A} \norm{\a}_1$, where $\norm{\a}_1$ denotes the $L_1$ norm.
Let $A^* \eqdef \cb{\a \in A: \norm{\a}_1 = \max_{\a'\in A} \norm{\a'}_1}$ be the set of optimal attributions.
We want to compare the solution provided by our greedy procedure to some optimal solution in $A^*$.
However, since there are possibly multiple solutions in $A^*$, we must choose the correct one for our purpose.

Hence, assume a lexical order on $A^*$ such that for any $\a,\a' \in A^*$, we have $\a > \a'$ if, for some value of $1 \le K \le \floormtwo$, we have $a_k = a_k'$ for all $k < K$ and $a_K > a_K'$.
Then, define $\a^*\in A^*$ such that $\a^* > \a$ for all $\a\in A^*\setminus\{ \a^*\}$ to be the optimal solution we will compare to\footnote{Note that such a (unique) $\a^*$ always exists since for any pair $\a'\ne \a$, we have, exclusively, $\a > \a'$ or $\a' > \a$.}, and let $\a^g$ be the solution found by our greedy procedure starting at $k=1$. 

We proceed by way of contradiction.
Assume that $\a^g$ is not optimal, \ie, $\norm{\a^g}_1 < \norm{\a^*}_1$.
Let $K$ be the smallest index for which $\a^g$ differs from $\a^*$.
We have two cases: $a^g_K < a^*_K$ or $a^g_K > a^*_K$.
We will show that the former is inconsistent with the greedy procedure, while the latter contradicts the optimality of $\a^*$.

Let us inspect the first case.
At every step $k$, our procedure attributes the \emph{maximum} possible number of partitions with a part of size $k$ among the remaining bullets not used in the previous $k-1$ steps.
Therefore, having $a^g_K < a^*_K$ implies that at some previous step $K'< K$, the procedure has made a bad choice that has blocked future attribution.
However, this is impossible since we know that we have attributed the correct number of partitions for every $k$ less than $K$ by our assumption that $\a^*$ is optimal, and because the algorithm selects only bullets which could block other attributions as a last resort (\ie the rightmost bullets are always used first).

For the second case, we need to show that if there exists some $K$ such that $a^g_K > a^*_K$, then it implies that there exists another optimal solution $\a'$ which is better than $\a^*$.
Given that our algorithm has already attributed the optimal number of partitions with a part of size $1$ up to $K-1$ by assumption, and that solution $\a^*$ has not used $a^g_K$ bullets for partitions with a part of size $K$, it means that solution $\a^*$ has either some unused bullets, or has used bullets for partitions with a part of size $K' > K$, or a combination of both.
If solution $\a^*$ has unused bullets, a strictly better solution is to attribute partitions with a part of size $K$ to them, which contradicts the optimality of $\a^*$.
On the other hand, if $\a^*$ has used bullets for partitions with a part of size $K'> K$ which could have been used for partitions with a part of size $K$, then there exists another optimal solution $\a'\in A^*$ which comes in lexical order before $\a^*$, which contradicts the definition of $\a^*$.

\subsection{Implementation of the algorithm}

We here derive and solve recurrence relations to make the greedy procedure into an actionable algorithm.
We start with some notation.

Let $\Obbar^0 \eqdef \Obbar$ and let $\Obbar^k$ be the new feature landscape conjugate after step $k$ of the attribution procedure.
Furthermore, let $\Omega \eqdef \max_i O_i$ denote the largest number of categories encountered so that columns can be labeled by integers $C \in [\Omega-1]$.
In prevision of what is to come, let $\Obar^k_C \eqdef 0$ for any $C > \Omega - 1$.
We want an expression for $R_k$, the number of 2-partitions with a part of size $k$ realizable at step $k$, which will imply $\abs{\R(S)} \le \sum_{k=1}^{\floormtwo}R_k$.
Thus, let $b^k_C$ be the number of 2-partitions with a part of size $k$ attributed in the $C$-th column of the Ferrers diagram at step $k$ so that
\begin{equation}
  R_k = \sum_{C=1}^{\Omega-1} b^k_C.\label{eq:R_k_as_bkC_sum}
\end{equation}
We consider two cases separately: 1) $k<\frac{m}{2}$ and 2) $k=\frac{m}{2}$, because the maximum number of 2-partitions with a part of size $k$ per feature is different in each case.

Consider the case $k<\frac{m}{2}$ first.
The procedure described previously instruct us to attribute the maximum amount of splits possible to the bullets to each column, starting from the last one (which has index $\Omega-1$).
This maximum number of attributions is limited by two constraints: the number of splits remaining and the number of admissible bullets in the column.
The value of $b^k_C$ is thus the minimum between these two values.
The procedure \emph{stops} when all splits have been attributed, or when the number of columns is exhausted (\ie $C=1$).

Clearly, the number of splits remaining is simply the maximum number of 2-partitions with a part of size $k$ minus the number of attributions already done, which implies $b^k_C \le \binom{m}{k} - \sum_{C'=C+1}^{\Omega-1} b^k_{C'}$.
In particular, if $b^k_C$ is equal to this quantity, then the procedure must stop, and $b^k_{C'} = 0$ for all $C'< C$.

The number of admissible bullets in the column is simply the number of rows which have only 0 or 1 attribution, since a single feature cannot yield more than two 2-partitions with a part of size $k$.
Hence, define $a^0_C$, $a^1_C$ and $a^2_C$ as the number of rows (with respect to column $C$) with 0, 1 and 2 colored bullets respectively \emph{before} the attribution step $C$ is done.
This will imply that $b^k_C \le a^0_C + a^1_C$.
We seek a recursive set of equations to solve for these quantities in terms of other known quantities.
As a base case, let $a^0_\Omega = a^1_\Omega = a^2_\Omega = 0$.
First, note that the number of rows with 0 attribution is simply the number of bullets in column $C$ minus the rows with 1 or 2 attributions, hence 
\begin{equation}
  a^0_C = \Obar^{k-1}_C - a^1_C - a^2_C.\label{eq:id_0}
\end{equation}
Second, the number of rows with a single attribution is simply the number of rows with 0 attribution at the previous step $C-1$.
Indeed, the procedure being greedy, we know that the previous step has used all admissible bullets, and therefore \emph{all} rows that had either 0 or 1 attribution have been used and now have 1 or 2 attributions.
(We know this is true for all rows because otherwise the procedure would have stopped due to the exhaustion of available splits.)
Therefore, we have the simple relation
\begin{equation}
  a^1_C = a^0_{C+1}.\label{eq:id_1}
\end{equation}
Third, the number of rows with 2 attributions corresponds to the number of rows that had previously only 1 attribution as we just argued, plus the number of rows that already had 2 attributions:
\begin{equation}
  a^2_C = a^1_{C+1} + a^2_{C+1}.\label{eq:id_2}
\end{equation}

Now, we solve for $a^0_C$ and $a^1_C$.
Start from Equation~\eqref{eq:id_0} and use Equations~\eqref{eq:id_1} and \eqref{eq:id_2} on the right-hand side to obtain
\begin{equation*}
  a^0_C = \Obar^{k-1}_C - a^0_{C+1} - (a^1_{C+1} + a^2_{C+1}).
\end{equation*}
Then, using Equation~\eqref{eq:id_0} with $C$ replaced by $C+1$, the equation simplifies to
\begin{equation*}
  a^0_C = \Obar^{k-1}_C - \Obar^{k-1}_{C+1}.
\end{equation*}
Substituting this value in Equation~\eqref{eq:id_1}, we also have $a^1_C = \Obar^{k-1}_{C+1} - \Obar^{k-1}_{C+2}$, so we conclude that the number of available bullets in column $C$ before the attribution is
\begin{equation*}
  a^0_C + a^1_C = \Obar^{k-1}_C - \Obar^{k-1}_{C+2}.
\end{equation*}

Gathering our results, we have, according to our previous discussion, that $b^k_C$ takes the form
\begin{equation*}
  b^k_C = \min\cb{ \binom{m}{k} - \sum_{C'=C+1}^{\Omega-1} b^k_{C'}, \Obar^{k-1}_C - \Obar^{k-1}_{C+2} }.
\end{equation*}
Remark that at any given step $C$, if $b^k_C$ is equal to the first element of the minimum, then the procedure stops.
Therefore, we have that for all previous steps $C'> C$, $b^k_{C'} = \Obar^{k-1}_{C'} - \Obar^{k-1}_{C'+2}$, and so we have that $\sum_{C'=C+1}^{\Omega-1} b^k_{C'}$ is a telescopic series that sums to $\Obar^{k-1}_C + \Obar^{k-1}_{C+1}$, yielding
\begin{equation}
  b^k_C = \min\cb{ \binom{m}{k} - \Obar^{k-1}_{C+1} - \Obar^{k-1}_{C+2}, \Obar^{k-1}_C - \Obar^{k-1}_{C+2} }.\label{eq:def_bkC}
\end{equation}

From there, we can compute $R_k$ using Equation~\eqref{eq:R_k_as_bkC_sum}.
We have two cases to consider: either the procedure stops from exhausting all available bullets, or exhausting all possible splits.
In the former case, the procedure stops when $C=1$ with $b^k_C$ never equal to the first element of the minimum, and we have
\begin{equation*}
  R_k = \sum_{C=1}^{\Omega-1} \Obar^{k-1}_C - \Obar^{k-1}_{C+2} = \Obar^{k-1}_1 + \Obar^{k-1}_{2}.
\end{equation*}
In the latter case, we trivially have $R_k = \binom{m}{k}$, the number of 2-partitions with a part of size $k$.
It is straightforward to verify that the formula~\eqref{eq:def_bkC} is consistent with this result.
Because the procedure is greedy, $R_k$ is sure to achieve one of these values, but cannot exceed either one, so we conclude that
\begin{equation} \label{eq:def_Rk_stump_ordinal}
  R_k = \min\cb{ \binom{m}{k}, \Obar^{k-1}_1 + \Obar^{k-1}_{2} }.
\end{equation}

This previous equation provides us with an expression for $R_k$ which is independent of the number of attributions $b^k_C$ made in each column.
However, these numbers are need to determine the new updated feature landscape $\Obbar^k$ at the next step $k+1$.
We have the natural trivial relation
\begin{equation*}
  \Obar^k_C = \Obar^{k-1}_C - b^k_C,
\end{equation*}
which by construction of the greedy procedure is always associated with a valid Ferrers diagram.
Let us inspect how we can reexpress $b^k_C$ in terms of other know quantities.

Whenever the procedure stops when we have exhausted all admissible bullets, we have $R_k = \Obar^{k-1}_1 + \Obar^{k-1}_2$ and $b^k_C = \Obar^{k-1}_C - \Obar^{k-1}_{C+2}$ for all $C$.
Thus, we directly have 
\begin{equation}\label{eq:obarkC_1}
  \Obar^k_C = \Obar^{k-1}_C - \Obar^{k-1}_C + \Obar^{k-1}_{C+2} = \Obar^{k-1}_{C+2}.
\end{equation}

Let $\Gamma_k$ be the index of the first column with colored bullets (\ie the last column used in the procedure).
If the procedure has stopped because all possible splits have been attributed, then $R_k = \binom{m}{k}$, and $\Gamma_k$ is the largest $C$ such that $b^k_C = \binom{m}{k} - \Obar^{k-1}_{C+1} - \Obar^{k-1}_{C+2}$.
Using this fact, Equation~\eqref{eq:def_bkC} implies $\binom{m}{k} \le \Obar^{k-1}_{\Gamma_k} + \Obar^{k-1}_{\Gamma_k+1}$.
Hence, a formal definition of $\Gamma_k$ is given by
\begin{equation}\label{eq:def_gammak_k<m/2}
  \Gamma_k \eqdef \max \cb{ 1 \le C \le \Omega-1 : R_k \le \Obar^{k-1}_C + \Obar^{k-1}_{C+1} }.
\end{equation}
Applying Equation~\eqref{eq:def_bkC} then yields
\begin{equation}\label{eq:obark_k<m/2}
    \Obar^k_C =
    \begin{cases}
        \Obar^{k-1}_{C} & \text{if } C < {\Gamma_k},\\
        \Obar^{k-1}_{\Gamma_k} + \Obar^{k-1}_{{\Gamma_k}+1} + \Obar^{k-1}_{{\Gamma_k}+2} - R_k & \text{if } C = {\Gamma_k}\\
        \Obar^{k-1}_{C+2} & \text{otherwise.}
    \end{cases}
\end{equation}

Note that Equations~\eqref{eq:def_gammak_k<m/2} and \eqref{eq:obark_k<m/2} also apply when the procedure has stopped because all the admissible bullets have been used.
Indeed, in this case, we have $R_k = \Obar^{k-1}_1 + \Obar^{k-1}_2 \ge \Obar^{k-1}_C + \Obar^{k-1}_{C+1}$ for any $C$ by defintion of $\Obbar^{k-1}$.
This implies that the definition~\eqref{eq:def_gammak_k<m/2} of $\Gamma_k$ indeed correctly returns the index of the first used column.
Furthermore, this allows us to write $R_k = \Obar^{k-1}_{\Gamma_k} + \Obar^{k-1}_{\Gamma_k+1}$ (this relation must hold even if $\Gamma_k \ne 1$, which happens only when $\Obar^{k-1}_1 = \Obar^{k-1}_2 = \Obar^{k-1}_3$).
Plugging this results in Equation~\eqref{eq:obark_k<m/2}, it is easy to see that it collapses to the expected relation~\eqref{eq:obarkC_1}.
Therefore, the set of Equations~\eqref{eq:def_Rk_stump_ordinal}, \eqref{eq:def_gammak_k<m/2} and \eqref{eq:obark_k<m/2} allows us to compute easily the quantities we are interested in for $k < \frac{m}{2}$.

All these operations have a nice intuitive interpretation in terms of Ferrers diagram.
Figure~\ref{fig:Ferrers_remove_colored} shows how the original procedure can be modified in order to retrieve $R_k$, $\Gamma_k$ and $\Obbar^k$ from a visual standpoint.

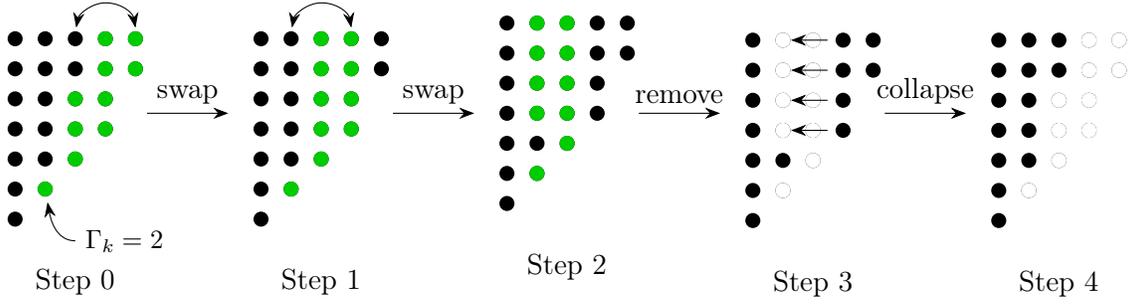
\begin{figure}[h!]
\centering
\begin{tikzpicture}
\def\circleDistance{.4}
\def\featDistConj{7,6,5,4,2}

\def\bulletsIndZero{2-6,3-5,3-4,3-3,4-4,4-3,4-2,4-1,5-2,5-1}
\def\bulletsIndOne{2-6,3-5,3-4,3-3,4-4,4-3,4-2,4-1,3-2,3-1}
\def\bulletsIndTwo{2-6,3-5,3-4,3-3,2-4,2-3,2-2,2-1,3-2,3-1}
\colorlet{bulletColor}{green!80!black}

\node(1) {
    \begin{tikzpicture}[anchor=center]

    \foreach \C [count=\i] in \featDistConj {
      \foreach \j in {1,...,\C}{
        \node[Ferrers bullet](\i-\j) at (\circleDistance*\i,-\circleDistance*\j) {};
      }
    }

    \foreach \ij in \bulletsIndZero{
      \node[Ferrers bullet, fill=bulletColor] at (\ij) {};
    }
    
    \draw[<->] ([yshift=.4mm]5-1.north) arc (10:170:\circleDistance);

    \node[below=3mm of 3-6, anchor=north west, overlay](gammak) {\small $\Gamma_k = 2$};
    \draw[shorten >=2pt, ->, overlay] (gammak) edge[out=180, in=-80] (2-6);
    \end{tikzpicture}
};

\node[right=12mm of 1](2) {
    \begin{tikzpicture}[anchor=center]
    \foreach \C [count=\i] in \featDistConj {
      \foreach \j in {1,...,\C}{
        \node[Ferrers bullet](\i-\j) at (\circleDistance*\i,-\circleDistance*\j) {};
      }
    }

    \foreach \ij in \bulletsIndOne{
      \node[Ferrers bullet, fill=bulletColor] at (\ij) {};
    }
    
    \draw[<->] ([yshift=.4mm]4-1.north) arc (10:170:\circleDistance);
    \end{tikzpicture}
};

\node[right=12mm of 2](3) {
    \begin{tikzpicture}[anchor=center]
    \foreach \C [count=\i] in \featDistConj {
      \foreach \j in {1,...,\C}{
        \node[Ferrers bullet](\i-\j) at (\circleDistance*\i,-\circleDistance*\j) {};
      }
    }

    \foreach \ij in \bulletsIndTwo{
      \node[Ferrers bullet, fill=bulletColor] at (\ij) {};
    }
    \end{tikzpicture}
};

\node[right=12mm of 3](4) {
    \begin{tikzpicture}[anchor=center]
    \foreach \C [count=\i] in \featDistConj {
      \foreach \j in {1,...,\C}{
        \node[Ferrers bullet](\i-\j) at (\circleDistance*\i,-\circleDistance*\j) {};
      }
    }

    \foreach \ij in \bulletsIndTwo{
      \node[Ferrers bullet, fill=white] at (\ij) {};
    }
    
    \draw[<->, opacity=0] ([yshift=.4mm]4-1.north) arc (0:180:\circleDistance);
    
    \foreach \i in {1,2,3,4} {
      \draw[->, shorten <=1mm] (4-\i.west) -- (2-\i.east);
    }
    
    \end{tikzpicture}
};

\node[right=12mm of 4](5) {
    \begin{tikzpicture}[anchor=center]
    \foreach \C [count=\i] in \featDistConj {
      \foreach \j in {1,...,\C}{
        \node[Ferrers bullet](\i-\j) at (\circleDistance*\i,-\circleDistance*\j) {};
      }
    }

    \foreach \ij in \bulletsIndZero{
      \node[Ferrers bullet, fill=white] at (\ij) {};
    }
    
    \draw[<->, opacity=0] ([yshift=.4mm]5-1.north) arc (0:180:\circleDistance);
    \end{tikzpicture}
};

\draw[->, shorten <=2mm, shorten >=2mm] ([xshift=-8pt]1.east) -- (2) node[above, midway]{swap};
\draw[->, shorten <=2mm, shorten >=2mm] ([xshift=-8pt]2.east) -- (3) node[above, midway]{swap};
\draw[->, shorten <=2mm, shorten >=2mm] ([xshift=-8pt]3.east) -- (4) node[above, midway]{remove};
\draw[->, shorten <=2mm, shorten >=2mm] ([xshift=-8pt]4.east) -- (5) node[above, midway]{collapse};

\foreach \i [count=\j] in {0,...,4} {
  \node[below=3mm of \j] {Step \i};
}

\end{tikzpicture}
\caption{The effect of the operations on the original procedure, assuming $\binom{m}{k}=10$ and $\protect\Obbar{}^{k-1} = (7,6,5,4,2)$ for the purpose of the example.
At step 0, the standard attribution procedure is used. $\Gamma_k=2$ is the first used column.
At step 1 and 2, all colored bullets are swapped successively into columns $\Gamma_k$ and $\Gamma_k+1$. These steps correspond to solve the recurrence relations between $a^0_C$, $a^1_C$ and $a^2_C$.
At step 3, we remove bullets that cannot be used again for the next value of $k$.
Finally at step 4, we reconstruct a valid Ferrers diagram by collapsing all remaining bullets to the left in order to obtained the feature landscape conjugate for the next step.
One can see that the first column was left untouched, column $\Gamma_k$ has a new number of bullets, and all subsequent columns are ``shifted'' twice to the left, in accordance with Equation~\eqref{eq:obark_k<m/2}.
Notice that the black bullets of step 0 are the same as at step 3, showing the operations ultimately do ``nothing'' except to provide us with an easy way to compute the new feature landscape.
}
\label{fig:Ferrers_remove_colored}
\end{figure}

We can now proceed with the case where $k = \frac{m}{2}$.
We can use the exact same reasoning as before, but keeping in mind that there are only $\frac{1}{2}\binom{m}{k}$ 2-partitions possible with a part of this size, and that the features can only produce one such 2-partition instead of two.
This implies that rows can only have 0 or 1 attributes, which simplifies things.
Instead of repeating every steps, we only briefly describes the main differences.

The first notable distinction is that we do not need to keep track of rows with 2 attributions when we count the number of admissible bullets in a column.
Thus we only need to consider the variables $a^0_C$ and $a^1_C$ (and not $a^2_C$).
The recurrence relations~\eqref{eq:id_0}, \eqref{eq:id_1} and \eqref{eq:id_2} become instead
\begin{equation*}
  a^0_C = \Obar^{k-1}_C - a^1_C,
\end{equation*}
and
\begin{equation*}
  a^1_C = a^0_{C+1} + a^1_{C+1}.
\end{equation*}
Combining both equations, we end up with $a^0_C = \Obar^{k-1}_C - \Obar^{k-1}_{C+1}$.
Then, Equation~\eqref{eq:def_bkC} becomes
\begin{equation*}
  b^k_C = \min \cb{ \frac{1}{2}\binom{m}{k} - \Obar^{k-1}_{C+1} , \Obar^{k-1}_C - \Obar^{k-1}_{C+1} }.
\end{equation*}
Then, one finds
\begin{equation*}
  R_k = \min \cb{  \frac{1}{2}\binom{m}{k}, \Obar^{k-1}_1 }.
\end{equation*}
Note that since the step $k = \frac{m}{2}$ is the last step of the procedure, we do not need to update the feature landscape $\Obbar^k$.
However, one could if desired, find an update rule using the same reasoning as before which would take the form
\begin{equation*}
    \Obar^k_C =
    \begin{cases}
        \Obar^{k-1}_{C} & \text{if } C < {\Gamma_k},\\
        \Obar^{k-1}_{\Gamma_k} + \Obar^{k-1}_{{\Gamma_k}+1} - R_k & \text{if } C = {\Gamma_k}\\
        \Obar^{k-1}_{C+1} & \text{otherwise,}
    \end{cases}
\end{equation*}
where
\begin{equation*}
    {\Gamma_k} \eqdef \max \cb{ 1 \le C \le \Omega - 1 : \Obar^{k-1}_C \ge R_k }.
\end{equation*}

Gathering our results, we have shown that $\pi^2_T(m, \Ob) = \max_{S:\abs{S}=m}\abs{\R(S)} \le \sum_{k=1}^\floormtwo R_k$, which proves Theorem~\ref{thm:ub_partitioning_func_decision_stumps_ordinal_feat}.

\section{Proof of the bound for stumps on nominal features}
\label{app:proof_stump_nominal}

We here inspect the case of a decision stump using the unitary comparison rule set.
Let $S$ be a sample of $m$ examples drawn from a domain with feature landscape $\Nb$, and define $\R(S)$ to be the set of 2-partitions realizable by the node (in the case of a stump, we have $\R(S) = \P^2_T(S)$, but this is not true for a general tree).
To prove Theorem~\ref{thm:ub_partitioning_func_decision_stumps_nominal_feat}, our goal is to upper bound $\abs{\R(S)}$ as a function of $m$ and $\Nb$.
To this end, we decompose $\R(S)$ into a union of subsets.
This decomposition can be done for the size $k$ of the parts of the 2-partitions realizable, for the feature $i$ on which occurs the split, or both simultaneously, yielding different bounds in each case.
The proof goes as follows.

\begin{proof}
We begin by the decomposition on the part sizes $k$ of the partitions.
Thus, define $\R_k(S)\subset \R(S)$ as the subset of 2-partitions with a part of size $k$ realizable on $S$.
We have
\begin{align*}
    \R(S) = \bigcup_{k=1}^{\floormtwo} \R_k(S).
\end{align*}
We are interested in bounding the cardinality of $\R_k(S)$.
By the same arguments exposed in Appendix~\ref{app:proof_stump_rl}, we have $\abs{\R_k(S)} \le \smallstirling{m}{2}_k$ (where $\smallstirling{m}{2}_k = \binom{m}{k}$ if $k < \frac{m}{2}$ and $\frac{1}{2}\binom{m}{k}$ if $k=\frac{m}{2}$ is the total number of distinct 2-partitions of $S$ with at least a part of size $k$).

To decompose on the feature $i$ used in the splitting process, we define $\R_i(S)\subseteq \R(S)$ as the subset of 2-partitions realizable on feature $i$, which leads to
\begin{align*}
    \R(S) = \bigcup_{i=1}^{\nu} \R_i(S).
\end{align*}
If the number of categories $N_i$ is equal to 1, the node can realize no partitions of the sample.
If $N_i=2$, the node can only realize a single 2-partition.
On the other hand, if $3 \le N_i \le m$, then one can assign each category to at least one example.
The node rule can discriminate every category from the rest, yielding $N_i$ distinct 2-partitions of $S$.
It is impossible to have more than $m$ different 2-partitions on a single feature, therefore for $N_i > m$, $\abs{\R_i(S)}\le m$.
This let us write $\abs{\R_i(S)} \le R_{N_i}$ for
\begin{equation*}
    R_{N} \eqdef \begin{cases}
        N-1 & \text{if } N = 1 \text{ or } 2,\\
        \min\cb{N, m} & \text{otherwise.}
    \end{cases}
\end{equation*}

Finally, one can consider decomposing $\R(S)$ on both $k$ and $i$.
In that case, let $\R_{i,k}(S)\subseteq \R(S)$ be the subset of 2-partitions with a part of size $k$ realizable on feature $i$, so that we have
\begin{align}\label{eq:nominal_parti_decomp_stump}
    \R(S) = \bigcup_{i=1}^\nu \bigcup_{k=1}^{\floormtwo} \R_{i,k}(S).
\end{align}
Here, the $\R_{i,k}(S)$ are most likely not disjoint since the same 2-partitions can occur on two different features.
Let us consider the constraints that exist on $\abs{\R_{i,k}(S)}$.
There are multiple cases possible depending on the value of $\frac{m}{k}$ and $N_i$.

We start with the special case when $k=\frac{m}{2}$ and $N_i > 1$.
In that case, there is a single 2-partition realizable, regardless of the number of categories available.
For the other cases, consider the following procedure where we construct a hypothetical worst-case sample using up to $N_i>0$ categories.
Begin by setting the $i$-th feature of all examples equal to 1, and denote this set by $S'$.
Then, pick and remove any $k$ examples from $S'$, and change their category to 2.
Repeat by picking and removing another set of $k$ examples from $S'$ and set their category to 3.
After $j$ repetitions of this process, we end up with $j$ subsets of $k$ examples, and a single subset $S'$ of $m-jk$ examples, all of them having the $i$-th feature taking different values.
At the end of the procedure, the number of subsets of size $k$ corresponds to the (hypothetical) maximum number of 2-partitions with a part of size $k$ realizable on feature $i$.
There are two possibilities for the procedure to stop: 1) there are no categories left to create a new subset of $k$ examples (\ie $j = N_i-1$), or 2) there are not enough examples left in $S'$ to form a new subset of $k$ examples (while having $\abs{S'}>0$) (\ie $0 < m-jk \le k$).

In the first case, we have $N_i-1 = j$, which implies that the cardinality of $S'$ is equal to $m - k(N_i-1)$.
There are two subcases: if $\abs{S'} = k$, we end up in reality with exactly $j+1 = N_i = \frac{m}{k}$ subsets of $k$ examples, but if $\abs{S'} \ne k$, we have $j = N_i-1$ subsets instead.

In the second case, we have the condition $0 < m - jk \le k$.
Since $j$ is an integer, solving the inequality yields a single positive solution which is given by $j=\floor{\frac{m}{k}}$.

These results can be summarized into the following definition
\begin{equation*}
    R_{N,k} \eqdef \begin{cases}
    1 & \text{if } \frac{m}{k}=2 \text{ and } N \neq 1 ,\\
    N-1 & \text{if } \frac{m}{k} > N,\\
    \floor{\frac{m}{k}} & \text{if } \frac{m}{k} \le N.
    \end{cases}
\end{equation*}
so that we may write $\abs{\R_{i,k}(S)} \le R_{N_i,k}$.

We can apply the union bound to Equation~\eqref{eq:nominal_parti_decomp_stump} in two different ways.
Observe that, since the union operation is commutative, we can use the facts that $\R_k(s) = \bigcup_{i=1}^\nu \R_{i,k}(S)$ and $\R_i(s) = \bigcup_{k=1}^\floormtwo \R_{i,k}(S)$.
Therefore, the first case yields
\begin{align*}
    \abs{\R(S)} &\le \sum_{k=1}^\floormtwo \abs{ \bigcup_{i=1}^\nu \R_{i,k}(S) }\\
    &\le \sum_{k=1}^\floormtwo \min \bigg\{ \stirling{m}{2}_k, \sum_{i=1}^\nu R_{N_i,k} \bigg\}
\end{align*}
while the second case becomes
\begin{align*}
    \abs{\R(S)} &\le \sum_{i=1}^\nu \abs{ \bigcup_{k=1}^\floormtwo \R_{N_i,k}(S) }\\
    &\le \sum_{i=1}^\nu \min \bigg\{ R_{N_i}, \sum_{k=1}^{\floormtwo} R_{N_i,k} \bigg\}
\end{align*}
Combining both results and noting that the bound depends on $S$ only via its size gives the claimed theorem.
\end{proof}

\clearpage

\section{Analysis of general decision tree classes and proofs of theorems}
\label{app:proof_of_ub_partitioning_functions}

In this section, we apply the partitioning framework to the class of general decision trees in order to bound the partitioning functions.
There are multiple cases to consider according to the types of features the examples are made of.
Hence, we subdivide this section into four parts.
In the first one, we examine at the highest level, without making assumptions on the nature of the features, what can be said about the partitions realizable by a tree.
In the subsequent subsections, we then specialize these latter observations to a given feature type.
Thus, the second subsection handles the case when the examples are composed exclusively of real-valued features.
The third subsection treats trees on categorical features.
Finally, the last subsection merges the results of the previous ones in order to obtain a general bound for the partitioning function of decision tree classes on mixture of feature types.

\subsection{Decision trees as partitioning machines}
\label{app:decision_trees_as_partitioning_machines}

In Section~\ref{sec:paritions_as_a_framework}, we introduced the notion of trees as partitioning machines.
We here formalize this idea by providing a recursive construction of partitions realizable by a tree class $T$.
The results presented in this section are agnostic to the nature of the features and will serve as common ground for all the subsequent analyses.

The goal of the present section is to derive an actionable expression for the set $\P^c_T(S)$ of $c$-partitions realizable by a decision tree class $T$ on some sample $S$ in terms of its left and right subtrees.
Hence, let $\R(S)$ be the set of 2-partitions realizable by the root node of $T$.
Remark that $\R(S)$ depends heavily on the decision rules allowed by the tree class; hence we do not make any assumptions on its content for the moment.
Then, let $\pgamma \eqdef \cb{\gamma_1,\dots, \gamma_c} \in \P_T^c(S)$ be some $c$-partition realizable by $T$ and let $\parti{\lambda} \eqdef \cb{\lambda_1, \lambda_2} \in \R(S)$ be a 2-partition realized by the root node which led to that particular $\pgamma$.

We assume that part $\lambda_1$ is forwarded to the left subtree class $T_l$, which produces an $a$-partition $\palpha(\lambda_1)$ while part $\lambda_2$ is sent to the right subtree class $T_r$, which produces a $b$-partition $\pbeta(\lambda_2)$.
The situation is illustrated in Figure~\ref{fig:partition_tree}.
As explained in Section~\ref{sec:paritions_as_a_framework}, $\pgamma$ arises from the union of some of the leaves.
Formally, we say that $\pgamma$ arises from $\palpha$ and $\pbeta$ if the $(a+b)$-partition $\palpha \cup \pbeta$ is a refinement of $\pgamma$, \ie that every part $\gamma\in\pgamma$ is made from the union of some of the parts of $\palpha\cup\pbeta$.
In particular, this directly informs us that $a+b$ must be greater than or equal to $c$.

\begin{figure}[h!]
\centering
\begin{tikzpicture}

\node[draw, circle, minimum width=.5cm](root) {};
\node[above] at (root.north) {root};
\path (root) +(-130:2.35) node[draw, regular polygon, regular polygon sides=3, minimum height=1cm](TL) {\phantom{$T$}};
\node at (TL) {$T_l$};
\path (root) +(-50:2.35) node[draw, regular polygon, regular polygon sides=3, minimum height=1cm](TR) {\phantom{$T$}};
\node at (TR) {$T_r$};
\draw (root) -- node[pos=.5, above, sloped](lambda){\phantom{$\lambda$}} (TL.north);
\node at (lambda) {$\lambda_1$};
\draw (root) -- node[pos=.5, above, sloped](S-lambda){\phantom{$\lambda$}} (TR.north);
\node at ([xshift=2mm]S-lambda) {$\lambda_2$};

\node[below] at (TL.south) {$\palpha$};
\node[below] at (TR.south) {$\pbeta$};

\end{tikzpicture}
\caption{The root node splits the set $S$ into two parts, $\lambda_1$ and $\lambda_2$, which are forwarded to the left subtree class $T_l$ and the right subtree class $T_r$ respectively. The subtrees produces partitions $\palpha$ and $\pbeta$, which can be combined to yield a $c$-partition $\pgamma$.}
\label{fig:partition_tree}
\end{figure}
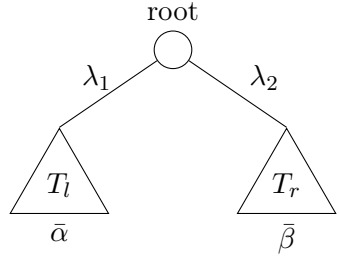

Note that, generally, for any partition $\pgamma$, there exists multiple partitions $\palpha$ and $\pbeta$ from which $\pgamma$ arises.
In particular, we have the following informative claim.
\begin{proposition}\label{prop:existence_palpha_pbeta}
Let $\pgamma \in \P^c_T(S)$ be any realizable $c$-partition, and let $\plambda\in\R(S)$ be any 2-partition made by the root node that led to $\pgamma$.
Then, there exists an $a$-partition $\palpha\in\P^a_{T_l}(\lambda_1)$ and a $b$-partition $\pbeta\in\P^b_{T_r}(\lambda_2)$ from which $\pgamma$ arises such that $a \le c$ and $b \le c$.
\end{proposition}
\begin{proof}
Consider any $c$-partition $\pgamma\in\P^c_T(S)$, any $\plambda\in\R(S)$ that led to $\pgamma$, any $a$-partition $\palpha\in\P^a_{T_l}(\lambda_1)$ and $b$-partition $\pbeta\in\P^b_{T_r}(\lambda_2)$ from which $\pgamma$ arises.
By construction, any part $\gamma_j \in \pgamma$ is the result of the union of some subset of parts $\palpha^j \subseteq \palpha$ and some other subset of parts $\pbeta^j \subseteq \pbeta$.
Note that $\palpha^j$ and $\pbeta^j$ can be empty, but not both at the same time.
Using this notation, we have that $\gamma_j = \bigcup_{\alpha \in \palpha^j} \alpha \cup \bigcup_{\beta \in \pbeta^j} \beta$ for every $\gamma_j$.
Then consider the following partition $\palpha'\eqdef \big\{ \alpha_j' : \alpha_j' \eqdef \bigcup_{\alpha \in \palpha^j} \alpha, \alpha_j' \neq \varnothing \big\}$ and define $\pbeta'$ similarly.
Clearly, $\palpha'\in\P^{a'}_{T_l}(\lambda_1)$ by construction since the operation $\bigcup_{\alpha \in \palpha^j} \alpha$ corresponds to taking the union of some leaves of the tree, and $\pbeta'\in\P^{b'}_{T_r}(\lambda_2)$ for the same reason.
In this formulation, $\gamma_j$ is equal to $\alpha_j'$, or $\beta_j'$, or $\alpha_j' \cup \beta_j'$; since there are $c$ parts in $\pgamma$, this implies that $a'$ and $b'$ are less than or equal to $c$.
Therefore, $\palpha'$ and $\pbeta'$ are the desired witnesses that confirms the claim.
\end{proof}

In the previous paragraphs, we have examined some properties of the partitions $\palpha$ and $\pbeta$ from which any realizable partitions $\pgamma$ arises.
To derive an expression for $\P^c_T(S)$, we now investigate the converse relation, that is, given a partition $\plambda\in\R(S)$, a partition $\palpha\in\P^a_{T_l}(\lambda_1)$ and a partition $\pbeta\in\P^b_{T_r}(\lambda_2)$, what are the possible partitions $\pgamma$ that arise from $\palpha$ and $\pbeta$?
To anwser this question, we take inspiration from the construction of $\palpha'$ and $\pbeta'$ in the proof of Proposition~\ref{prop:existence_palpha_pbeta}, and we define the following quantity.

\begin{definition}\label{def:partitions_set_pairwise_unions}
Let $\palpha$ be an $a$-partition of some set $A$ and $\pbeta$ be a $b$-partition of some other set $B$, disjoint from $A$.
Define the set $\Q^c(\palpha, \pbeta)$ of $c$-partitions that can be constructed from pairwise unions of $\palpha$ and $\pbeta$ as follows:
\begin{align*}
    \Q^c(\palpha, \pbeta) \eqdef \left\{ \right. \pgamma \text{ is a $c$-partition of $A\cup B$} :
    & \text{ for all } \gamma \in \pgamma, \text{ there exists } \alpha \in \palpha, \beta \in \pbeta \nonumber\\
    &\text{ s.t.} \,\gamma = \alpha \text{ or } \gamma = \beta \text{ or } \gamma = \alpha \cup \beta \left. \right\}.
\end{align*}
Furthermore, if $\mathcal{A}^a(A)$ is some set of $a$-partitions of $A$ and $\mathcal{B}^b(B)$ is some set of $b$-partitions of $B$, we denote by
\begin{equation*}
    \Q^c(\mathcal{A}^a(A), \mathcal{B}^b(B)) \eqdef \bigcup_{\substack{\palpha\in\A^a(A),\\\pbeta\in\mathcal{B}^b(B)}} \Q^c(\palpha, \pbeta)
\end{equation*}
the union set of the $\Q^c$.
\end{definition}

We are now equipped to write a recursive relation of the set of partitions a tree $T$ can realize knowing the set of partitions its subtrees can realize.

\begin{proposition}[$c$-partition-set decomposition of decision trees]
\label{prop:c-partitions-set_decomposition_decision_trees}
Let $\P_T^c(S)$ be the set of $c$-partitions that a binary decision tree class $T$ can realize on a sample $S$, and let $T_l$ and $T_r$ be the hypothesis classes of its left and right subtrees.
Moreover, let $\R(S)$ denote the set of 2-partitions the root node can realize on $S$.
Then, the following decomposition holds.
\begin{equation*}
    \P^c_T(S) = 
    \hspace{0pt}
    \bigcup_{\plambda \in \R(S)} \S^c_T(\plambda)
\end{equation*}
where
\begin{equation}\label{eq:def_Sc}
    \S^c_T(\cb{\lambda_1, \lambda_2}) \eqdef \bigcup_{1 \leq a, b \leq c}
     \Q^c\pr{ \P^a_{T_l}(\lambda_1), \P^b_{T_r}(\lambda_2) } \cup \Q^c\pr{ \P^a_{T_l}(\lambda_2), \P^b_{T_r}(\lambda_1) }
\end{equation}
represents the set of realizable $c$-partitions knowning the root node has realized partition $\plambda$.
\end{proposition}
\begin{proof}
Assume $\pgamma \in \P^c_T(S)$ and let $\plambda\in\R(S)$ be the partition that led to $\pgamma$.
Then, Proposition~\ref{prop:existence_palpha_pbeta} ensures us that there exists $\palpha\in\P^a_{T_l}(\lambda_1)$ with $a\le c$ and $\pbeta\in\P^b_{T_r}(\lambda_2)$ and $b\le c$ from which $\pgamma$ arises exactly as defined in Definition~\eqref{def:partitions_set_pairwise_unions} of $\Q^c$.
The same reasoning applies also if the decision rule sends $\lambda_1$ to the right subtree and $\lambda_2$ to the left one instead.
Therefore, since we take the union over all 2-partitions $\plambda$, over all integers $1 \le a, b \le c$, and over both ways of routing to the subtrees the partitioned sample $\plambda$, we necessarily have $\pgamma \in \bigcup_{\plambda\in\R(S)}\S^c_T(\plambda)$.

Now, assume $\pgamma \in \bigcup_{\plambda\in\R(S)}\S^c_T(\plambda)$.
$\pgamma$ is clearly realizable, since by definition of $\Q^c$, it arises in a valid way from some realizable $a$-partition $\palpha$ and some $b$-partition $\pbeta$ of the subtrees.

This implies that $\P^c_T(S)$ is simultaneously a subset and a superset of $\bigcup_{\plambda \in \R(S)} \S^c_T(\plambda)$, which concludes the proof.
\end{proof}

Recall that we are interested in bounding $\abs{\P^c_T(S)}$ for the worse sample $S$ possible to obtain a bound on the partitoning function $\pi^c_T(m)$.
Hence, before tackling each feature types, let us inspect the cardinality of $\S^c$ as it depends only on combinatorial arguments (and not on the nature of the decision rules used).
By the union bound, we have
\begin{equation*}
  \abs{\S^c_T(\plambda)}
  \le
  \hspace{-5pt}
  \sum_{\substack{1 \le a, b \le c\\ a+b \ge c}}
  \hspace{-5pt}
    \abs{\Q^c\pr{ \P^a_{T_l}(\lambda_1), \P^b_{T_r}(\lambda_2)}}
    + \abs{\Q^c\pr{ \P^a_{T_l}(\lambda_2), \P^b_{T_r}(\lambda_1)}},
\end{equation*}
where the condition on $a$ and $b$ follows from the fact that by construction of $\Q^c$, we have $\Q^c(\palpha, \pbeta) = \varnothing$ if $a + b < c$, $a > c$, or $b > c$.
Note that this is only an upper bound as the union over $a$ and $b$ is most likely not disjoint.
Now, analyzing the cardinality of $\Q^c$, we have again by the union bound that
\begin{equation*}
    \abs{\Q^c\pr{ \P^a_{T_l}(\lambda_1), \P^b_{T_r}(\lambda_2) }}
    = \Bigg|
      \bigcup_{
        \substack{\palpha\in \P^a_{T_l}(\lambda_1),\\\pbeta\in \P^b_{T_r}(\lambda_2)}
        }
        \hspace{-9pt}
        \Q^c(\palpha, \pbeta)
    \Bigg|
    = \sum_{\palpha \in \P^a_{T_l}(\lambda_1)} \sum_{\pbeta \in \P^b_{T_r}(\lambda_2)} \big|\Q^c\pr{ \palpha, \pbeta }\big|,\label{eq:cardinality_Qc_union_bound}
\end{equation*}
where here we have equality because the union over $\palpha$ and $\pbeta$ is actually disjoint by definition of $\Q^c$.

To evaluate the size of $\Q^c(\palpha, \pbeta)$, consider the following combinatorial argument.
Choose any $\palpha \in \P^a_{T_l}(\lambda_1)$ and $\pbeta \in \P^b_{T_r}(\lambda_2)$.
According to Definition~\ref{def:partitions_set_pairwise_unions}, we must take the unions of some parts of $\palpha$ and $\pbeta$ to end up with a $c$-partition, with the constraint that the joined parts belongs to different partitions.
We start with a total of $a+b$ parts and we must take the union of some pairs to end up with only $c$ parts.
Taking the union of such a pair effectively reduces the total number of parts by one, therefore we must make $a+b-c$ unions.
To make these unions, choose $a+b-c$ parts from $\palpha$ (there are $\binom{a}{a - (a+b-c)} = \binom{a}{c-b}$ ways to do so) and choose $a+b-c$ parts from $\pbeta$ (there are $\binom{b}{c-a}$ ways to do so) and join them (there are $(a+b-c)!$ ways to do so).
Therefore, we conclude that $\abs{\Q^c(\palpha, \pbeta)} = \binom{a}{c-b} \binom{b}{c-a} (a + b - c)!$, where we have used the fact that $\binom{a}{c-b}$.
Since this argument holds for any partitions, the previous equation becomes
\begin{equation}\label{eq:Qc_cardinality}
    \sum_{\palpha \in \P^a_{T_l}(\lambda_1)} \sum_{\pbeta \in \P^b_{T_r}(\lambda_2)} \big|\Q^c\pr{ \palpha, \pbeta }\big|
    = \tbinom{a}{c-b} \tbinom{b}{c-a} (a+b-c)! \big|\P^a_{T_l}(\lambda_1)\big| \big|\P^b_{T_r}(\lambda_2)\big|.
\end{equation}

Putting everything back together, we end up with
\begin{equation*}
  \abs{\S^c_T(\plambda)}
  \le
  \hspace{-5pt}
  \sum_{\substack{1 \le a, b \le c\\ a+b \ge c}}
  \hspace{-5pt}
    \tbinom{a}{c-b} \tbinom{b}{c-a} (a+b-c)! \pr{
      \big|\P^a_{T_l}(\lambda_1)\big| \big|\P^b_{T_r}(\lambda_2)\big|
      + \big|\P^a_{T_l}(\lambda_2)\big| \big|\P^b_{T_r}(\lambda_1)\big|
    }.
\end{equation*}

Finally, notice that when the left and right subtrees are identical, one has that expression \eqref{eq:def_Sc} collapses to $\S^c(T_l, T_l, \plambda) = \bigcup_{1 \leq a, b \leq c} \Q^c\pr{ \P^a_{T_l}(\lambda_1), \P^b_{T_l}(\lambda_2) }$.
Repeating the same steps to bounds its cardinality, one can readily see that the previous over-estimates the result by a factor of 2.
Hence, we have as a final result
\begin{align}
  \abs{\S^c_T(\plambda)}
  \le
  2^{-\delta_{lr}}
  \hspace{-5pt}
  \sum_{\substack{1 \le a, b \le c\\ a+b \ge c}}
  \hspace{-5pt}
    \tbinom{a}{c-b} \tbinom{b}{c-a} (a+b-c)! \Big(&
      \big|\P^a_{T_l}(\lambda_1)\big| \big|\P^b_{T_r}(\lambda_2)\big| \nonumber \\[-18pt]
      &+ \big|\P^a_{T_l}(\lambda_2)\big| \big|\P^b_{T_r}(\lambda_1)\big|
    \Big),\label{eq:Sc_cardinality}
\end{align}
where $\delta_{lr} \eqdef \Id{T_l = T_r}$ is the Kronecker delta.
We now have all the tools we need to upper bound the partitioning functions of general tree classes for different types of features.

\subsection{Decision trees on real-valued features}
\label{app:proof_decision_trees_rl}

In this section, apply Proposition~\ref{prop:c-partitions-set_decomposition_decision_trees} to trees on real-valued features in order to prove Theorem~\ref{thm:ub_partitioning_functions_decision_trees_rl_feat}.
Therefore, for the rest of this section, $T$ denotes a decision tree class endowed with the threshold split rule set and the examples are made of $\ell$ real-valued features.

\begin{proof}
The first part of Theorem~\ref{thm:ub_partitioning_functions_decision_trees_rl_feat} states that $\pi^c_T(m) = \smallstirling{m}{c}$ for any tree class with $L_T \ge m$.
Thus, assume the number of examples $m$ is less than or equal to the number of leaves $L_T$ of the tree $T$.
We want to show that there exists a sample $S$ such that $T$ can realize every $c$-partitions of $S$.

Let $S$ be a sample such that one feature takes distinct values for each of the $m$ examples.
Then, one can choose for the root of $T$ the appropriate threshold on that feature such that $m_l \le L_{T_l}$ examples will be redirected to the left and $m_r \le L_{T_r}$ examples will be redirected to the right (where we have $L_{T_l} + L_{T_r} = L_T$ and $m_l + m_r = m$).
Then each of the subtrees can do the required split on the same feature, with the required constraints on the number of examples that need to be redirected on each children, until that we have eventually at most one example per leaf.
In that case, by choosing any labeling in $[c]$ for the leaves, the tree class $T$ can perform any $c$-partition of the $m$ examples out of the $\smallstirling{m}{c}$ possible ones.
Consequently, we have $\pi^c_T(m) = \smallstirling{m}{c}$ for any tree class with $L_T \ge m$.

The rest of the Theorem states that when $m > L_T$, we have:
\begin{equation}\label{eq:app_thm_restatement}
    \pi^c_T(m) \leq \pr{\frac{1}{2}}^{\delta_{lr}} \sum_{k=L_{T_l}}^{m-L_{T_r}} \!\!\min \cb{ 2\ell, \tbinom{m}{k} } \hspace{-7pt} {\sum_{\substack{1 \leq a, b \leq c \\ a + b \geq c}}} \!\!\tbinom{a}{c-b} \tbinom{b}{c-a} (a + b - c)!\; \pi^a_{T_l}(k) \pi^b_{T_r}(m-k).
\end{equation}

In the following, we assume every examples of $S$ have distinct feature values, \ie it is always possible to distinguish two examples using any feature.
Indeed, assuming otherwise can only reduce the number of partitions that can be made on a sample, and therefore we have, for such a sample $S$ of $m$ examples, that $\abs{\P_T^c(S)} \le \pi^c_T(m)$.

In the previous section, we have derived Proposition~\ref{prop:c-partitions-set_decomposition_decision_trees}, a recursive union decomposition of $\P^c_T$ in terms of the partitioning sets of its subtrees.
Inspecting the bound~\eqref{eq:Sc_cardinality} on the cardinality of $\S^c_T(\plambda)$, we observe that we could derive a bound for $\pi^c_T(m)$ if we can get rid of the dependence on $\plambda$.
This suggests that we should decompose $\P^c_T(S)$ further over the size of the parts of $\plambda$, hence we let $\R_k(S) \subset \R(S)$ be the subset of 2-partitions realizable by the root node such that at least one part has size $k$, and we have
\begin{equation*}
    \P^c_T(S) = \bigcup_{k=1}^{\floormtwo} \bigcup_{\plambda \in \R_k(S) } \;\S^c_T(\plambda).
\end{equation*}
Because we are dealing with identical features (in contrast to categorical features, where each feature takes values in some possibly different number of categories), there are some optimization that can be made on the starting value of $k$.
For this purpose, we consider the first part of $\S^c$ and we define
\begin{equation*}
  \A_{lr}
  \eqdef
  \bigcup_{k=1}^{\floormtwo}
  \bigcup_{\cb{\lambda_1, \lambda_2} \in \R_k(S)}\;
  \bigcup_{1 \leq a, b \leq c}
  \Q^c\pr{ \P^a_{T_l}(\lambda_1), \P^b_{T_r}(\lambda_2) },
\end{equation*}
where we mute the other dependencies of $\A_{lr}$ to alleviate the notation.
Notice that exchanging the indices $l$ and $r$ in the expression of $\A_{lr}$ gives the missing part of our original expression.
This is because $\Q^c$ is symmetric in its arguments and because the union over $a$ and $b$ is invariant under the exchange of $a$ and $b$, we have that
\begin{equation*}
    \bigcup_{a,b}\, \Q^c \pr{\P^a_{T_l}(\lambda_2), \P^b_{T_r}(\lambda_1)}
    = \,\bigcup_{a,b}\, \Q^c \pr{ \P^a_{T_r}(\lambda_1), \P^b_{T_l}(\lambda_2) },
\end{equation*}
which is equivalent to say that one can exchange the subtrees instead of sending $\lambda_1$ both to the left and to the right.
With this notation, the recursive expression for $\P^c_T(S)$ simplifies to
\begin{equation}\label{eq:PcT_as_union_of_Alr_Arl}
    \P^c_T(S) = \A_{lr} \cup \A_{rl}.
\end{equation}
By the union bound, we have $\abs{\P^c_T(S)} \leq \abs{\A_{lr}} + \abs{\A_{rl}}$.
Therefore, upper bounding $\abs{\A_{lr}}$ will give us the result we are aiming for.

We start by showing that the union over $k$ can be changed to go from $L_{T_l}$ (where $L_{T}$ is the number of leaves of the tree class $T$) to $\min\cb{\floor{\frac{m}{2}}\!, m - L_{T_r}}$ without changing $\A_{lr}$.
To do so, we need to show that for any partition $\pgamma \in \A_{lr}$, there exists at least one 2-partition $\plambda = \{\lambda_1,\lambda_2\}$ realized by the root node with $L_{T_l} \le \abs{\lambda_1} \le \min\cb{\floor{\frac{m}{2}}\!, m - L_{T_r}}$ that leads to $\pgamma$.
Indeed, assume $\abs{\lambda_1} < L_{T_l}$.
Because of our assumption below Equation~\eqref{eq:app_thm_restatement}, one can always modify the threshold of the root node to send $L_{T_l}$ examples in the subtree $T_l$ and modify the subtree so that every example ends up alone in a leaf (as we have shown in the first part of the present proof).
These examples can then be united into the part they belonged in $\pgamma$ to give the same partition as before.
An analogous argument also holds for $\abs{\lambda_2} \geq L_{T_r}$, which implies $\abs{\lambda_1} \le m - L_{T_r}$ (since $m > L_T$ by assumption).

Letting $M_r \eqdef \min\cb{\floor{\frac{m}{2}}, m - L_{T_r}}$ and taking the union bound over $k$ and over $\R_k(S)$, one ends up with
\begin{equation*}
    \abs{\A_{lr}} \leq \sum_{k=L_{T_l}}^{M_r} \abs{\R_k(S)} \max_{\cb{\lambda_1, \lambda_2} \in \R_k(S)} \abs{\bigcup_{1 \leq a, b \leq c}  \Q^c\pr{ \P^a_{T_l}(\lambda_1), \P^b_{T_r}(\lambda_2) }}.
\end{equation*}

We have already discussed the cardinality of $\Q^c$ in the previous section, and it is given by inequality~\eqref{eq:Qc_cardinality}.
Inserting this result, one has
\begin{align*}
    \abs{\A_{lr}} &\leq \sum_{k=L_{T_l}}^{M_r} \abs{\R_k(S)} \hspace{-4pt}\sum_{\substack{1 \leq a, b, \leq c \\ a+b \geq c}}\hspace{-4pt} \tbinom{a}{c-b} \tbinom{b}{c-a} (a + b - c)! \max_{\cb{\lambda_1, \lambda_2} \in \R_k(S)} \big|\P^a_{T_l}(\lambda_1)\big| \big|\P^b_{T_r}(\lambda_2)\big|.
\end{align*}
Then, using Definition~\ref{def:partitioning_function} for $\pi^c_T(m)$ yields
\begin{equation}\label{eq:Alr_bound}
    \abs{\A_{lr}}\leq \sum_{k=L_{T_l}}^{M_r} \abs{\R_k(S)} \hspace{-4pt}\sum_{\substack{1 \leq a, b, \leq c \\ a+b \geq c}}\hspace{-4pt} \tbinom{a}{c-b} \tbinom{b}{c-a} (a + b - c)!\; \pi^a_{T_l}(k) \pi^b_{T_r}(m-k).
\end{equation}
Then, by exchanging indices $l$ and $r$, letting $k \to m-k$, and renaming $a$ to $b$ and $b$ to $a$, we have
\begin{equation}\label{eq:Arl_bound}
    \abs{\A_{rl}}
    \leq
    \sum_{k=M_l}^{m-L_{T_r}}
    \abs{\R_k(S)}
    \hspace{-4pt}
    \sum_{\substack{1 \leq a, b, \leq c \\ a+b \geq c}}
    \hspace{-4pt}
    \tbinom{a}{c-b} \tbinom{b}{c-a} (a + b - c)!\;
    \pi^a_{T_l}(k) \pi^b_{T_r}(m-k),
\end{equation}
where $M_l \eqdef \max\cb{\ceil{\frac{m}{2}}\!, L_{T_l} }$.
Notice that the coefficients inside the sum over $k$ are the same in Equations~\eqref{eq:Alr_bound} and \eqref{eq:Arl_bound}.
For convenience, let
\begin{equation*}
    C_k \eqdef \sum_{\substack{1 \leq a, b, \leq c \\ a+b \geq c}}\hspace{-4pt} \tbinom{a}{c-b} \tbinom{b}{c-a} (a + b - c)!\; \pi^a_{T_l}(k) \pi^b_{T_r}(m-k),
\end{equation*}
so that $\abs{\A_{lr}}$ and $\abs{\A_{rl}}$ can written in the form $\sum_k \abs{\R_k(S)} C_k$, with the only difference being the values that $k$ takes.
We can now show that the sum over $k$ in Equations~\eqref{eq:Alr_bound} and \eqref{eq:Arl_bound} can be put together to yield the theorem.

There are 4 cases to consider according to the values of $M_r = \min\cb{\floor{\frac{m}{2}}, m - \Lr}$ and $M_l = \max\cb{\ceil{\frac{m}{2}}, L_{T_l}}$.
First, let $M_r = \floor{\frac{m}{2}}$ and $M_l = \ceil{\frac{m}{2}}$.
The sum over $k$ then goes from $L_{T_l}$ to $\floor{\frac{m}{2}}$ for $\abs{\A_{lr}}$ and from $\ceil{\frac{m}{2}}$ to $m-L_{T_r}$ for $\abs{\A_{rl}}$.
Then, if $m$ is odd, both sums can be joined directly to go from $L_{T_l}$ to $m-L_{T_{r}}$.
If $m$ is even, one has an extra term for $k = \frac{m}{2}$. Thus
\begin{equation*}
    \abs{\A_{lr}} + \abs{\A_{rl}} \leq \left\{ \begin{array}{ll}\displaystyle
         \sum_{k=L_{T_l}}^{m-L_{T_r}} \abs{\R_k(S)} C_k & \text{if $m$ is odd}\\
         \displaystyle
        \abs{\R_{\frac{m}{2}}(S)} C_{\frac{m}{2}} + \sum_{k=L_{T_l}}^{m-L_{T_r}} \abs{\R_k(S)} C_k & \text{if $m$ is even.}
    \end{array}\right.
\end{equation*}
Using the upper bound on $\abs{\R_k(S)}$ in Equation~\eqref{eq:bound_on_RkS}, the above expression simplifies to
\begin{equation*}
    \abs{\A_{lr}} + \abs{\A_{rl}} \leq \sum_{k=L_{T_l}}^{m-L_{T_r}} \min\cb{2\ell, \tbinom{m}{k}} C_k,
\end{equation*}
valid for both cases.

Second, let $M_r = \min\cb{\floor{\frac{m}{2}}, m - \Lr} = \floor{\frac{m}{2}}$ and $M_l = \max\cb{\ceil{\frac{m}{2}}, L_{T_l}} = L_{T_l}$.
This implies that $L_{T_l} \geq \ceil{\frac{m}{2}}$.
The sum over $k$ then goes from $L_{T_l}$ to $\floor{\frac{m}{2}}$ for $\abs{\A_{lr}}$, which consists in exactly one term if $L_{T_l}=\frac{m}{2}$ and none otherwise.
For $\abs{\A_{rl}}$, the sum over $k$ goes from $\Ll$ to $m-\Lr$.
Therefore, we have
\begin{equation*}
    \abs{\A_{lr}} + \abs{\A_{rl}} \leq \left\{ \begin{array}{ll}
        \displaystyle
        \abs{\R_{\frac{m}{2}}(S)} C_{\frac{m}{2}} + \sum_{k=L_{T_l}}^{m-L_{T_r}} \abs{\R_k(S)} C_k & \text{if $\Ll = \frac{m}{2}$.}\\
        \displaystyle
        \sum_{k=L_{T_l}}^{m-L_{T_r}} \abs{\R_k(S)} C_k & \text{otherwise}
    \end{array}\right.
\end{equation*}
Again, using the upper bound on $\abs{\R_k(S)}$ in Equation~\eqref{eq:bound_on_RkS}, the above expression simplifies to
\begin{equation*}
    \abs{\A_{lr}} + \abs{\A_{rl}} \leq \sum_{k=L_{T_l}}^{m-L_{T_r}} \min\cb{2\ell, \tbinom{m}{k}} C_k,
\end{equation*}
valid for both cases.

Third, let $M_r = \min\cb{\floor{\frac{m}{2}}, m - \Lr} = m-\Lr$ and $M_l = \max\cb{\ceil{\frac{m}{2}}, L_{T_l}} = \ceil{\frac{m}{2}}$.
This case is very similar to the second case, where $\abs{\A_{rl}}$ consists in one or zero term instead of $\abs{\A_{lr}}$.
Thus, the same conclusion applies.

Fourth, let $M_r = \min\cb{\floor{\frac{m}{2}}, m - \Lr} = m-\Lr$ and $M_l = \max\cb{\ceil{\frac{m}{2}}, L_{T_l}} = \Ll$.
This case violates our starting assumption that $m$ is greater than $L_T$.
Hence, we can simply ignore this case.

Collecting our results, one concludes that for all $m > \Ll + \Lr$, we have
\begin{equation}\label{eq:almost_bound_on_picT}
    \abs{\P^c_T(S)} \leq \abs{\A_{lr}} + \abs{\A_{rl}} \leq \sum_{k=L_{T_l}}^{m-L_{T_r}} \min\cb{2\ell, \tbinom{m}{k}} C_k.
\end{equation}
Observe that the right-hand-side of this inequality is independent of $S$.
Therefore, by taking the maximum value over all sample $S$ of size $m$, we have a bound for $\pi^c_T(m)$.

One can improve this result when the left and the right subtrees are the same.
Indeed, in this case $\A_{lr} = \A_{rl}$ so that $\P^c_T(S)$ is simply equal to $\A_{lr}$ according to Equation~\eqref{eq:PcT_as_union_of_Alr_Arl}.
Moreover, the condition that $m  > \Ll + \Lr$ implies $\Lr < \frac{m}{2}$, so that $M_r$ is always equal to $\floor{\frac{m}{2}}$.
Equation~\eqref{eq:Alr_bound} then becomes
\begin{align*}
    \abs{\A_{lr}} \leq \sum_{k=\Ll}^{\floor{\frac{m}{2}}} \abs{\R_k(S)} \hspace{-4pt}\sum_{\substack{1 \leq a, b, \leq c \\ a+b \geq c}}\hspace{-4pt} \tbinom{a}{c-b} \tbinom{b}{c-a} (a + b - c)!\; \pi^a_{T_l}(k) \pi^b_{T_r}(m-k).
\end{align*}
Using the fact that $\R_k(S) = \R_{m-k}(S)$, that $T_l = T_r$, and that the summation over $a$ and $b$ is symmetric, along with the bound of Equation~\eqref{eq:bound_on_RkS} on $\abs{\R_k(S)}$, one can show that
\begin{align*}
    \abs{\P^c_T(S)} \leq \abs{\A_{lr}} \leq \frac{1}{2}\sum_{k=\Ll}^{m-\Lr} \min\cb{2\ell, \tbinom{m}{k}} \hspace{-4pt}\sum_{\substack{1 \leq a, b, \leq c \\ a+b \geq c}}\hspace{-4pt} \tbinom{a}{c-b} \tbinom{b}{c-a} (a + b - c)!\; \pi^a_{T_l}(k) \pi^b_{T_r}(m-k),
\end{align*}
which is different from Equation~\eqref{eq:almost_bound_on_picT} by a factor of $1/2$ only.

We finally obtain the statement of the theorem if we use the indicator function $2^{-\delta_{lr}}$ to handle into a single expression the cases when $T_l$ and $T_r$ are the same or not.
\end{proof}

\subsection{Decision trees on categorical features}
\label{app:proof_decision_trees_cat}

In this section, we address the case where a tree class acts on categorical features.
We first introduce some general notation and prove a lemma that applies to both ordinal and nominal features, before proving the main theorems separately.

Let $\kappa$ be the number of categorical features (where $\kappa = \nu$ in the case of nominal feature or $\omega$ in the case of ordinal features) and let $\Kb \eqdef (K_1, \dots, K_\kappa) \in \naturals^\kappa$ be the feature landscape of the examples such that $K_i$ corresponds to the number of categories available for the $i$-th feature (where $\Kb = \Nb$ in the case of nominal feature or $\Ob$ in the case of ordinal features).
For convenience, we label the categories by an integer so that we can write the domain space as $\X \eqdef \bigtimes_{i=1}^{\kappa} [K_i]$, where the cross denotes the generalized Cartesian product.
While in all practicality it does not makes sense to consider the possibility $K_i=1$ (such features cannot be used to differentiate examples), we do not forbid it as it will prove useful later on.

Let $S \in \X^m$ be a sample of $m\ge 1$ examples.
Notice that, for some sample $S$ and some feature $i$, it is possible that no examples satisfies $x_i = C$ for some $C \in[K_i]$.
In fact, this will always happen when there are less examples than the number of available features.
Therefore, it makes sense to define the \emph{empirical feature landscape} $\Kbhat(S)$ of the sample $S$, where the component $\Khat_i(S)$ corresponds to the number of distinct values feature $i$ actually takes.
Formally, we have $\Khat_i(S) \eqdef \abs{ \cb{ x_i : \x \in S } }$.
One can see that while $K_i$ should be greater than 1, it is very possible that $\Khat_i(S)$ is (and this will be the case when $m=1$).
Furthermore, we have by definition the property that $\Khat_i(S) \le K_i$ for all features $i\in[\kappa]$ and all samples $S$.

This incites us to define some kind of natural order between feature landscapes that are related.
Given two feature landscapes $\Kb$ and $\Kb'$, we say that $\Kb \le \Kb'$ iff for every $i$, we have $K_i \le K_i'$.
This order, often called \emph{component-wise order} or \emph{product order}, is only partial, and therefore it is possible to have two feature landscapes $\Kb$ and $\Kb'$ such that neither $\Kb \le \Kb'$ nor $\Kb \ge \Kb'$ is true.
As such, the symbol $\nleq$ is not equivalent to the symbol $>$ and vice versa.
Using this notation, for any sample $S$, we always have that $\Kbhat(S)\leq \Kb$.

We can incorporate the information conveyed by the feature landscape into the partitioning functions by considering a non-disjoint decomposition under the different achievable feature landscapes.
Hence, we can write
\begin{equation}\label{def:partitioning_func_nominal}
  \pi^c_T(m, \Kb) = \max_{\Kb' \le \Kb} \max_{\substack{S:\abs{S}=m:\\ \Kbhat(S)=\Kb'}} \abs{\P^c_T(S)}.
\end{equation}
This expanded definition allows us to make the dependence on $\Kb$ explicit.
Furthermore, the following lemma that stems from it provides us with a useful identity, and confirms the intuition that more available categories implies more expressiveness.
\begin{lemma}\label{lem:order_part_func}
Let $\Kb$ and $\Kb'$ be feature landscapes such that $\Kb'\le\Kb$.
Then, we have that
\begin{align*}
  \pi^c_T(m,\Kb') \le \pi^c_T(m,\Kb).
\end{align*}
\end{lemma}

\begin{proof}
Starting from Equation~\eqref{def:partitioning_func_nominal}, we have
\begin{align*}
  \pi^c_T(m,\Kb)
    &= \max_{\Kb'' \le \Kb} \max_{\substack{S:\abs{S}=m:\\ \Kbhat(S)=\Kb''}} \abs{\P^c_T(S)}\\
    &= \max\Big\{
        \max_{\Kb'' \le \Kb'} \max_{\substack{S:\abs{S}=m:\\ \Kbhat(S)=\Kb''}} \abs{\P^c_T(S)},
        \max_{\substack{\Kb'' \le \Kb:\\ \Kb'' \nleq \Kb'}} \max_{\substack{S:\abs{S}=m:\\ \Kbhat(S)=\Kb''}} \abs{\P^c_T(S)}
        \Big\}
        \\
    &= \max\Big\{
        \pi^c_T(m,\Kb'),
        \max_{\substack{\Kb'' \le \Kb:\\ \Kb'' \nleq \Kb'}} \max_{\substack{S:\abs{S}=m:\\ \Kbhat(S)=\Kb''}} \abs{\P^c_T(S)} 
        \Big\}
        \\
    &\ge \pi^c_T(m,\Kb'),
\end{align*}
as desired.
\end{proof}

Remark that the lemma does not say anything for the case where neither $\Kb' \le \Kb$ nor $\Kb \le \Kb'$ hold true.
In fact, for such a pair $\Kb$ and $\Kb'$ of feature landscapes, it is possible to find some trees $T$ and $T'$, some numbers of examples $m$ and $m'$ and some number of parts $c$ and $c'$ such that we have on one hand $\pi^c_T(m, \Kb') < \pi^c_T(m, \Kb)$ while on the other hand $\pi^{c'}_{T'}(m', \Kb') > \pi^{c'}_{T'}(m', \Kb)$.
Therefore, we conclude that it is futile to try to enforce any other kind of order between landscapes to eliminate the indeterminate cases, as it would not to provide us with useful information on the behavior of the partitioning functions.

\subsubsection{Decision trees on ordinal features}
\label{app:proof_decision_trees_ordinal}

Before getting into the core of the proof of Theorem~\ref{thm:ub_partitioning_functions_decision_trees_ordinal_feat}, we take some time to investigate the empirical feature landscape of the parts $\lambda_1$ and $\lambda_2$ of the 2-partition $\plambda$ produced by the root node on some sample $S\in\X^m$.
We assume $\X$ has feature landscape $\Ob$, and we let $\Obhat(S)$ be the empirical feature landscape.
We also relabel the categories (always preserving the order between categories) of every features so that the $x_i$ take values in $[\Ohat_i(S)]$ without loss of generality, and that $\cb{x_i:\x \in S} = [\Ohat_i(S)]$, \ie all values between 1 and $\Ohat_i(S)$ are used in the sample.

Suppose a node realizes a partition $\plambda\eqdef\{\lambda_1,\lambda_2\}$ of $S$ using a threshold split with parameters $i'\in[\omega]$, $\theta\in [\Ohat_{i'}(S)-1]$ and $s\in\cb{\pm 1}$.
Let $\lambda_1$ be the set made of examples satisfying $x_{i'}\le\theta$, and $\lambda_2$ be the rest of the sample.
(Here, the sign $s$ can be reinterpreted as a direction such that $\lambda_1$ is sent to the left subtree and $\lambda_2$ to the right one if $s=1$ the other way around if $s=-1$.)
We want find the constraints that apply to $\Obhat(\lambda_1)$ and $\Obhat(\lambda_2)$ and how they relate to $\Obhat(S)$ and to $\Ob$.

Start by observing that by our previous assumption that $\cb{x_i:\x\in S} = [\Ohat_i(S)]$, the number of categories used in $\lambda_1$ for feature $i'$ is exactly $\theta$ (\ie $\Ohat_{i'}(\lambda_1) = \theta$), and is equal to $\Ohat_{i'}(S) - \theta$ for categories used in $\lambda_2$.
Hence, we have $\Ohat_{i'}(\lambda_1) = \theta$ and $\Ohat_{i'}(\lambda_2) = \Ohat_{i'}(S) -\theta$.
However, it is simply impossible to keep track of every possible decision rules that lead to partition $\plambda$, and we would like to restrict our final expressions to be dependent on the feature $i'$ and the size of the parts $\lambda_1$ and $\lambda_2$ only (and thus independent of $s$ and $\theta$).
Fortunately, we can bound $\Ohat_{i'}(\lambda_1)$ and $\Ohat_{i'}(\lambda_2)$ in two ways.

First, notice that since $1 \le \theta \le \Ohat_{i'}(S)-1$, this readily implies that $\Ohat_{i'}(\lambda_1) \le \Ohat_{i'}(S)-1$ and similarly $\Ohat_{i'}(\lambda_2) \le \Ohat_{i'}(S)-1$.
For other features $i\ne i'$, we only have the trivial bound $\Ohat_i(\lambda_j) \le \Ohat_i(S)$.
These inequalities can be combined with the trivial fact that $\Ohat_i(S) \le O_i$ to yield bounds independent on the details of the sample $S$.
Second, remark that for every feature, the number of available categories is also bounded by the number of examples in each part.
Therefore, we have, for both $j=1$ and $j=2$, that
\begin{equation*}
    \Obhat(\lambda_j) \le \Ob^{\abs{\lambda_j},i'},
\end{equation*}
where the components of $\Ob^{k,i'}$ are defined as
\begin{equation*}
    O^{k,i'}_i \eqdef \min\cb{O_i - \Id{i=i'}, k}.
\end{equation*}
Furthermore, one can get rid of the $i'$ dependence by noting that $O_i - \Id{i=i'} \le O_i$.
Thus we define another surrogate $\Ob^k$ with components
\begin{equation*}
    O^k_i \eqdef \min\cb{O_i, k}.
\end{equation*}
Finally, note that the component-wise order that holds between these feature landscapes implies, by Lemma~\ref{lem:order_part_func}, the following chain of identities
\begin{equation}\label{eq:inequality_pi_emp_landscape}
    \pi^c_T\!\pr{\abs{\lambda_j}, \Obhat\pr{\lambda_j}} \le \pi^c_T\!\pr{\abs{\lambda_j}, \Ob^{\abs{\lambda_j},i'}} \le \pi^c_T\!\pr{\abs{\lambda_j}, \Ob^{\abs{\lambda_j}}},
\end{equation}
which will prove useful in the proof of Theorem~\ref{thm:ub_partitioning_functions_decision_trees_ordinal_feat} that follows.

\begin{proof}
Theorem~\ref{thm:ub_partitioning_functions_decision_trees_ordinal_feat} is stated in two parts, both of which start from Proposition~\ref{prop:c-partitions-set_decomposition_decision_trees} that states
\begin{equation}
    \P^c_T(S) = \bigcup_{\plambda\in\R(S)} \S^c_T(\plambda). \label{eq:PcT_bound_1}
\end{equation}
The first part of the theorem is proven by decomposing the union over the size of the parts of $\plambda$ as well as over the feature $i$, while the second part corresponds to decomposing only on the part size.
We consider the former case first.

Let $\R_{i,k}(S)$ be the set of 2-partitions with at least one part of $\plambda$ with size $k$ that are realizable from feature $i$.
Decomposition~\eqref{eq:PcT_bound_1} can thus be written as
\begin{equation}
    \P^c_T(S) = \bigcup_{i=1}^\omega \bigcup_{k=1}^{\floormtwo} \bigcup_{\plambda\in\R_{i,k}(S)} \S^c_T(\plambda).\label{eq:PcT_bound_2}
\end{equation}
Taking the union bound, we have
\begin{equation}
    \abs{\P^c_T(S)} \le \sum_{i=1}^\omega \sum_{k=1}^{\floormtwo} \sum_{\plambda\in\R_{i,k}(S)} \abs{\S^c_T(\plambda)}.\label{eq:PcT_bound_3}
\end{equation}
A bound on the cardinality of $\S^c_T(\plambda)$ has been obtained in Equation~\eqref{eq:Sc_cardinality}.
Combining this last inequality with the definition of the partitioning function, which implies $\abs{\P^c_T(\lambda)} \le \pi^c_T(\abs{\lambda}, \Obhat(\lambda))$, we have
\begin{align*}
    \abs{\S^c_T(\plambda)}
    \hspace{-2.5pt}
    &\le 
    \hspace{-2pt}
    2^{-\delta_{lr}}
    \hspace{-11pt}
    \sum_{\substack{1 \leq a, b \leq c\\ a+b\ge c}}
    \hspace{-8pt}
    K_{ab}^c
    \big(
        \pi^a_{T_l}(\abs{\lambda_1}\!,\Obhat(\lambda_1)) \pi^b_{T_r}(\abs{\lambda_2}\!,\Obhat(\lambda_2))
        +
        \pi^a_{T_l}(\abs{\lambda_2}\!,\Obhat(\lambda_2)) \pi^b_{T_r}(\abs{\lambda_1}\!,\Obhat(\lambda_1))
    \big),
\end{align*}
where $K_{ab}^c \eqdef \tbinom{a}{c-b}\tbinom{b}{c-a}(a+b-c)!$.
Inequality~\eqref{eq:inequality_pi_emp_landscape} tells us that $\Obhat(\lambda) \le \Ob^{\abs{\lambda},i}$, which yields
\begin{align*}
    \abs{\S^c_T(\plambda)}
    \hspace{-1pt}
    &\le 
    \hspace{-1pt}
    2^{-\delta_{lr}}
    \hspace{-9pt}
    \sum_{\substack{1 \leq a, b \leq c\\ a+b\ge c}}
    \hspace{-8pt}
    K_{ab}^c
    \big(
        \pi^a_{T_l}(\abs{\lambda_1}\!,\Ob^{\abs{\lambda_1},i}) \pi^b_{T_r}(\abs{\lambda_2}\!,\Ob^{\abs{\lambda_2},i})
        +
        \pi^a_{T_l}(\abs{\lambda_2}\!,\Ob^{\abs{\lambda_2},i}) \pi^b_{T_r}(\abs{\lambda_1}\!,\Ob^{\abs{\lambda_1},i})
    \big).
\end{align*}
This bound depends on $\plambda$ only via parameters $k$ and $i$, and on $S$ only via parameter $m$.
Hence, let us define
\begin{equation*}
    A_k^i \eqdef
    2^{-\delta_{lr}}
    \hspace{-8pt}
    \sum_{\substack{1 \leq a, b \leq c\\ a+b\ge c}}
    \hspace{-7pt}
    K_{ab}^c
    \pi^a_{T_l}(k,\Ob^{k,i}) \pi^b_{T_r}(m-k,\Ob^{m-k,i})
\end{equation*}
so that by identifying arbitrarily $\lambda_1$ as the part with size $k$ we can write $\abs{\S^c_T(\plambda)} \le A_k^i + A_{m-k}^i$, and therefore inequality~\eqref{eq:PcT_bound_3} simplifies to
\begin{equation*}
    \abs{\P^c_T(S)} \le \sum_{i=1}^\omega \sum_{k=1}^{\floormtwo} \abs{\R_{i,k}(S)} (A_k^i + A_{m-k}^i).
\end{equation*}

It can be argued rather easily that 
\begin{equation*}
    \abs{\R_{i,k}(S)} \le \begin{cases}
        \min \cb{O_i-1, 2} & \text{if } k \ne \frac{m}{2}\\
        \min \cb{O_i-1, 1} & \text{if } k = \frac{m}{2}.
    \end{cases}
\end{equation*}
Indeed, a single feature can realize at most 2 distinct 2-partition with a part of size $k \ne \frac{m}{2}$, and at most 1 if $k=\frac{m}{2}$.
Furthermore, the number of partitions that can be realized is limited by the number of categories available.
Hence, let us define for all $k\in[m-1]$
\begin{equation*}
    R_{i,k} = \min\cb{2^{\delta_{k,\frac{m}{2}}}(O_i-1), 2}.
\end{equation*}
This allows us to write
\begin{align*}
    \abs{\P^c_T(S)} 
    &\le \sum_{i=1}^\omega \sum_{k=1}^{\floormtwo} \abs{\R_{i,k}(S)} (A_{k}^i + A_{m-k}^i)\\
    &= \sum_{i=1}^\omega \pr{\sum_{k=1}^{\floormtwo} \abs{\R_{i,k}(S)} A_{k}^i +  \sum_{k=1}^{\floormtwo} \abs{\R_{i,k}(S)} A_{m-k}^i}\\
    &= \sum_{i=1}^\omega \pr{\sum_{k=1}^{\floormtwo} \abs{\R_{i,k}(S)} A_{k}^i +  \sum_{k=m-1}^{\ceil{\frac{m}{2}}} \abs{\R_{i,m-k}(S)} A_{k}^i}\\
    &\le \sum_{i=1}^\omega \sum_{k=1}^{m-1} R_{i,k} A_{k}^i,
\end{align*}
where at the last step we have used the fact that $\R_{i,m-k}(S)=\R_{i,k}(S)$ and that $\abs{R_{i,k}} \le R_k$ if $k \ne \frac{m}{2}$ and $\abs{R_{i,k}} \le \frac{1}{2}R_k$ otherwise to handle both cases when $m$ is odd and even simultaneously.

Taking the maximum over the samples $S$, this directly gives us a bound on the partitioning functions:
\begin{equation*}
    \pi^c_T(m,\Ob) \le
    2^{-\delta_{lr}}
    \sum_{i=1}^\omega
    \sum_{k=1}^{m-1}
    R_{i,k}
    \hspace{-7pt}
    \sum_{\substack{1 \leq a, b \leq c\\ a+b\ge c}}
    \hspace{-7pt}
    K_{ab}^c
    \pi^a_{T_l}(k,\Ob^{k,i}) \pi^b_{T_r}(m-k,\Ob^{m-k,i}),
\end{equation*}
which is exactly the first bound of Theorem~\ref{thm:ub_partitioning_functions_decision_trees_ordinal_feat}.

We now consider the case where we decompose the union of the original Equation~\eqref{eq:PcT_bound_1} over the part sizes only.
Let $\R_k(S)$ denotes the set of 2-partitions of $\R(S)$ such that at least one part is of size $k$.
Then, instead of Equation~\eqref{eq:PcT_bound_2}, we have
\begin{equation*}
    \P^c_T(S) = \bigcup_{k=1}^{\floormtwo} \bigcup_{\plambda\in\R_{k}(S)} \S^c_T(\plambda).
\end{equation*}
Again by union bound, this implies
\begin{equation*}
    \abs{\P^c_T(S)} \le \sum_{k=1}^{\floormtwo} \sum_{\plambda\in\R_{k}(S)} \abs{\S^c_T(\plambda)}.
\end{equation*}
The same arguments as before shows that for
\begin{equation*}
    A_k \eqdef
    2^{-\delta_{lr}}
    \hspace{-8pt}
    \sum_{\substack{1 \leq a, b \leq c\\ a+b\ge c}}
    \hspace{-7pt}
    K_{ab}^c
    \pi^a_{T_l}(k,\Ob^{k}) \pi^b_{T_r}(m-k,\Ob^{m-k}),
\end{equation*}
we have $\abs{\S^c_T(\plambda)} \le A_k + A_{m-k}$.
Hence, we can write
\begin{equation}\label{eq:PcT_bound_5}
    \abs{\P^c_T(S)} \le \sum_{k=1}^{\floormtwo} \abs{\R_{k}(S)} (A_k + A_{m-k}).
\end{equation}

Observe that $\abs{\R_k(S)}$ cannot be bounded directly by the expression of $R_k$ of Theorem~\ref{thm:ub_partitioning_func_decision_stumps_ordinal_feat} (which bounds the number of 2-partitions with a part of size $k$ realizable by a decision stump on ordinal features).
Indeed, the attribution procedure derived for the stumps implies that $R_k$ depends on $R_{k'}$ for all $k'< k$ (Figure~\ref{fig:Ferrers_colored} gives an example that the $R_k$'s are different when one starts with different values of $k$).
This was of no importance in the case of a stump, since the partition achieved by the root was the final partition.
On the other hand, when the partition is refined by the subtrees, the number of realizable partitions varies according to the size of the parts $k$.
Therefore, an optimal attribution procedure would quickly become intractable.
However, a bound on $R_k$ independent of the attribution order would be an appropriate bound for $\abs{\R_k(S)}$.

This can be solved readily by noting that $\Obar^k_C \le \Obar_C$ for any $C$ and $k$, where $\Obbar$ is the feature landscape conjugate and $\Obbar^k$ is the feature landscape conjugate at step $k$ of the attribution procedure.
Hence, we define
\begin{equation*}
    R_k' \eqdef \begin{cases}
        \min\cb{\Obar_1 + \Obar_2, \binom{m}{k}} & \text{if } k < \frac{m}{2}\\
        \min\cb{2\Obar_1, \binom{m}{k}} & \text{if } k = \frac{m}{2},
    \end{cases}
\end{equation*}
which satisfies $R_k \le R_k'$ for all $k < \frac{m}{2}$ and $R_k \le \frac{1}{2} R_k'$ for $k = \frac{m}{2}$.
Thus, by our previous arguments, the properties of $R_k'$ also implies $\abs{\R_k(S)} \le R_k'$ for all $k < \frac{m}{2}$ and $\abs{\R_k(S)} \le \frac{1}{2} R_k'$ for all $k < \frac{m}{2}$ and $\abs{\R_k(S)} \le \frac{1}{2} R_k'$.
Substituting this expression in Equation~\eqref{eq:PcT_bound_5}, and using the fact that $\R_k(S) = \R_{m-k}(S)$,
\begin{align*}
    \abs{\P^c_T(S)}
    &\le \sum_{k=1}^{\floormtwo} \abs{\R_{k}(S)} (A_k + A_{m-k})\\
    &= \sum_{k=1}^{\floormtwo} \abs{\R_k(S)} A_{k} +  \sum_{k=1}^{\floormtwo} \abs{\R_k(S)} A_{m-k}\\
    &= \sum_{k=1}^{\floormtwo} \abs{\R_k(S)} A_{k} +  \sum_{k=m-1}^{\ceil{\frac{m}{2}}} \abs{\R_{m-k}(S)} A_{k}\\
    &\le \sum_{k=1}^{m-1} R_k' A_{k},
\end{align*}
where the definition of $R_k'$ handles both cases of $m$ odd and even at the last inequality.
Taking the maximum over the samples $S$ on both sides, we have shown the following result:
\begin{equation*}
    \pi^c_T(m,\Ob) \le
    2^{-\delta_{lr}}
    \sum_{k=1}^{m-1}
    R_k'
    \hspace{-6pt}
    \sum_{\substack{1 \leq a, b \leq c\\ a+b\ge c}}
    \hspace{-7pt}
    K_{ab}^c
    \pi^a_{T_l}(k,\Ob^k) \pi^b_{T_r}(m-k,\Ob^{m-k}),
\end{equation*}
which concludes the proof.
\end{proof}

\subsubsection{Decision trees on nominal features}
\label{app:proof_decision_trees_nominal}

The proof for the bound on the partitioning functions of decision trees on nominal features is similar to that on ordinal features with multiple small modifications.
We again discuss first the empirical landscape of a 2-partition originating from the root node before presenting the proof of Theorem~\ref{thm:ub_partitioning_functions_decision_trees_nominal_feat}.

In our proof, we will need to know how the feature landscape changes when splitting the sample.
Therefore, we examine the feature landscape of the parts of the partitions produced by a single node on some sample $S\in\X^m$.
We assume $\X$ has feature landscape $\Nb$, and we let $\Nbhat(S)$ be the empirical feature landscape.
We also assume that the features $x_i$ take value in $[\Nhat_i(S)]$ without loss of generality; if they do not, simply relabel the categories.
This will ensure us that $\cb{x_i : \x \in S} = [\Nhat_i(S)]$.

Suppose a node realizes a partition $\plambda\eqdef\{\lambda_1,\lambda_2\}$ of $S$ using some decision rule $\phi \in \Phi$, where $\Phi = \cb{ \phi_{i,C,s}(\x) = s \cdot \IdSign{x_i = C} : i \in [\nu], C \in [\Nhat_{i}(S)], s \in \cb{\pm 1} }$ is the unitary comparison rule set used for nominal features.
For the purpose of the present argumentation, we let $\lambda_1$ be the subset of the examples that satisfy $x_{i'} = C$ for some ${i'}\in[\nu]$ and some $C \in [\Nhat_{i'}(S)]$, and let $\lambda_2$ denote the rest of the sample.
Here, the role of $s$ can be interpreted as a ``routing'' parameter, determining to which subtrees $\lambda_1$ and $\lambda_2$ are sent.
We want to relate $\Nbhat(\lambda_1)$ and $\Nbhat(\lambda_2)$ to $\Nbhat(S)$ and to $\Nb$ and to find the constraints that apply to them.

Observe that the decision rule implies that the feature $i'$ of the examples of $\lambda_1$ take a single value, and thus $\Nhat_{i'}(\lambda_1) = 1$.
Similarly, feature $i'$ of the examples of $\lambda_2$ take all values in $[\Nhat_i(S)]$ except for $C$, hence $\Nhat_{i'}(\lambda_2) = \Nhat_{i'}(S) - 1$.
For other features $i\ne i'$, either the examples of $\lambda_j$ take all values of categories or not, and so we only have the trivial bound $\Nhat_i(\lambda_j) \le \Nhat_i(S)$.
These inequalities can be combined with the trivial fact that $\Nhat_i(S) \le N_i$ to yield bounds independent on the details of the sample $S$.
Finally, remark that for every feature, the number of available categories is naturally bounded by the number of examples in each part.
Therefore, we have that
\begin{equation}\label{eq:Nbhat_lambda_def}
    \Nbhat(\lambda_1) \le \Nb^{\abs{\lambda_1},i',1}
    \qquad \text{and} \qquad
    \Nbhat(\lambda_2) \le \Nb^{\abs{\lambda_2},i',N_{i'}-1}
\end{equation}
where the components of $\Nb^{k,i',N}$ are defined as
\begin{equation*}
    N^{k,i',N}_i \eqdef \min\cb{N_i - (N_i - N)\Id{i = i'}, k},
\end{equation*}
where $N$ represents the new number categories used for feature $i'$.
Furthermore, since $1 \le N_{i'} - 1$ (indeed, no partitioning of the sample can occur if $N_{i'} \le 1$), we also have that $\Nb^{\abs{\lambda_1},i',1} \le \Nb^{\abs{\lambda_1},i',N_{i'}-1}$.
Finally, to alleviate the notation, define $\Nbhat^{k,i'} \eqdef \Nbhat^{k,i',N_{i'}-1}$.
Then, by Lemma~\ref{lem:order_part_func}, the following chain of useful identities can be obtained for $\lambda_1$:
\begin{equation}\label{eq:inequality_pi_emp_landscape_lambda_1}
    \pi^c_T\!\pr{\abs{\lambda_1}, \Nbhat\pr{\lambda_1}} \le \pi^c_T\!\pr{\abs{\lambda_1}, \Nb^{\abs{\lambda_1},i',1}} \le \pi^c_T\!\pr{\abs{\lambda_1}, \Nb^{\abs{\lambda_1},i'}}.
\end{equation}
Similarly, for $\lambda_2$, we have that
\begin{equation}\label{eq:inequality_pi_emp_landscape_lambda_2}
    \pi^c_T\!\pr{\abs{\lambda_2}, \Nbhat\pr{\lambda_2}} \le \pi^c_T\!\pr{\abs{\lambda_2}, \Nb^{\abs{\lambda_2},i'}}.
\end{equation}

With these identities at hand, we can proceed with the proof of Theorem~\ref{thm:ub_partitioning_functions_decision_trees_nominal_feat}.

\begin{proof}
Similarly to Theorem~\ref{thm:ub_partitioning_functions_decision_trees_ordinal_feat} (upper bound for ordinal features), Theorem~\ref{thm:ub_partitioning_functions_decision_trees_nominal_feat} is also stated in two parts derived in two ways depending on the union decomposition of Proposition~\ref{prop:c-partitions-set_decomposition_decision_trees}.
We again proceed in an analogous fashion.

We start by decomposing Proposition~\ref{prop:c-partitions-set_decomposition_decision_trees} over features $i$ part sizes $k$.
Thus, define $\R_{i,k}(S)$ as the set of 2-partitions realizable by the root node on feature $i$ and such that at least one part has size $k$ so that we can write
\begin{align*}
    \P^c_T(S) = 
        \bigcup_{k=1}^{\floormtwo}
        \bigcup_{i=1}^\nu
        \bigcup_{\plambda \in \R_{i,k}(S)}
        \hspace{-8pt}
        \S^c_T(\plambda).
\end{align*}
Applying the union bound, this simplifies it to
\begin{align}
    \abs{\P^c_T(S)}
      &\le
        \sum_{k=1}^{\floormtwo}
        \sum_{i=1}^\nu
        \sum_{\plambda \in \R_{i,k}(S)}
        \abs{\S^c_T(\plambda)} \nonumber \\
      &\le
        \sum_{k=1}^{\floormtwo}
        \sum_{i=1}^\nu
        \abs{\R_{i,k}(S)\vphantom{\plambda}}
        \max_{\plambda \in \R_{i,k}(S)}
        \abs{\S^c_T(\plambda)}.\label{eq:tree_2}
\end{align}
According to our previous analysis of the cardinality of $\S^c_T$ given in Equation~\eqref{eq:Sc_cardinality}, and using the fact that $\abs{\P^c_T(S)} \le \pi^c_T(m, \Nbhat(S))$ for all $S$ with $m$ examples following feature landscape $\Nbhat(S)$, we have
\begin{align*}
    \abs{\S^c_T(\plambda)}
    \hspace{-2.5pt}
    &\le 
    \hspace{-2pt}
    2^{-\delta_{lr}}
    \hspace{-11pt}
    \sum_{\substack{1 \leq a, b \leq c\\ a+b\ge c}}
    \hspace{-8pt}
    K_{ab}^c
    \big(
        \pi^a_{T_l}(\abs{\lambda_1}\!,\Nbhat(\lambda_1)) \pi^b_{T_r}(\abs{\lambda_2}\!,\Nbhat(\lambda_2))
        +
        \pi^a_{T_l}(\abs{\lambda_2}\!,\Nbhat(\lambda_2)) \pi^b_{T_r}(\abs{\lambda_1}\!,\Nbhat(\lambda_1))
    \big),
\end{align*}
where $K_{ab}^c \eqdef \tbinom{a}{c-b}\tbinom{b}{c-a}(a+b-c)!$.
As argued previously, Equation~\eqref{eq:Nbhat_lambda_def} states that
\begin{equation*}
    \Nbhat(\lambda_1) \le \Nb^{\abs{\lambda_1},i,1}
    \qquad \text{and} \qquad
    \Nbhat(\lambda_2) \le \Nb^{\abs{\lambda_2},i,N_{i}-1}
\end{equation*}
hold when $\lambda_1$ is made from all the examples satisfying $x_{i} = C$ for some $C$ (let us call this possibility case 1), or the converse when $\lambda_2$ is made this way instead (case 2).
Therefore, using the fact that $\abs{\lambda_1} = k$ and $\abs{\lambda_2} = m-k$, and identities \eqref{eq:inequality_pi_emp_landscape_lambda_1} and \eqref{eq:inequality_pi_emp_landscape_lambda_2}, we have
\begin{equation*}
    \pi^a_{T_l}(k,\Nbhat(\lambda_1)) \pi^b_{T_r}(m-k,\Nbhat(\lambda_2))
    \le \Pi^{ab}_{lr}(k, i, \Nb)
\end{equation*}
with
\begin{align*}
    \Pi^{ab}_{lr}(k, i, \Nb) \eqdef
    \max\Big\{&
        \pi^a_{T_l}(k, \Nb^{k,i,1}) \pi^b_{T_r}(m-k,\Nb^{m-k,i,N_{i}-1})
        ,\\
        &
        \pi^a_{T_l}(k,\Nb^{k,i,N_{i}-1}) \pi^b_{T_r}(m-k, \Nb^{m-k,i,1})
    \Big\},
\end{align*}
where $\Pi^{ab}_{lr}$ is constructed by taking the worse possibility between case 1 and case 2.
A similar results holds for $\pi^a_{T_l}(m-k,\Nbhat(\lambda_2)) \pi^b_{T_r}(k,\Nbhat(\lambda_1))$ (simply by exchanging index $l$ with $r$ and index $a$ with $b$ in the definition of $\Pi$), which implies
\begin{align*}
    \abs{\S^c_T(\plambda)}
    &\le
    2^{-\delta_{lr}}
    \hspace{-7pt}
    \sum_{\substack{1 \leq a, b \leq c\\ a+b\ge c}}
    \hspace{-7pt}
    \tbinom{a}{c-b}\tbinom{b}{c-a}(a+b-c)!
    \pr{\Pi^{ab}_{lr}(k, i, \Nb) + \Pi^{ba}_{rl}(k, i, \Nb)}.
\end{align*}
Note that this expression depends on $\plambda$ only via parameters $i$ and $k$.
Hence, Equation~\eqref{eq:tree_2} simplifies to
\begin{align}
    \abs{\P^c_T(S)}
        &\le
        2^{-\delta_{lr}}
        \sum_{k=1}^{\floormtwo}
        \sum_{i=1}^\nu
        R_{N_i,k}
        \hspace{-7pt}\sum_{\substack{1 \leq a, b \leq c\\
                                 a+b\ge c}} \hspace{-7pt}
        K_{ab}^c
        \pr{
        \Pi^{ab}_{lr}(k, i, \Nb) + \Pi^{ba}_{rl}(k, i, \Nb)},\label{eq:tree_3}
\end{align}
where we used the fact that $\abs{\R_{i,k}(S)} \le R_{N_i,k}$ according to Section~\ref{app:proof_stump_nominal} (the proof for decision stumps on nominal features).
This expression is the tightest bound independent of $S$ we could work out, but it is quite unwieldy.
One can simplify a little bit further by bounding $\Pi^{ab}_{lr}(k, i, \Nb)$.
Note that by definition of the order between feature distributions, we have, as explained at the beginning of the subsection, $\Nb^{k,i,1} \le \Nb^{k,i,N_i-1} \eqdef \Nb^{k,i}$.
This implies
\begin{equation*}
    \Pi^{ab}_{lr}(k, i, \Nb) \le \pi^a_{T_l}(k,\Nb^{k,i}) \pi^b_{T_r}(m-k,\Nb^{m-k,i}).
\end{equation*}
From this observation, we define
\begin{equation*}
    A^i_k \eqdef
    \hspace{0pt}
    \sum_{\substack{1 \leq a, b \leq c\\ a+b\ge c}}
    \hspace{-7pt}
    K_{ab}^c \pi^a_{T_l}(k,\Nb^{k,i}) \pi^b_{T_r}(m-k,\Nb^{m-k,i})
\end{equation*}
so that $\sum_{a,b} K_{ab}^c \Pi_{lr}^{ab}(k,i,\Nb) \le A^i_k$.
Remark that we also have $\sum_{a,b} K_{ab}^c \Pi_{rl}^{ba}(k,i,\Nb) \le A^i_{m-k}$ simply by the symmetries of $A^i_k$, and thus Equation~\eqref{eq:tree_3} takes the form
\begin{align}
    \abs{\P^c_T(S)}
        &\le
        2^{-\delta_{lr}}
        \sum_{i=1}^\nu
        \sum_{k=1}^{\floormtwo}
        R_{N_i,k} \pr{A^i_k + A^i_{m-k}}
        .\label{eq:tree_4}
\end{align}

In anticipation of the following steps, we define the quantity
\begin{equation*}
    R'_{N,k} \eqdef \begin{cases}
    N-1 & \text{if } \frac{m}{k} > N,\\
    \floor{\frac{m}{k}} & \text{if } \frac{m}{k} \le N.
    \end{cases}
\end{equation*}
This is almost the same as $R_{N,k}$ (as defined in Theorem~\ref{thm:ub_partitioning_func_decision_stumps_nominal_feat}) except for the case where $k = \frac{m}{2}$, where $R_{N_i,k}'$ is twice $R_{N,k}$.
Equation~\eqref{eq:tree_4} becomes
\begin{align*}
    \abs{\P^c_T(S)}
        &\le 2^{-\delta_{lr}} \sum_{i=1}^\nu \sum_{k=1}^{\floormtwo} R_{N_i,k} \pr{A^i_k + A^i_{m-k}}\\
        &= 2^{-\delta_{lr}} \sum_{i=1}^\nu \pr{ \sum_{k=1}^{\floormtwo} R_{N_i,k} A^i_k + \sum_{k=1}^{\floormtwo} R_{N_i,k} A^i_{m-k}}\\
        &= 2^{-\delta_{lr}} \sum_{i=1}^\nu \pr{\sum_{k=1}^{\floormtwo} R_{N_i,k} A^i_k + \sum_{k=\ceil{\frac{m}{2}}}^{m-1} R_{N_i,m-k} A^i_{k}}\\
        &= 2^{-\delta_{lr}} \sum_{i=1}^\nu \sum_{k=1}^{m-1} R'_{N_i,\min\cb{k,m-k}} A^i_k.
\end{align*}
where we have handled both cases when $m$ is odd and even simultaneously from the definition of $R'_{N,k}$.
Putting everything together, we have
\begin{equation*}
    \abs{\P^c_T(S)}
        \le
        2^{-\delta_{lr}}
        \sum_{i=1}^\nu \sum_{k=1}^{m-1} 
        R'_{N_i,\min\cb{k,m-k}}  \hspace{-7pt}
        \sum_{\substack{1 \leq a, b \leq c\\ a+b\ge c}} \hspace{-7pt}
        K_{ab}^c\,
        \pi^a_{T_l}(k,\Nb^{k,i}) \pi^b_{T_r}(m-k,\Nb^{k,i}).
\end{equation*}
Taking the maximum over the sample $S$ on both sides yields the first bound of Theorem~\ref{thm:ub_partitioning_functions_decision_trees_nominal_feat}.

The second bound of the theorem is obtained by decomposing Proposition~\ref{prop:c-partitions-set_decomposition_decision_trees} on the part sizes only.
The steps follow the same line of reasoning with some technicalities.
Hence, we only give here the series of inequalities that lead to the theorem without further explanations.
This yields
\begin{align*}
\abs{\P^c_T(S)}
    &\le
        \sum_{k=1}^{\floormtwo}
        \sum_{\plambda \in \R_k(S)}
        \hspace{-7pt}
        \abs{\S^c_T(\plambda)}\\
    &\le
        2^{-\delta_{lr}}
        \hspace{-3pt}
        \sum_{k=1}^{\floormtwo}
        \sum_{\plambda \in \R_k(S)}
        \sum_{\substack{1 \leq a, b \leq c\\ a+b\ge c}}
        \hspace{-7pt}
            K_{ab}^c
            \Big(
            \pi^a_{T_l}(\abs{\lambda_1}\!,\Nbhat(\lambda_1)) \pi^b_{T_r}(\abs{\lambda_2}\!,\Nbhat(\lambda_2))\nonumber\\[-14pt]
            &\hspace{130pt}
            +
            \pi^a_{T_l}(\abs{\lambda_2}\!,\Nbhat(\lambda_2)) \pi^b_{T_r}(\abs{\lambda_1}\!,\Nbhat(\lambda_1))
            \Big)\\
    &\le
        2^{-\delta_{lr}}
        \hspace{-4pt}
        \sum_{k=1}^{\floormtwo}
        \hspace{-1pt}
        \abs{\R_k(S)}
        \hspace{-8pt}
        \sum_{\substack{1 \leq a, b \leq c\\ a+b\ge c}}
        \hspace{-8pt}
            K_{ab}^c
            \Big(
            \pi^a_{T_l}(k,\Nb^k) \pi^b_{T_r}(m-k,\Nb^k)
            \hspace{-1pt}+\hspace{-1pt}
            \pi^a_{T_l}(m-k,\Nb^k) \pi^b_{T_r}(k,\Nb^k)
            \Big)\\
    &\le
        2^{-\delta_{lr}}
        \hspace{-3pt}
        \sum_{k=1}^{m-1}
        \min\cb{\stirling{m}{2}_k\!,\, \nu\! \textstyle\floor{\frac{m}{\min\cb{k,m-k}}}}
        \hspace{-2pt}
        \sum_{\substack{1 \leq a, b \leq c\\ a+b\ge c}}
        \hspace{-7pt}
        K_{ab}^c
        \Big( \pi^a_{T_l}(k,\Nb^k) \pi^b_{T_r}(m-k,\Nb^k) \Big),
\end{align*}
where at the third line we used the fact that $\Nbhat(\lambda_j) \le \Nb^{k,i} \le \Nb^k$ with $N^k_i = \min\cb{N_i, k}$, and where the bound on $\abs{\R_k(S)}$ stems from the bound for the decision stump found in Section~\ref{app:proof_stump_nominal}.
The ultimate line holds as a bound for $\pi^c_T(m, \Nb)$ since it does not depend on the sample $S$ other than on its size $m$ and the feature landscape $\Nb$.
This concludes the proof.
\end{proof}

\subsection{Decision trees on a mixture of feature types}
\label{app:proof_decision_trees_mixture}

Our last result, Theorem~\ref{thm:ub_partitioning_functions_decision_trees_mixture_feat}, is a bound on the partitioning functions of a tree on a mixture of features types.
The proof of the bound is straightforward to obtain as a corollary of the three previous theorem.
Before proceeding with the proof, we introduce some notation so that everything in the proof is well defined.

We assume that the examples are made from $\ell \in \naturals$ real-valued features, $\omega \in \naturals$ ordinal features and $\nu \in \naturals$ nominal features.
The domain space takes the form $\X = \mathds{R}^\ell \times \bigtimes_{j=1}^\nu [N_j]\times \bigtimes_{j=1}^\omega [O_j]$, where $\Ob \in \naturals^\omega$ and $\Nb \in \naturals^\nu$ are the feature landscapes of the ordinal and nominal features.
In this setting, the features $i$ takes value in $[\ell+\omega+\nu]$.

The rule set that applies to such a domain space is simply the union of the rule sets for real-valued features, ordinal features and nominal features.
Formally, we can write $\Phi = \Phi_{\text{real}} \cup \Phi_{\text{ord}} \cup \Phi_{\text{nom}}$, where 
\begin{equation*}
  \Phi_{\text{real}} = \cb{\phi(\x) = s \cdot \IdSign{x_i \le \theta } : s \in \cb{\pm 1}, i \in [\ell], \theta \in \reals},
\end{equation*}
\begin{equation*}
  \Phi_{\text{ord}} = \cb{\phi(\x) = s \cdot \IdSign{x_{j+\ell} \le \theta } : s \in \cb{\pm 1}, j \in [\omega], \theta \in [O_i-1]}
\end{equation*}
and
\begin{equation*}
  \Phi_{\text{nom}} = \cb{  \phi(\x) = s \cdot \IdSign{x_{j+\ell+\omega} = C} : s \in \cb{\pm 1}, i \in [\nu], C \in [N_i] }
\end{equation*}
are the standard threshold split and unitary comparison rule sets where the feature indices have been adjusted to account for the mixture.

Naturally, the partitioning functions will depend on the features landscapes, so let us make this dependence explicit by writing 
\begin{equation*}
    \pi^c_T(m, \ell, \Ob, \Nb) \eqdef
    \max_{\substack{S: |S|=m,\\
    \Obhat(S)\le \Ob,\\
    \Nbhat(S)\le \Nb}}
    \abs{\P^c_T(S)}.
\end{equation*}
We are now ready to prove Theorem~\ref{thm:ub_partitioning_functions_decision_trees_mixture_feat}.

\begin{proof} 
As per the previous proofs, we start from the union decomposition of $\P^c_T(S)$ of Proposition~\ref{prop:c-partitions-set_decomposition_decision_trees}.
To bound the partitioning functions, we decompose the realizable partitions according to the type of feature used by the decision rule (this decomposition may or may not be disjoint).
Hence, define $\R^{\text{real}}(S)$, $\R^{\text{ord}}(S)$, and $\R^{\text{nom}}(S)$ as the set of 2-partitions realizable on S by the root node by using a real-valued feature, an ordinal feature or a nominal feature.
We have
\begin{align*}
    \P^c_T(S)
        = \hspace{-5pt}
        \bigcup_{\plambda \in \R^{\text{real}}(S)} \hspace{-8pt} \S^c_T(\plambda)
        \hspace{3pt} \cup \hspace{-8pt}
        \bigcup_{\plambda \in \R^{\text{nom}}(S)} \hspace{-8pt} \S^c_T(\plambda)
        \hspace{3pt} \cup \hspace{-8pt}
        \bigcup_{\plambda \in \R^{\text{ord}}(S)} \hspace{-8pt} \S^c_T(\plambda).
\end{align*}
Using the union bound, we can bound independently each part of the expression using the bounds already found for the real-valued features, the nominal features and the ordinal features respectively.
Therefore, one can obtain a new bound by combining any of the results of Theorems~\ref{thm:ub_partitioning_functions_decision_trees_rl_feat}, \ref{thm:ub_partitioning_functions_decision_trees_ordinal_feat} and \ref{thm:ub_partitioning_functions_decision_trees_nominal_feat}.
In particular, choosing the bound that sums only on the part sizes $k$, we have (following the exact same steps as in the previous proofs)
\begin{align*}
    \abs{\P^c_T(S)} &
        \hspace{-1.5pt} \le \hspace{-1.5pt}
        2^{-\delta_{lr}}
        \hspace{-1.5pt}
        \sum_{k=1}^{m-1}
        \hspace{-0pt}
        R_k
        \hspace{-7pt}
        \sum_{\substack{1 \leq a, b \leq c\\ a+b\ge c}}
        \hspace{-9pt}
        \tbinom{a}{c-b}\tbinom{b}{c-a}(a\!+\!b\!-\!c)!
        \pi^a_{T_l}(k,\ell, \Ob^k, \Nb^k) \pi^b_{T_r}(m-k, \Ob^k, \Nb^k),
\end{align*}
where $\delta_{lr} = \Id{T_l = T_r}$, $O^k_i = \min\cb{O_i, k}$, $N^k_i = \min\cb{N_i, k}$, and
\begin{equation*}
    R_k \eqdef \min \cb{2\ell + 2\omega +  \floor{\frac{m}{\min\cb{k,m-k}}}\!\nu}.
\end{equation*}
Finally, taking the maximum on the samples $S$ on both side, one has the desired result.
\end{proof}

\clearpage
\section{Algorithms for the growth function, the partitioning functions and the VC dimension}
\label{app:algorithms}

In this Appendix, we present algorithms in the form of pseudo-code used to compute bounds on the partitioning functions, the growth function and the VC dimension of decision trees.
Algorithm~\ref{algo:partition_func_upper_bound} uses Theorem~\ref{thm:ub_partitioning_functions_decision_trees_mixture_feat} to upper bound the $c$-partitioning function of a tree class.
Algorithm~\ref{algo:vcdim_upper_bound} uses Algorithm~\ref{algo:partition_func_upper_bound} and Equation~\eqref{eq:vcdim_def_partition_func} to compute an upper bound on the VC dimension of a tree class.

\begin{algorithm2e}
\caption{PartiFuncUB$(T, c, m, \ell, \Ob, \Nb)$}\label{algo:partition_func_upper_bound}
\DontPrintSemicolon
\SetAlgoVlined

\KwIn{A tree class $T$, the number $c$ of parts in the partitions, the number $m$ of elements, the number $\ell$ of features, the feature landscapes $\Ob$ and $\Nb$ of ordinal and nominal features.}

\uIf{$c > m$ \textbf{or} $c > L_T$}{
  Let $N \leftarrow 0$.\;
  }
\uElseIf{$c = m$ \textbf{or} $c = 1$ \textbf{or} $m = 1$}{
  Let $N \leftarrow 1$.\;
  }
\Else{
  Let $T_l$ and $T_r$ be the left and right subtree classes of $T$.\;
  
  Let $N \leftarrow 0$.\;
  
  \For{$k = 1$ to $m-1$}{
    Let $R_k \leftarrow \min\cb{ 2(\ell + \omega) + \floor{\frac{m}{\min\cb{k,m-k}}} \!\nu }$\;

    \raggedright Let $\displaystyle N \leftarrow N + R_k \sum_{a=1}^c \sum_{b=\max\cb{1,\,c-a}}^c \tbinom{a}{c-b}\tbinom{b}{c-a}(a+b-c)!$\;
      \hspace*{50mm}$\times\;\text{PartiFuncUB}(T_l, a, k, \ell, \Ob^k, \Nb^k)$\;
      \hspace*{50mm}$\times\; \text{PartiFuncUB}(T_r, b, m-k, \ell, \Ob^k, \Nb^k)$.\;    
    \If{$T_l = T_r$}{
      Let $N \leftarrow \frac{N}{2}$.\;
    }
  }
}

\KwOut{$\min\pr{N, \smallstirling{m}{c}}$, an upper bound on $\pi^c_T(m, \ell, \Ob, \Nb)$.}
\end{algorithm2e}

\begin{algorithm2e}[ht]
\caption{VCdimUpperBound$(T, \ell, \Ob, \Nb)$}\label{algo:vcdim_upper_bound}
\DontPrintSemicolon
\SetAlgoVlined

\KwIn{A tree class $T$, the number $\ell$ of features.}

\If{$T$ is a leaf}{\KwOut{1}}

Let $m \leftarrow 1$.\;

\While{$\textrm{PartiFuncUB}(T, 2, m, \ell, \Ob, \Nb) \geq 2^{m-1}-1$}{
    Let $m \leftarrow m+1.$\;
  }
\KwOut{$m-1$, an upper bound on $\text{VCdim} T$.}
\end{algorithm2e}

Algorithm~\ref{algo:vcdim_upper_bound} can become quite inefficient because one has to compute the values of PartitionFuncUpperBound for increasing values of $m$, which may already have been computed for smaller values of $m$.
It is thus suggested to store the values of PartitionFuncUpperBound computed for each $T$ and each $m$ to be more efficient.

Algorithm~\ref{algo:log_partition_func_upper_bound} provides an algorithm to compute the logarithm of the $c$-partitioning functions, faster than Algorithm~\ref{algo:partition_func_upper_bound} but much quicker and avoiding overflows.
Algorith~\ref{algo:growth_func_upper_bound} provides an upper bound on the logarithm of the growth function using the same technique and relying on Algorithm~\ref{algo:log_partition_func_upper_bound}.

\begin{algorithm2e}
\caption{LogPartiFuncUB$(T, c, m, \ell, \Ob, \Nb)$}\label{algo:log_partition_func_upper_bound}
\DontPrintSemicolon
\SetAlgoVlined

\KwIn{A tree class $T$, the number $c$ of parts in the partitions, the number $m$ of elements, the number $\ell$ of features, the feature landscapes $\Ob$ and $\Nb$ of ordinal and nominal features.}
Let $L_T$ be the number of leaves of $T$.\;

\uIf{$c > m$ \textbf{or} $c > L_T$}{
  Let $\log\pi \leftarrow -\infty$.\;
  }
\uElseIf{$c = m$ \textbf{or} $c = 1$ \textbf{or} $m = 1$}{
  Let $\log\pi \leftarrow 0$.\;
  }
\uElseIf{$m \leq L_T$}{
    Let $\log\pi \leftarrow \log\pr{\smallstirling{m}{c}}$.}
\Else{
  Let $T_l$ and $T_r$ be the left and right subtree classes of $T$.\;
  
  Let $a' \leftarrow \min\cb{c, L_{T_l}}$ and $b' \leftarrow \min\cb{c, L_{T_r}}$.\;

  \raggedright Let $\text{main} \leftarrow \log\pr{\tbinom{a'}{c-b'}\tbinom{b'}{c-a'}(a'+b'-c)!}$\;
    \hspace*{30mm}$+ \text{LogPartiFuncUB}(T_l, a', m-1, \ell, \Ob, \Nb)$\;
    \hspace*{30mm}$+ \text{LogPartiFuncUB}(T_r, b', m-1, \ell, \Ob, \Nb)$.\;

  Let $\text{cumul}\leftarrow 0$.\;

  \For{$a = 1$ to $c$}{
    \For{$b = \max\cb{1,c-a}$ to $c$}{
      \raggedright $\text{cumul} \leftarrow \text{cumul} +
        \exp\Big(
          -\text{main} + \log\pr{ \tbinom{a}{c-b}\tbinom{b}{c-a}(a+b-c)! } $\;
          \hspace*{40mm}$+ \text{LogPartiFuncUB}(T_l, a, m-1, \ell, \Ob, \Nb)$\;
          \hspace*{40mm}$+ \text{LogPartiFuncUB}(T_r, b, m-1, \ell, \Ob, \Nb)
      \Big)$\;
    }
  }

  Let $\log\pi \leftarrow \log(2^{-\delta_{lr}}) + \log(m-1) + \log(R_{m-1}) + \text{main} + \log(\text{cumul})$.\;

}

\KwOut{$\min\pr{\log\pi, \log\pr{\smallstirling{m}{c}}}$, an upper bound on $\log(\pi^c_T(m, \ell, \Ob, \Nb))$.}
\end{algorithm2e}

\begin{algorithm2e}
\caption{LogGrowthFuncUB$(T, n, m, \ell, \Ob, \Nb)$}\label{algo:growth_func_upper_bound}
\DontPrintSemicolon
\SetAlgoVlined

\KwIn{A tree class $T$, the number $n$ of classes, the number $m$ of elements, the number $\ell$ of features, the feature landscapes $\Ob$ and $\Nb$ of ordinal and nominal features.}
Let $L_T$ be the number of leaves of $T$.\;

Let $M \leftarrow \min\cb{ n, L_T, m }$.\;

Let $\texttt{main}\leftarrow \log((n)_M) + \text{LogPartiFuncUB}(T, M, m, \ell, \Ob, \Nb)$ be the main term of the sum, with $(n)_M$ the falling factorial.\;

Let $\texttt{cumul} \leftarrow 1$.\;

\For{$a = 1$ to $M-1$}{
    \raggedright $\texttt{cumul} \leftarrow \texttt{cumul} +
      \exp\Big(
        -\texttt{main} + \log\!\pr{ (n)_a } + \text{LogPartiFuncUB}(T, a, m, \ell, \Ob, \Nb)\Big)$\;
}

\KwOut{$\texttt{main} + \log(\texttt{cumul})$, an upper bound on $\log(\tau_T(m))$.}
\end{algorithm2e}

\clearpage
\section{Supplementary materials about the experiments}
\label{app:experiments}

In this Appendix, we provide more details about the experiments that were done.

\subsection{Greedy algorithm to grow a decision tree}
\label{app:algorithms_to_grow_tree}

The algorithm we used to grow the decision tree is the traditional greedy one, where at each step, one examines every leaf of the tree, evaluates the gain made by splitting each leaf into a stump, and applies the split that generates the best gain.
Algorithm~\ref{algo:greedy_tree} presents the formal steps in the form of pseudo-code.

\def\makesplit{\textup{MakeSplit}}
\begin{algorithm2e}
\caption{GreedyTree$(S, K, \makesplit)$}\label{algo:greedy_tree}
\DontPrintSemicolon
\SetAlgoVlined


\KwIn{The sample $S$ of the $m$ examples, a set $K$ of conditions the tree must satisfy (\eg the maximum number of leaves, the maximum height, etc.), a function $\makesplit$ that takes as input a leaf and a sample, and outputs a stump and the impurity scores of its left and right leaves.}

Let $t$ be the initial tree made from a unique leaf with the majority label.

Let $(s, g_{s_l}, g_{s_r}) \leftarrow \makesplit(t, S)$ be the tuple formed from the stump $s$ and the impurity scores $g_{s_l}$ and $g_{s_r}$ of its left leaf $s_l$ and its right leaf $s_r$.

Let $\Sigma \leftarrow \cb{(s, g_{s_l}, g_{s_r})}$ be the set of all candidate for the next splits.

\While{$\Sigma$ is not empty \textbf{and} the empirical risk of $t$ is not null}{
  Let $(s, g_{s_l}, g_{s_r}) \leftarrow \argmin_{(\sigma, g_{\sigma_l}, g_{\sigma_r}) \in \Sigma} (g_{\sigma_l} + g_{\sigma_r})$ be the best candidate split.

  Let $t'$ be the tree $t$ updated by replacing the appropriate leaf with the new split $s$.

  Update $\Sigma \leftarrow \Sigma \backslash \cb{(s, g_{s_l}, g_{s_r})}$ to remove the split from candidates.

  \If{$t'$ satisfies all conditions in $K$}{
    
    Update $t \leftarrow t'$.
  

    \For{both leaves $\lambda \in \cb{s_l, s_r}$}{

      Let $S_\lambda$ be the subset of $S$ that reaches leaf $\lambda$.

      Let $(\sigma, g_{\sigma_l}, g_{\sigma_r}) \leftarrow \makesplit(\lambda, S_\lambda)$ be the split that could replace leaf $\lambda$.

      \If{$g_{\sigma_l} + g_{\sigma_r} < g_\lambda$}{

        Update $\Sigma \leftarrow \Sigma \cup \cb{ (\sigma, g_\sigma) }$.

      }
    }
  }
}

\KwOut{The greedily overgrown tree $t$.}
\end{algorithm2e}

To determine if a stump makes gain, we require a function of the form $f:S \to \reals$ which maps a sample to a real number, called the impurity score.
We say that the impurity score of a leaf $\lambda$ corresponds to the output of $f(S_\lambda)$, where $S_\lambda$ is the subset of the examples that reach leaf $\lambda$.
In our algorithm, we assume that the impurity score of a stump is the sum of the impurity score of its leaves (a fact used a line 5 and 13 of Algorithm~\ref{algo:greedy_tree}).
Furthermore, we assume that a stump makes gain over a leaf if its impurity score is less than that of the leaf.
These assumptions are standard, but they could easily be modified in the algorithm if one desired so.
In our experiments, we use the Gini index as impurity score function, which is defined as 
\begin{equation}\label{eq:def_gini}
  \text{Gini}(S) \eqdef 1 - \sum_{a=1}^n \pr{ \frac{ \abs{\cb{ (\x,y) \in S : y = a }}}{\abs{S}} }^2.
\end{equation}

To facilitate the comprehension of the algorithm, we abstract the step of making a leaf into a stump into a function ``MakeSplit'', which takes as input a leaf and the subset of the sample that reaches it in order to produce the best stump that could replace the leaf.
It outputs a decision stump (made from a decision rules and two leaves with class labels), as well as the impurity scores of its left and right subtrees.
In our implementation, MakeSplit calls a method which finds the decision rule that minimizes the impurity score of the resulting split.
The steps to achieve this are detailed in Algorithm~\ref{algo:find_best_split}.

\begin{algorithm2e}
\caption{FindBestSplit$(S, n, f)$}\label{algo:find_best_split}
\DontPrintSemicolon
\SetAlgoVlined

\KwIn{The sample $S$, the number $n$ of classes, an impurity score function $f$ that takes as input a vector of $n$ real numbers and outputs a real number.}

Let $\rho$ be the current decision rule in the form of a tuple consisting in a feature index and category or threshold.

Let $\z^l \leftarrow (0,\dots,0)$ be the number of examples of each label that are sent to the left leaf.

Let $\z^r$ be the number of examples of each label that are sent to the right leaf such that $z^r_a \eqdef \abs{ \cb{ (\x, y) \in S : y = a } }$ for every $a \in [n]$.

Let $g_l \leftarrow f(\z^l)$ and $g_r \leftarrow f(\z^r)$ be the impurity scores of the left and right leaves.

\For{every feature $i=1$ to $\ell+\omega+\nu$}{
  
  Let $S'$ be the sample sorted according to feature $i$.

  Let $\z^l \leftarrow (0,\dots,0)$ and $\z^r$ such that $z^r_a \eqdef \abs{ \cb{ (\x_j, y_j) \in S' : y_j = a } }$.

  \For{every example index $j=1$ to $m-1$}{

    Let $(\x_j, y_j)$ be the $j$-th example of $S'$.

    \If{$x_{j,i} \ne x_{j+1,i}$}{

      \If{$f(\z^l) + f(\z^r) \le g_l + g_r$}{

        \uIf{feature $i$ is nominal \textbf{or} feature $i$ is ordinal}{

          Let $\rho \leftarrow (i, x_{j,i})$ be the new decision rule.
        }
        \Else{
          Let $\rho \leftarrow (i, \frac{1}{2}(x_{j,i}+x_{j+1,i})$ be the new decision rule.
        }

        Let $g_l \leftarrow f(\z^l)$ and $g_r \leftarrow f(\z^r)$ be the new best impurity scores.

      }

      \If{feature $i$ is nominal}{

        Reinitialize $\z^r \leftarrow \z^r + \z^l$, and then $\z^l \leftarrow (0,\dots,0)$.
      }
    }

    Update $z^r_{y_j} \leftarrow z^r_{y_j} - 1$, and $z^l_{y_j} \leftarrow z^l_{y_j} + 1$.

  }
}

\KwOut{The decision rule $\rho$ that minimizes the impurity score, the impurity scores $g_l$ and $g_r$ of the left and right leaves.}
\end{algorithm2e}

Algorithm~\ref{algo:find_best_split} examines sequentially every possible split of the data for every features, and retains the one that minimizes the impurity score.
This algorithm is the conventional one used by \citet{breiman1984classification}'s CART and C4.5 algorithms, but keeps track of the variables $\z^l$ and $\z^r$, which counts the number of examples of each class are sent in each leaf, in order to easily compute the impurity scores and the classes of the leaves.
We here abuse notation by writing $f(\z)$, but as one can see from Equation~\eqref{eq:def_gini}, $\z$ provides all the information required to compute the Gini index.
Furthermore, this simple modification allows us to handle the case of nominal feature with the unitary comparison decision rule by simply resetting $\z^l$ and $\z^r$ (lines 17 and 18 of Algorithm~\ref{algo:find_best_split}).
The running time of this algorithm is thus in $\O((\ell+\omega+\nu)m (n+ \log m)$ (sorting $S$ is in $\O(m \log m)$ while the loop on $j$ is in $\O(mn)$ assuming that $f$ is in $\O(n)$.)
Note that the loop on the features can be easily parallelized, as we do in our code implemention.

As one can see, it is completely unnecessary to convert each nominal feature $i$ into $N_i$ binary features as is done in many popular libraries, such as CatBoost.
In fact, this conversion slows down the algorithm, as it replaces $\nu$ in the running time by $\sum_{i=1}^\nu N_i$, which can be \emph{much} larger than $\nu$.


\subsection{More statistics about model performances}
\label{app:more_stats}

We here give more statics on the performances of the model tested, such as the training accuracy, the number of leaves and the height of the final tree, the time it took to prune the original tree, and the computed bound in the case of our pruning algorithm.
For each table, the caption gives the data set name, the total number of examples it contains, the number of features each example has (as well as their type), and the number of classes to predict.
The reported values are the average over 25 runs.
The standard deviation is indicated in parentheses.
For more details about the table columns, see the methodology section~\ref{ssec:methodology}.

\begin{table}[h!]
\centering
\caption{Statistics for the Acute Inflammation data set \citep{czerniak2003application}. There are 120 examples, 7 features (1 real-valued, 0 ordinal and 6 nominal), and 2 classes.}
\vspace{6pt}
\small
\begin{tabular}{l@{\hspace{6pt}}c@{\hspace{6pt}}c@{\hspace{6pt}}c@{\hspace{6pt}}c@{\hspace{6pt}}c@{\hspace{6pt}}cl@{\hspace{6pt}}c@{\hspace{6pt}}c@{\hspace{6pt}}c@{\hspace{6pt}}c@{\hspace{6pt}}c@{\hspace{6pt}}cl@{\hspace{6pt}}c@{\hspace{6pt}}c@{\hspace{6pt}}c@{\hspace{6pt}}c@{\hspace{6pt}}c@{\hspace{6pt}}cl@{\hspace{6pt}}c@{\hspace{6pt}}c@{\hspace{6pt}}c@{\hspace{6pt}}c@{\hspace{6pt}}c@{\hspace{6pt}}cl@{\hspace{6pt}}c@{\hspace{6pt}}c@{\hspace{6pt}}c@{\hspace{6pt}}c@{\hspace{6pt}}c@{\hspace{6pt}}cl@{\hspace{6pt}}c@{\hspace{6pt}}c@{\hspace{6pt}}c@{\hspace{6pt}}c@{\hspace{6pt}}c@{\hspace{6pt}}cl@{\hspace{6pt}}c@{\hspace{6pt}}c@{\hspace{6pt}}c@{\hspace{6pt}}c@{\hspace{6pt}}c@{\hspace{6pt}}c}
\toprule
 & OG & CC & RE & KM & Ours & Oracle\\
\cmidrule{2-7}
Train acc. & $1.000 (0.000)$ & $0.917 (0.011)$ & $0.977 (0.043)$ & $1.000 (0.000)$ & $1.000 (0.000)$ & $0.976 (0.051)$\\
Val. acc. & NA & NA & $1.000 (0.000)$ & NA & NA & NA\\
Test acc. & $\mathbf{1.000 (0.000)}$ & $0.916 (0.062)$ & $0.971 (0.062)$ & $\mathbf{1.000 (0.000)}$ & $\mathbf{1.000 (0.000)}$ & $1.000 (0.000)$\\
Leaves & $3.0 (0.0)$ & $2.0 (0.0)$ & $2.8 (0.4)$ & $3.0 (0.0)$ & $3.0 (0.0)$ & $2.8 (0.4)$\\
Height & $2.0 (0.0)$ & $1.0 (0.0)$ & $1.8 (0.4)$ & $2.0 (0.0)$ & $2.0 (0.0)$ & $1.8 (0.4)$\\
Time $[s]$ & $0.000 (0.000)$ & $0.167 (0.002)$ & $0.001 (0.000)$ & $0.133 (0.001)$ & $0.002 (0.000)$ & $0.001 (0.000)$\\
Bound & NA & NA & NA & NA & $0.884 (0.000)$ & NA\\
\bottomrule
\end{tabular}
\end{table}
\begin{table}[h!]
\centering
\caption{Statistics for the Amphibians data set \citep{blachnik2019predicting}. There are 189 examples, 14 features (2 real-valued, 3 ordinal and 9 nominal), and 2 classes.}
\vspace{6pt}
\small
\begin{tabular}{l@{\hspace{6pt}}c@{\hspace{6pt}}c@{\hspace{6pt}}c@{\hspace{6pt}}c@{\hspace{6pt}}c@{\hspace{6pt}}cl@{\hspace{6pt}}c@{\hspace{6pt}}c@{\hspace{6pt}}c@{\hspace{6pt}}c@{\hspace{6pt}}c@{\hspace{6pt}}cl@{\hspace{6pt}}c@{\hspace{6pt}}c@{\hspace{6pt}}c@{\hspace{6pt}}c@{\hspace{6pt}}c@{\hspace{6pt}}cl@{\hspace{6pt}}c@{\hspace{6pt}}c@{\hspace{6pt}}c@{\hspace{6pt}}c@{\hspace{6pt}}c@{\hspace{6pt}}cl@{\hspace{6pt}}c@{\hspace{6pt}}c@{\hspace{6pt}}c@{\hspace{6pt}}c@{\hspace{6pt}}c@{\hspace{6pt}}cl@{\hspace{6pt}}c@{\hspace{6pt}}c@{\hspace{6pt}}c@{\hspace{6pt}}c@{\hspace{6pt}}c@{\hspace{6pt}}cl@{\hspace{6pt}}c@{\hspace{6pt}}c@{\hspace{6pt}}c@{\hspace{6pt}}c@{\hspace{6pt}}c@{\hspace{6pt}}c}
\toprule
 & OG & CC & RE & KM & Ours & Oracle\\
\cmidrule{2-7}
Train acc. & $1.000 (0.000)$ & $0.714 (0.133)$ & $0.688 (0.069)$ & $0.769 (0.172)$ & $1.000 (0.000)$ & $0.671 (0.053)$\\
Val. acc. & NA & NA & $0.814 (0.079)$ & NA & NA & NA\\
Test acc. & $0.591 (0.098)$ & $0.597 (0.110)$ & $\mathbf{0.617 (0.090)}$ & $0.560 (0.087)$ & $0.589 (0.097)$ & $0.814 (0.075)$\\
Leaves & $44.1 (3.0)$ & $7.3 (10.4)$ & $6.2 (3.1)$ & $14.4 (16.7)$ & $43.8 (2.9)$ & $6.0 (3.2)$\\
Height & $11.9 (2.1)$ & $3.4 (4.0)$ & $3.6 (1.8)$ & $5.1 (4.9)$ & $11.7 (2.0)$ & $3.4 (1.8)$\\
Time $[s]$ & $0.000 (0.000)$ & $1.690 (0.112)$ & $0.237 (0.164)$ & $1.985 (0.163)$ & $0.265 (0.127)$ & $0.266 (0.195)$\\
Bound & NA & NA & NA & NA & $6.674 (0.067)$ & NA\\
\bottomrule
\end{tabular}
\end{table}
\begin{table}[h!]
\centering
\caption{Statistics for the Breast Cancer Wisconsin Diagnostic data set \citep{street1993nuclear}. There are 569 examples, 30 features (30 real-valued, 0 ordinal and 0 nominal), and 2 classes.}
\vspace{6pt}
\small
\begin{tabular}{l@{\hspace{6pt}}c@{\hspace{6pt}}c@{\hspace{6pt}}c@{\hspace{6pt}}c@{\hspace{6pt}}c@{\hspace{6pt}}cl@{\hspace{6pt}}c@{\hspace{6pt}}c@{\hspace{6pt}}c@{\hspace{6pt}}c@{\hspace{6pt}}c@{\hspace{6pt}}cl@{\hspace{6pt}}c@{\hspace{6pt}}c@{\hspace{6pt}}c@{\hspace{6pt}}c@{\hspace{6pt}}c@{\hspace{6pt}}cl@{\hspace{6pt}}c@{\hspace{6pt}}c@{\hspace{6pt}}c@{\hspace{6pt}}c@{\hspace{6pt}}c@{\hspace{6pt}}cl@{\hspace{6pt}}c@{\hspace{6pt}}c@{\hspace{6pt}}c@{\hspace{6pt}}c@{\hspace{6pt}}c@{\hspace{6pt}}cl@{\hspace{6pt}}c@{\hspace{6pt}}c@{\hspace{6pt}}c@{\hspace{6pt}}c@{\hspace{6pt}}c@{\hspace{6pt}}cl@{\hspace{6pt}}c@{\hspace{6pt}}c@{\hspace{6pt}}c@{\hspace{6pt}}c@{\hspace{6pt}}c@{\hspace{6pt}}c}
\toprule
 & OG & CC & RE & KM & Ours & Oracle\\
\cmidrule{2-7}
Train acc. & $1.000 (0.000)$ & $0.952 (0.072)$ & $0.956 (0.015)$ & $0.986 (0.008)$ & $0.983 (0.005)$ & $0.952 (0.018)$\\
Val. acc. & NA & NA & $0.951 (0.023)$ & NA & NA & NA\\
Test acc. & $0.921 (0.033)$ & $0.908 (0.064)$ & $0.924 (0.033)$ & $\mathbf{0.932 (0.028)}$ & $\mathbf{0.933 (0.028)}$ & $0.952 (0.022)$\\
Leaves & $19.2 (1.9)$ & $6.3 (3.4)$ & $4.3 (1.2)$ & $10.6 (2.9)$ & $8.2 (1.4)$ & $5.3 (1.9)$\\
Height & $7.1 (1.0)$ & $3.6 (1.7)$ & $3.0 (0.9)$ & $5.3 (1.0)$ & $4.4 (0.6)$ & $3.2 (1.0)$\\
Time $[s]$ & $0.000 (0.000)$ & $3.319 (0.273)$ & $0.084 (0.043)$ & $1.815 (0.148)$ & $0.189 (0.050)$ & $0.131 (0.056)$\\
Bound & NA & NA & NA & NA & $1.306 (0.102)$ & NA\\
\bottomrule
\end{tabular}
\end{table}
\begin{table}[h!]
\centering
\caption{Statistics for the Cardiotocography10 data set \citep{ayres2000sisporto}. There are 2126 examples, 21 features (21 real-valued, 0 ordinal and 0 nominal), and 10 classes.}
\vspace{6pt}
\small
\begin{tabular}{l@{\hspace{6pt}}c@{\hspace{6pt}}c@{\hspace{6pt}}c@{\hspace{6pt}}c@{\hspace{6pt}}c@{\hspace{6pt}}cl@{\hspace{6pt}}c@{\hspace{6pt}}c@{\hspace{6pt}}c@{\hspace{6pt}}c@{\hspace{6pt}}c@{\hspace{6pt}}cl@{\hspace{6pt}}c@{\hspace{6pt}}c@{\hspace{6pt}}c@{\hspace{6pt}}c@{\hspace{6pt}}c@{\hspace{6pt}}cl@{\hspace{6pt}}c@{\hspace{6pt}}c@{\hspace{6pt}}c@{\hspace{6pt}}c@{\hspace{6pt}}c@{\hspace{6pt}}cl@{\hspace{6pt}}c@{\hspace{6pt}}c@{\hspace{6pt}}c@{\hspace{6pt}}c@{\hspace{6pt}}c@{\hspace{6pt}}cl@{\hspace{6pt}}c@{\hspace{6pt}}c@{\hspace{6pt}}c@{\hspace{6pt}}c@{\hspace{6pt}}c@{\hspace{6pt}}cl@{\hspace{6pt}}c@{\hspace{6pt}}c@{\hspace{6pt}}c@{\hspace{6pt}}c@{\hspace{6pt}}c@{\hspace{6pt}}c}
\toprule
 & OG & CC & RE & KM & Ours & Oracle\\
\cmidrule{2-7}
Train acc. & $0.622 (0.008)$ & $0.578 (0.017)$ & $0.628 (0.056)$ & $0.598 (0.013)$ & $0.593 (0.009)$ & $0.586 (0.010)$\\
Val. acc. & NA & NA & $0.626 (0.073)$ & NA & NA & NA\\
Test acc. & $0.571 (0.027)$ & $0.570 (0.024)$ & $\mathbf{0.608 (0.057)}$ & $0.578 (0.026)$ & $0.582 (0.027)$ & $0.603 (0.023)$\\
Leaves & $75.0 (0.0)$ & $10.4 (10.2)$ & $15.7 (4.5)$ & $22.0 (11.7)$ & $12.1 (4.2)$ & $13.8 (4.6)$\\
Height & $20.5 (1.1)$ & $5.8 (4.2)$ & $7.3 (1.6)$ & $9.4 (3.6)$ & $6.5 (1.6)$ & $7.3 (2.1)$\\
Time $[s]$ & $0.000 (0.000)$ & $19.071 (0.517)$ & $3.248 (2.182)$ & $24.631 (1.524)$ & $19.793 (2.944)$ & $3.597 (2.298)$\\
Bound & NA & NA & NA & NA & $13.033 (0.189)$ & NA\\
\bottomrule
\end{tabular}
\end{table}
\begin{table}[h!]
\centering
\caption{Statistics for the Climate Model Simulation Crashes data set \citep{lucas2013failure}. There are 540 examples, 18 features (18 real-valued, 0 ordinal and 0 nominal), and 2 classes.}
\vspace{6pt}
\small
\begin{tabular}{l@{\hspace{6pt}}c@{\hspace{6pt}}c@{\hspace{6pt}}c@{\hspace{6pt}}c@{\hspace{6pt}}c@{\hspace{6pt}}cl@{\hspace{6pt}}c@{\hspace{6pt}}c@{\hspace{6pt}}c@{\hspace{6pt}}c@{\hspace{6pt}}c@{\hspace{6pt}}cl@{\hspace{6pt}}c@{\hspace{6pt}}c@{\hspace{6pt}}c@{\hspace{6pt}}c@{\hspace{6pt}}c@{\hspace{6pt}}cl@{\hspace{6pt}}c@{\hspace{6pt}}c@{\hspace{6pt}}c@{\hspace{6pt}}c@{\hspace{6pt}}c@{\hspace{6pt}}cl@{\hspace{6pt}}c@{\hspace{6pt}}c@{\hspace{6pt}}c@{\hspace{6pt}}c@{\hspace{6pt}}c@{\hspace{6pt}}cl@{\hspace{6pt}}c@{\hspace{6pt}}c@{\hspace{6pt}}c@{\hspace{6pt}}c@{\hspace{6pt}}c@{\hspace{6pt}}cl@{\hspace{6pt}}c@{\hspace{6pt}}c@{\hspace{6pt}}c@{\hspace{6pt}}c@{\hspace{6pt}}c@{\hspace{6pt}}c}
\toprule
 & OG & CC & RE & KM & Ours & Oracle\\
\cmidrule{2-7}
Train acc. & $1.000 (0.000)$ & $0.934 (0.024)$ & $0.935 (0.021)$ & $0.971 (0.021)$ & $0.972 (0.007)$ & $0.933 (0.021)$\\
Val. acc. & NA & NA & $0.940 (0.020)$ & NA & NA & NA\\
Test acc. & $0.897 (0.031)$ & $0.909 (0.022)$ & $0.907 (0.024)$ & $\mathbf{0.913 (0.025)}$ & $\mathbf{0.914 (0.025)}$ & $0.940 (0.021)$\\
Leaves & $23.5 (1.6)$ & $3.1 (3.0)$ & $4.1 (2.4)$ & $10.1 (4.4)$ & $8.2 (1.2)$ & $4.3 (2.3)$\\
Height & $7.8 (1.0)$ & $1.7 (2.3)$ & $2.5 (1.6)$ & $5.8 (2.2)$ & $5.2 (0.8)$ & $2.9 (1.8)$\\
Time $[s]$ & $0.000 (0.000)$ & $2.650 (0.304)$ & $0.072 (0.046)$ & $1.692 (0.176)$ & $0.352 (0.349)$ & $0.129 (0.078)$\\
Bound & NA & NA & NA & NA & $1.664 (0.125)$ & NA\\
\bottomrule
\end{tabular}
\end{table}
\begin{table}[h!]
\centering
\caption{Statistics for the Connectionist Bench Sonar data set \citep{gorman1988analysis}. There are 208 examples, 60 features (60 real-valued, 0 ordinal and 0 nominal), and 2 classes.}
\vspace{6pt}
\small
\begin{tabular}{l@{\hspace{6pt}}c@{\hspace{6pt}}c@{\hspace{6pt}}c@{\hspace{6pt}}c@{\hspace{6pt}}c@{\hspace{6pt}}cl@{\hspace{6pt}}c@{\hspace{6pt}}c@{\hspace{6pt}}c@{\hspace{6pt}}c@{\hspace{6pt}}c@{\hspace{6pt}}cl@{\hspace{6pt}}c@{\hspace{6pt}}c@{\hspace{6pt}}c@{\hspace{6pt}}c@{\hspace{6pt}}c@{\hspace{6pt}}cl@{\hspace{6pt}}c@{\hspace{6pt}}c@{\hspace{6pt}}c@{\hspace{6pt}}c@{\hspace{6pt}}c@{\hspace{6pt}}cl@{\hspace{6pt}}c@{\hspace{6pt}}c@{\hspace{6pt}}c@{\hspace{6pt}}c@{\hspace{6pt}}c@{\hspace{6pt}}cl@{\hspace{6pt}}c@{\hspace{6pt}}c@{\hspace{6pt}}c@{\hspace{6pt}}c@{\hspace{6pt}}c@{\hspace{6pt}}cl@{\hspace{6pt}}c@{\hspace{6pt}}c@{\hspace{6pt}}c@{\hspace{6pt}}c@{\hspace{6pt}}c@{\hspace{6pt}}c}
\toprule
 & OG & CC & RE & KM & Ours & Oracle\\
\cmidrule{2-7}
Train acc. & $1.000 (0.000)$ & $0.869 (0.128)$ & $0.785 (0.089)$ & $0.978 (0.035)$ & $0.962 (0.012)$ & $0.796 (0.063)$\\
Val. acc. & NA & NA & $0.813 (0.072)$ & NA & NA & NA\\
Test acc. & $0.714 (0.073)$ & $0.697 (0.087)$ & $0.683 (0.082)$ & $0.725 (0.067)$ & $\mathbf{0.728 (0.061)}$ & $0.835 (0.050)$\\
Leaves & $18.4 (1.4)$ & $7.9 (5.2)$ & $4.6 (1.7)$ & $15.1 (3.7)$ & $11.5 (1.4)$ & $5.4 (2.0)$\\
Height & $6.7 (0.7)$ & $3.8 (2.1)$ & $2.7 (1.1)$ & $6.2 (1.1)$ & $5.3 (1.0)$ & $3.3 (1.0)$\\
Time $[s]$ & $0.000 (0.000)$ & $1.524 (0.071)$ & $0.052 (0.029)$ & $1.002 (0.059)$ & $0.136 (0.045)$ & $0.061 (0.027)$\\
Bound & NA & NA & NA & NA & $4.096 (0.224)$ & NA\\
\bottomrule
\end{tabular}
\end{table}
\begin{table}[h!]
\centering
\caption{Statistics for the Diabetic Retinopathy Debrecen data set \citep{antal2014ensemble}. There are 1151 examples, 19 features (19 real-valued, 0 ordinal and 0 nominal), and 2 classes.}
\vspace{6pt}
\small
\begin{tabular}{l@{\hspace{6pt}}c@{\hspace{6pt}}c@{\hspace{6pt}}c@{\hspace{6pt}}c@{\hspace{6pt}}c@{\hspace{6pt}}cl@{\hspace{6pt}}c@{\hspace{6pt}}c@{\hspace{6pt}}c@{\hspace{6pt}}c@{\hspace{6pt}}c@{\hspace{6pt}}cl@{\hspace{6pt}}c@{\hspace{6pt}}c@{\hspace{6pt}}c@{\hspace{6pt}}c@{\hspace{6pt}}c@{\hspace{6pt}}cl@{\hspace{6pt}}c@{\hspace{6pt}}c@{\hspace{6pt}}c@{\hspace{6pt}}c@{\hspace{6pt}}c@{\hspace{6pt}}cl@{\hspace{6pt}}c@{\hspace{6pt}}c@{\hspace{6pt}}c@{\hspace{6pt}}c@{\hspace{6pt}}c@{\hspace{6pt}}cl@{\hspace{6pt}}c@{\hspace{6pt}}c@{\hspace{6pt}}c@{\hspace{6pt}}c@{\hspace{6pt}}c@{\hspace{6pt}}cl@{\hspace{6pt}}c@{\hspace{6pt}}c@{\hspace{6pt}}c@{\hspace{6pt}}c@{\hspace{6pt}}c@{\hspace{6pt}}c}
\toprule
 & OG & CC & RE & KM & Ours & Oracle\\
\cmidrule{2-7}
Train acc. & $0.783 (0.018)$ & $0.621 (0.064)$ & $0.668 (0.042)$ & $0.741 (0.079)$ & $0.727 (0.024)$ & $0.647 (0.036)$\\
Val. acc. & NA & NA & $0.715 (0.037)$ & NA & NA & NA\\
Test acc. & $0.622 (0.037)$ & $0.592 (0.065)$ & $0.607 (0.035)$ & $0.617 (0.052)$ & $\mathbf{0.634 (0.038)}$ & $0.704 (0.030)$\\
Leaves & $75.0 (0.0)$ & $3.7 (3.2)$ & $13.9 (4.5)$ & $48.8 (20.3)$ & $20.5 (5.7)$ & $12.6 (5.6)$\\
Height & $12.8 (2.4)$ & $2.3 (2.3)$ & $6.4 (1.4)$ & $10.4 (4.2)$ & $7.4 (1.2)$ & $5.8 (1.6)$\\
Time $[s]$ & $0.000 (0.000)$ & $9.978 (0.332)$ & $1.689 (0.986)$ & $9.084 (0.876)$ & $4.353 (1.435)$ & $1.769 (1.310)$\\
Bound & NA & NA & NA & NA & $9.510 (0.518)$ & NA\\
\bottomrule
\end{tabular}
\end{table}
\begin{table}[h!]
\centering
\caption{Statistics for the Fertility data set \citep{gil2012predicting}. There are 100 examples, 9 features (9 real-valued, 0 ordinal and 0 nominal), and 2 classes.}
\vspace{6pt}
\small
\begin{tabular}{l@{\hspace{6pt}}c@{\hspace{6pt}}c@{\hspace{6pt}}c@{\hspace{6pt}}c@{\hspace{6pt}}c@{\hspace{6pt}}cl@{\hspace{6pt}}c@{\hspace{6pt}}c@{\hspace{6pt}}c@{\hspace{6pt}}c@{\hspace{6pt}}c@{\hspace{6pt}}cl@{\hspace{6pt}}c@{\hspace{6pt}}c@{\hspace{6pt}}c@{\hspace{6pt}}c@{\hspace{6pt}}c@{\hspace{6pt}}cl@{\hspace{6pt}}c@{\hspace{6pt}}c@{\hspace{6pt}}c@{\hspace{6pt}}c@{\hspace{6pt}}c@{\hspace{6pt}}cl@{\hspace{6pt}}c@{\hspace{6pt}}c@{\hspace{6pt}}c@{\hspace{6pt}}c@{\hspace{6pt}}c@{\hspace{6pt}}cl@{\hspace{6pt}}c@{\hspace{6pt}}c@{\hspace{6pt}}c@{\hspace{6pt}}c@{\hspace{6pt}}c@{\hspace{6pt}}cl@{\hspace{6pt}}c@{\hspace{6pt}}c@{\hspace{6pt}}c@{\hspace{6pt}}c@{\hspace{6pt}}c@{\hspace{6pt}}c}
\toprule
 & OG & CC & RE & KM & Ours & Oracle\\
\cmidrule{2-7}
Train acc. & $0.991 (0.005)$ & $0.880 (0.010)$ & $0.846 (0.050)$ & $0.905 (0.050)$ & $0.880 (0.010)$ & $0.868 (0.030)$\\
Val. acc. & NA & NA & $0.909 (0.077)$ & NA & NA & NA\\
Test acc. & $0.763 (0.096)$ & $\mathbf{0.880 (0.054)}$ & $0.845 (0.083)$ & $0.843 (0.100)$ & $\mathbf{0.880 (0.054)}$ & $0.899 (0.055)$\\
Leaves & $15.9 (2.7)$ & $1.0 (0.0)$ & $2.0 (1.3)$ & $3.6 (5.0)$ & $1.0 (0.0)$ & $1.7 (1.2)$\\
Height & $7.2 (0.9)$ & $0.0 (0.0)$ & $1.0 (1.3)$ & $1.8 (3.3)$ & $0.0 (0.0)$ & $0.7 (1.2)$\\
Time $[s]$ & $0.000 (0.000)$ & $0.387 (0.036)$ & $0.025 (0.012)$ & $0.456 (0.055)$ & $0.022 (0.007)$ & $0.036 (0.029)$\\
Bound & NA & NA & NA & NA & $3.996 (0.299)$ & NA\\
\bottomrule
\end{tabular}
\end{table}
\begin{table}[h!]
\centering
\caption{Statistics for the Habermans Survival data set \citep{haberman1976generalized}. There are 306 examples, 3 features (3 real-valued, 0 ordinal and 0 nominal), and 2 classes.}
\vspace{6pt}
\small
\begin{tabular}{l@{\hspace{6pt}}c@{\hspace{6pt}}c@{\hspace{6pt}}c@{\hspace{6pt}}c@{\hspace{6pt}}c@{\hspace{6pt}}cl@{\hspace{6pt}}c@{\hspace{6pt}}c@{\hspace{6pt}}c@{\hspace{6pt}}c@{\hspace{6pt}}c@{\hspace{6pt}}cl@{\hspace{6pt}}c@{\hspace{6pt}}c@{\hspace{6pt}}c@{\hspace{6pt}}c@{\hspace{6pt}}c@{\hspace{6pt}}cl@{\hspace{6pt}}c@{\hspace{6pt}}c@{\hspace{6pt}}c@{\hspace{6pt}}c@{\hspace{6pt}}c@{\hspace{6pt}}cl@{\hspace{6pt}}c@{\hspace{6pt}}c@{\hspace{6pt}}c@{\hspace{6pt}}c@{\hspace{6pt}}c@{\hspace{6pt}}cl@{\hspace{6pt}}c@{\hspace{6pt}}c@{\hspace{6pt}}c@{\hspace{6pt}}c@{\hspace{6pt}}c@{\hspace{6pt}}cl@{\hspace{6pt}}c@{\hspace{6pt}}c@{\hspace{6pt}}c@{\hspace{6pt}}c@{\hspace{6pt}}c@{\hspace{6pt}}c}
\toprule
 & OG & CC & RE & KM & Ours & Oracle\\
\cmidrule{2-7}
Train acc. & $0.954 (0.017)$ & $0.737 (0.012)$ & $0.737 (0.028)$ & $0.737 (0.012)$ & $0.832 (0.064)$ & $0.727 (0.053)$\\
Val. acc. & NA & NA & $0.819 (0.065)$ & NA & NA & NA\\
Test acc. & $0.654 (0.061)$ & $\mathbf{0.728 (0.066)}$ & $0.700 (0.072)$ & $\mathbf{0.728 (0.066)}$ & $0.716 (0.076)$ & $0.810 (0.068)$\\
Leaves & $75.0 (0.0)$ & $1.0 (0.0)$ & $5.2 (4.5)$ & $1.0 (0.0)$ & $18.0 (23.5)$ & $5.5 (5.5)$\\
Height & $14.5 (1.9)$ & $0.0 (0.0)$ & $2.7 (2.6)$ & $0.0 (0.0)$ & $6.8 (4.1)$ & $2.7 (2.6)$\\
Time $[s]$ & $0.000 (0.000)$ & $4.262 (0.209)$ & $1.308 (1.588)$ & $5.462 (0.806)$ & $1.733 (1.071)$ & $0.991 (1.402)$\\
Bound & NA & NA & NA & NA & $7.189 (0.403)$ & NA\\
\bottomrule
\end{tabular}
\end{table}
\begin{table}[h!]
\centering
\caption{Statistics for the Heart Disease Cleveland Processed data set \citep{detrano1989international}. There are 303 examples, 13 features (6 real-valued, 0 ordinal and 7 nominal), and 5 classes.}
\vspace{6pt}
\small
\begin{tabular}{l@{\hspace{6pt}}c@{\hspace{6pt}}c@{\hspace{6pt}}c@{\hspace{6pt}}c@{\hspace{6pt}}c@{\hspace{6pt}}cl@{\hspace{6pt}}c@{\hspace{6pt}}c@{\hspace{6pt}}c@{\hspace{6pt}}c@{\hspace{6pt}}c@{\hspace{6pt}}cl@{\hspace{6pt}}c@{\hspace{6pt}}c@{\hspace{6pt}}c@{\hspace{6pt}}c@{\hspace{6pt}}c@{\hspace{6pt}}cl@{\hspace{6pt}}c@{\hspace{6pt}}c@{\hspace{6pt}}c@{\hspace{6pt}}c@{\hspace{6pt}}c@{\hspace{6pt}}cl@{\hspace{6pt}}c@{\hspace{6pt}}c@{\hspace{6pt}}c@{\hspace{6pt}}c@{\hspace{6pt}}c@{\hspace{6pt}}cl@{\hspace{6pt}}c@{\hspace{6pt}}c@{\hspace{6pt}}c@{\hspace{6pt}}c@{\hspace{6pt}}c@{\hspace{6pt}}cl@{\hspace{6pt}}c@{\hspace{6pt}}c@{\hspace{6pt}}c@{\hspace{6pt}}c@{\hspace{6pt}}c@{\hspace{6pt}}c}
\toprule
 & OG & CC & RE & KM & Ours & Oracle\\
\cmidrule{2-7}
Train acc. & $0.964 (0.018)$ & $0.568 (0.079)$ & $0.601 (0.053)$ & $0.547 (0.017)$ & $0.723 (0.041)$ & $0.577 (0.043)$\\
Val. acc. & NA & NA & $0.660 (0.077)$ & NA & NA & NA\\
Test acc. & $0.468 (0.064)$ & $0.507 (0.099)$ & $0.500 (0.083)$ & $0.511 (0.099)$ & $\mathbf{0.515 (0.081)}$ & $0.638 (0.089)$\\
Leaves & $75.0 (0.0)$ & $3.7 (10.6)$ & $7.5 (4.0)$ & $1.0 (0.0)$ & $14.2 (4.6)$ & $6.6 (4.4)$\\
Height & $11.0 (0.9)$ & $0.7 (2.5)$ & $4.0 (2.1)$ & $0.0 (0.0)$ & $6.9 (1.2)$ & $3.2 (2.3)$\\
Time $[s]$ & $0.000 (0.000)$ & $4.081 (0.146)$ & $0.639 (0.492)$ & $5.860 (0.604)$ & $6.192 (1.946)$ & $0.696 (0.531)$\\
Bound & NA & NA & NA & NA & $11.133 (0.606)$ & NA\\
\bottomrule
\end{tabular}
\end{table}
\begin{table}[h!]
\centering
\caption{Statistics for the Image Segmentation data set . There are 210 examples, 19 features (19 real-valued, 0 ordinal and 0 nominal), and 7 classes.}
\vspace{6pt}
\small
\begin{tabular}{l@{\hspace{6pt}}c@{\hspace{6pt}}c@{\hspace{6pt}}c@{\hspace{6pt}}c@{\hspace{6pt}}c@{\hspace{6pt}}cl@{\hspace{6pt}}c@{\hspace{6pt}}c@{\hspace{6pt}}c@{\hspace{6pt}}c@{\hspace{6pt}}c@{\hspace{6pt}}cl@{\hspace{6pt}}c@{\hspace{6pt}}c@{\hspace{6pt}}c@{\hspace{6pt}}c@{\hspace{6pt}}c@{\hspace{6pt}}cl@{\hspace{6pt}}c@{\hspace{6pt}}c@{\hspace{6pt}}c@{\hspace{6pt}}c@{\hspace{6pt}}c@{\hspace{6pt}}cl@{\hspace{6pt}}c@{\hspace{6pt}}c@{\hspace{6pt}}c@{\hspace{6pt}}c@{\hspace{6pt}}c@{\hspace{6pt}}cl@{\hspace{6pt}}c@{\hspace{6pt}}c@{\hspace{6pt}}c@{\hspace{6pt}}c@{\hspace{6pt}}c@{\hspace{6pt}}cl@{\hspace{6pt}}c@{\hspace{6pt}}c@{\hspace{6pt}}c@{\hspace{6pt}}c@{\hspace{6pt}}c@{\hspace{6pt}}c}
\toprule
 & OG & CC & RE & KM & Ours & Oracle\\
\cmidrule{2-7}
Train acc. & $1.000 (0.000)$ & $0.947 (0.056)$ & $0.899 (0.048)$ & $0.996 (0.006)$ & $0.961 (0.012)$ & $0.889 (0.057)$\\
Val. acc. & NA & NA & $0.892 (0.068)$ & NA & NA & NA\\
Test acc. & $0.861 (0.059)$ & $0.840 (0.082)$ & $0.815 (0.074)$ & $\mathbf{0.864 (0.058)}$ & $0.856 (0.061)$ & $0.882 (0.051)$\\
Leaves & $18.4 (1.6)$ & $11.1 (3.2)$ & $8.5 (1.3)$ & $17.2 (1.8)$ & $10.6 (1.0)$ & $8.3 (1.3)$\\
Height & $9.9 (1.2)$ & $7.2 (1.5)$ & $6.4 (0.9)$ & $9.8 (1.2)$ & $7.3 (0.7)$ & $6.2 (0.8)$\\
Time $[s]$ & $0.000 (0.000)$ & $1.181 (0.047)$ & $0.059 (0.025)$ & $1.250 (0.092)$ & $0.294 (0.089)$ & $0.092 (0.030)$\\
Bound & NA & NA & NA & NA & $4.001 (0.227)$ & NA\\
\bottomrule
\end{tabular}
\end{table}
\begin{table}[h!]
\centering
\caption{Statistics for the Ionosphere data set \citep{sigillito1989classification}. There are 351 examples, 34 features (34 real-valued, 0 ordinal and 0 nominal), and 2 classes.}
\vspace{6pt}
\small
\begin{tabular}{l@{\hspace{6pt}}c@{\hspace{6pt}}c@{\hspace{6pt}}c@{\hspace{6pt}}c@{\hspace{6pt}}c@{\hspace{6pt}}cl@{\hspace{6pt}}c@{\hspace{6pt}}c@{\hspace{6pt}}c@{\hspace{6pt}}c@{\hspace{6pt}}c@{\hspace{6pt}}cl@{\hspace{6pt}}c@{\hspace{6pt}}c@{\hspace{6pt}}c@{\hspace{6pt}}c@{\hspace{6pt}}c@{\hspace{6pt}}cl@{\hspace{6pt}}c@{\hspace{6pt}}c@{\hspace{6pt}}c@{\hspace{6pt}}c@{\hspace{6pt}}c@{\hspace{6pt}}cl@{\hspace{6pt}}c@{\hspace{6pt}}c@{\hspace{6pt}}c@{\hspace{6pt}}c@{\hspace{6pt}}c@{\hspace{6pt}}cl@{\hspace{6pt}}c@{\hspace{6pt}}c@{\hspace{6pt}}c@{\hspace{6pt}}c@{\hspace{6pt}}c@{\hspace{6pt}}cl@{\hspace{6pt}}c@{\hspace{6pt}}c@{\hspace{6pt}}c@{\hspace{6pt}}c@{\hspace{6pt}}c@{\hspace{6pt}}c}
\toprule
 & OG & CC & RE & KM & Ours & Oracle\\
\cmidrule{2-7}
Train acc. & $1.000 (0.000)$ & $0.821 (0.149)$ & $0.920 (0.019)$ & $0.974 (0.015)$ & $0.961 (0.008)$ & $0.913 (0.022)$\\
Val. acc. & NA & NA & $0.949 (0.032)$ & NA & NA & NA\\
Test acc. & $0.883 (0.034)$ & $0.796 (0.119)$ & $0.882 (0.041)$ & $0.891 (0.035)$ & $\mathbf{0.896 (0.038)}$ & $0.934 (0.028)$\\
Leaves & $21.6 (1.9)$ & $5.3 (5.5)$ & $5.2 (1.9)$ & $12.7 (4.4)$ & $8.6 (1.5)$ & $5.4 (1.7)$\\
Height & $9.9 (1.7)$ & $2.8 (3.1)$ & $3.6 (0.9)$ & $7.2 (1.7)$ & $5.4 (0.8)$ & $3.8 (1.1)$\\
Time $[s]$ & $0.000 (0.000)$ & $2.734 (0.179)$ & $0.079 (0.033)$ & $1.647 (0.114)$ & $0.219 (0.270)$ & $0.126 (0.061)$\\
Bound & NA & NA & NA & NA & $2.505 (0.143)$ & NA\\
\bottomrule
\end{tabular}
\end{table}
\begin{table}[h!]
\centering
\caption{Statistics for the Iris data set \citep{fisher1936use}. There are 150 examples, 4 features (4 real-valued, 0 ordinal and 0 nominal), and 3 classes.}
\vspace{6pt}
\small
\begin{tabular}{l@{\hspace{6pt}}c@{\hspace{6pt}}c@{\hspace{6pt}}c@{\hspace{6pt}}c@{\hspace{6pt}}c@{\hspace{6pt}}cl@{\hspace{6pt}}c@{\hspace{6pt}}c@{\hspace{6pt}}c@{\hspace{6pt}}c@{\hspace{6pt}}c@{\hspace{6pt}}cl@{\hspace{6pt}}c@{\hspace{6pt}}c@{\hspace{6pt}}c@{\hspace{6pt}}c@{\hspace{6pt}}c@{\hspace{6pt}}cl@{\hspace{6pt}}c@{\hspace{6pt}}c@{\hspace{6pt}}c@{\hspace{6pt}}c@{\hspace{6pt}}c@{\hspace{6pt}}cl@{\hspace{6pt}}c@{\hspace{6pt}}c@{\hspace{6pt}}c@{\hspace{6pt}}c@{\hspace{6pt}}c@{\hspace{6pt}}cl@{\hspace{6pt}}c@{\hspace{6pt}}c@{\hspace{6pt}}c@{\hspace{6pt}}c@{\hspace{6pt}}c@{\hspace{6pt}}cl@{\hspace{6pt}}c@{\hspace{6pt}}c@{\hspace{6pt}}c@{\hspace{6pt}}c@{\hspace{6pt}}c@{\hspace{6pt}}c}
\toprule
 & OG & CC & RE & KM & Ours & Oracle\\
\cmidrule{2-7}
Train acc. & $1.000 (0.000)$ & $0.930 (0.116)$ & $0.957 (0.025)$ & $0.998 (0.006)$ & $0.980 (0.008)$ & $0.961 (0.011)$\\
Val. acc. & NA & NA & $0.975 (0.030)$ & NA & NA & NA\\
Test acc. & $0.929 (0.040)$ & $0.880 (0.121)$ & $0.927 (0.045)$ & $0.929 (0.040)$ & $\mathbf{0.936 (0.037)}$ & $0.960 (0.036)$\\
Leaves & $8.0 (0.9)$ & $4.2 (1.8)$ & $3.2 (0.4)$ & $7.6 (1.2)$ & $4.0 (0.6)$ & $3.2 (0.4)$\\
Height & $5.0 (0.9)$ & $3.0 (1.6)$ & $2.2 (0.4)$ & $4.9 (0.9)$ & $3.0 (0.6)$ & $2.2 (0.4)$\\
Time $[s]$ & $0.000 (0.000)$ & $0.350 (0.023)$ & $0.009 (0.004)$ & $0.306 (0.026)$ & $0.017 (0.005)$ & $0.013 (0.004)$\\
Bound & NA & NA & NA & NA & $1.688 (0.179)$ & NA\\
\bottomrule
\end{tabular}
\end{table}
\begin{table}[h!]
\centering
\caption{Statistics for the Mushroom data set \citep{lincoff1997field}. There are 8124 examples, 22 features (0 real-valued, 0 ordinal and 22 nominal), and 2 classes.}
\vspace{6pt}
\small
\begin{tabular}{l@{\hspace{6pt}}c@{\hspace{6pt}}c@{\hspace{6pt}}c@{\hspace{6pt}}c@{\hspace{6pt}}c@{\hspace{6pt}}cl@{\hspace{6pt}}c@{\hspace{6pt}}c@{\hspace{6pt}}c@{\hspace{6pt}}c@{\hspace{6pt}}c@{\hspace{6pt}}cl@{\hspace{6pt}}c@{\hspace{6pt}}c@{\hspace{6pt}}c@{\hspace{6pt}}c@{\hspace{6pt}}c@{\hspace{6pt}}cl@{\hspace{6pt}}c@{\hspace{6pt}}c@{\hspace{6pt}}c@{\hspace{6pt}}c@{\hspace{6pt}}c@{\hspace{6pt}}cl@{\hspace{6pt}}c@{\hspace{6pt}}c@{\hspace{6pt}}c@{\hspace{6pt}}c@{\hspace{6pt}}c@{\hspace{6pt}}cl@{\hspace{6pt}}c@{\hspace{6pt}}c@{\hspace{6pt}}c@{\hspace{6pt}}c@{\hspace{6pt}}c@{\hspace{6pt}}cl@{\hspace{6pt}}c@{\hspace{6pt}}c@{\hspace{6pt}}c@{\hspace{6pt}}c@{\hspace{6pt}}c@{\hspace{6pt}}c}
\toprule
 & OG & CC & RE & KM & Ours & Oracle\\
\cmidrule{2-7}
Train acc. & $1.000 (0.000)$ & $0.996 (0.002)$ & $1.000 (0.000)$ & $1.000 (0.000)$ & $1.000 (0.000)$ & $1.000 (0.001)$\\
Val. acc. & NA & NA & $1.000 (0.000)$ & NA & NA & NA\\
Test acc. & $\mathbf{1.000 (0.000)}$ & $0.996 (0.002)$ & $\mathbf{1.000 (0.001)}$ & $\mathbf{1.000 (0.000)}$ & $\mathbf{1.000 (0.000)}$ & $1.000 (0.000)$\\
Leaves & $10.0 (0.2)$ & $7.9 (1.1)$ & $8.0 (0.3)$ & $10.0 (0.2)$ & $8.0 (0.0)$ & $7.9 (0.3)$\\
Height & $5.0 (0.2)$ & $4.0 (0.2)$ & $5.0 (0.2)$ & $5.0 (0.2)$ & $5.0 (0.0)$ & $4.9 (0.3)$\\
Time $[s]$ & $0.000 (0.000)$ & $51.102 (1.097)$ & $0.098 (0.035)$ & $14.799 (0.654)$ & $0.701 (0.081)$ & $0.084 (0.012)$\\
Bound & NA & NA & NA & NA & $0.066 (0.000)$ & NA\\
\bottomrule
\end{tabular}
\end{table}
\begin{table}[h!]
\centering
\caption{Statistics for the Parkinson data set \citep{little2007exploiting}. There are 195 examples, 22 features (22 real-valued, 0 ordinal and 0 nominal), and 2 classes.}
\vspace{6pt}
\small
\begin{tabular}{l@{\hspace{6pt}}c@{\hspace{6pt}}c@{\hspace{6pt}}c@{\hspace{6pt}}c@{\hspace{6pt}}c@{\hspace{6pt}}cl@{\hspace{6pt}}c@{\hspace{6pt}}c@{\hspace{6pt}}c@{\hspace{6pt}}c@{\hspace{6pt}}c@{\hspace{6pt}}cl@{\hspace{6pt}}c@{\hspace{6pt}}c@{\hspace{6pt}}c@{\hspace{6pt}}c@{\hspace{6pt}}c@{\hspace{6pt}}cl@{\hspace{6pt}}c@{\hspace{6pt}}c@{\hspace{6pt}}c@{\hspace{6pt}}c@{\hspace{6pt}}c@{\hspace{6pt}}cl@{\hspace{6pt}}c@{\hspace{6pt}}c@{\hspace{6pt}}c@{\hspace{6pt}}c@{\hspace{6pt}}c@{\hspace{6pt}}cl@{\hspace{6pt}}c@{\hspace{6pt}}c@{\hspace{6pt}}c@{\hspace{6pt}}c@{\hspace{6pt}}c@{\hspace{6pt}}cl@{\hspace{6pt}}c@{\hspace{6pt}}c@{\hspace{6pt}}c@{\hspace{6pt}}c@{\hspace{6pt}}c@{\hspace{6pt}}c}
\toprule
 & OG & CC & RE & KM & Ours & Oracle\\
\cmidrule{2-7}
Train acc. & $1.000 (0.000)$ & $0.919 (0.068)$ & $0.878 (0.050)$ & $0.957 (0.086)$ & $0.968 (0.015)$ & $0.885 (0.058)$\\
Val. acc. & NA & NA & $0.901 (0.050)$ & NA & NA & NA\\
Test acc. & $0.850 (0.061)$ & $0.822 (0.069)$ & $0.837 (0.086)$ & $0.825 (0.080)$ & $\mathbf{0.852 (0.066)}$ & $0.908 (0.055)$\\
Leaves & $13.2 (1.9)$ & $4.8 (2.9)$ & $3.7 (1.4)$ & $9.6 (4.4)$ & $6.9 (1.3)$ & $4.1 (1.5)$\\
Height & $5.5 (0.7)$ & $2.6 (1.5)$ & $2.3 (1.0)$ & $4.3 (2.0)$ & $3.6 (0.6)$ & $2.5 (1.0)$\\
Time $[s]$ & $0.000 (0.000)$ & $0.825 (0.071)$ & $0.024 (0.013)$ & $0.602 (0.058)$ & $0.057 (0.027)$ & $0.037 (0.017)$\\
Bound & NA & NA & NA & NA & $2.665 (0.331)$ & NA\\
\bottomrule
\end{tabular}
\end{table}
\begin{table}[h!]
\centering
\caption{Statistics for the Planning Relax data set \citep{bhatt2012planning}. There are 182 examples, 12 features (12 real-valued, 0 ordinal and 0 nominal), and 2 classes.}
\vspace{6pt}
\small
\begin{tabular}{l@{\hspace{6pt}}c@{\hspace{6pt}}c@{\hspace{6pt}}c@{\hspace{6pt}}c@{\hspace{6pt}}c@{\hspace{6pt}}cl@{\hspace{6pt}}c@{\hspace{6pt}}c@{\hspace{6pt}}c@{\hspace{6pt}}c@{\hspace{6pt}}c@{\hspace{6pt}}cl@{\hspace{6pt}}c@{\hspace{6pt}}c@{\hspace{6pt}}c@{\hspace{6pt}}c@{\hspace{6pt}}c@{\hspace{6pt}}cl@{\hspace{6pt}}c@{\hspace{6pt}}c@{\hspace{6pt}}c@{\hspace{6pt}}c@{\hspace{6pt}}c@{\hspace{6pt}}cl@{\hspace{6pt}}c@{\hspace{6pt}}c@{\hspace{6pt}}c@{\hspace{6pt}}c@{\hspace{6pt}}c@{\hspace{6pt}}cl@{\hspace{6pt}}c@{\hspace{6pt}}c@{\hspace{6pt}}c@{\hspace{6pt}}c@{\hspace{6pt}}c@{\hspace{6pt}}cl@{\hspace{6pt}}c@{\hspace{6pt}}c@{\hspace{6pt}}c@{\hspace{6pt}}c@{\hspace{6pt}}c@{\hspace{6pt}}c}
\toprule
 & OG & CC & RE & KM & Ours & Oracle\\
\cmidrule{2-7}
Train acc. & $1.000 (0.000)$ & $0.769 (0.093)$ & $0.691 (0.064)$ & $0.712 (0.016)$ & $1.000 (0.000)$ & $0.679 (0.048)$\\
Val. acc. & NA & NA & $0.739 (0.059)$ & NA & NA & NA\\
Test acc. & $0.567 (0.085)$ & $0.684 (0.084)$ & $0.696 (0.105)$ & $\mathbf{0.726 (0.092)}$ & $0.570 (0.086)$ & $0.790 (0.075)$\\
Leaves & $32.2 (2.4)$ & $5.4 (8.4)$ & $2.4 (2.1)$ & $1.0 (0.0)$ & $32.0 (2.5)$ & $3.2 (2.8)$\\
Height & $12.7 (2.7)$ & $2.6 (4.6)$ & $1.2 (1.8)$ & $0.0 (0.0)$ & $12.5 (2.8)$ & $1.8 (2.0)$\\
Time $[s]$ & $0.000 (0.000)$ & $1.549 (0.101)$ & $0.057 (0.055)$ & $1.502 (0.133)$ & $0.152 (0.047)$ & $0.112 (0.128)$\\
Bound & NA & NA & NA & NA & $6.437 (0.059)$ & NA\\
\bottomrule
\end{tabular}
\end{table}
\begin{table}[h!]
\centering
\caption{Statistics for the Qsar Biodegradation data set \citep{mansouri2013quantitative}. There are 1055 examples, 41 features (41 real-valued, 0 ordinal and 0 nominal), and 2 classes.}
\vspace{6pt}
\small
\begin{tabular}{l@{\hspace{6pt}}c@{\hspace{6pt}}c@{\hspace{6pt}}c@{\hspace{6pt}}c@{\hspace{6pt}}c@{\hspace{6pt}}cl@{\hspace{6pt}}c@{\hspace{6pt}}c@{\hspace{6pt}}c@{\hspace{6pt}}c@{\hspace{6pt}}c@{\hspace{6pt}}cl@{\hspace{6pt}}c@{\hspace{6pt}}c@{\hspace{6pt}}c@{\hspace{6pt}}c@{\hspace{6pt}}c@{\hspace{6pt}}cl@{\hspace{6pt}}c@{\hspace{6pt}}c@{\hspace{6pt}}c@{\hspace{6pt}}c@{\hspace{6pt}}c@{\hspace{6pt}}cl@{\hspace{6pt}}c@{\hspace{6pt}}c@{\hspace{6pt}}c@{\hspace{6pt}}c@{\hspace{6pt}}c@{\hspace{6pt}}cl@{\hspace{6pt}}c@{\hspace{6pt}}c@{\hspace{6pt}}c@{\hspace{6pt}}c@{\hspace{6pt}}c@{\hspace{6pt}}cl@{\hspace{6pt}}c@{\hspace{6pt}}c@{\hspace{6pt}}c@{\hspace{6pt}}c@{\hspace{6pt}}c@{\hspace{6pt}}c}
\toprule
 & OG & CC & RE & KM & Ours & Oracle\\
\cmidrule{2-7}
Train acc. & $0.945 (0.012)$ & $0.832 (0.053)$ & $0.856 (0.021)$ & $0.903 (0.019)$ & $0.884 (0.019)$ & $0.844 (0.023)$\\
Val. acc. & NA & NA & $0.888 (0.025)$ & NA & NA & NA\\
Test acc. & $0.800 (0.030)$ & $0.790 (0.058)$ & $0.805 (0.030)$ & $0.813 (0.028)$ & $\mathbf{0.822 (0.025)}$ & $0.879 (0.021)$\\
Leaves & $75.0 (0.0)$ & $9.2 (10.5)$ & $15.9 (3.6)$ & $32.0 (10.5)$ & $17.0 (4.5)$ & $15.1 (3.7)$\\
Height & $13.8 (1.9)$ & $4.5 (2.9)$ & $6.9 (1.2)$ & $9.6 (1.4)$ & $6.6 (1.2)$ & $6.2 (1.0)$\\
Time $[s]$ & $0.000 (0.000)$ & $13.301 (0.496)$ & $2.854 (1.213)$ & $10.114 (0.797)$ & $3.160 (0.931)$ & $2.730 (1.269)$\\
Bound & NA & NA & NA & NA & $4.570 (0.367)$ & NA\\
\bottomrule
\end{tabular}
\end{table}
\begin{table}[h!]
\centering
\caption{Statistics for the Seeds data set \citep{charytanowicz2010complete}. There are 210 examples, 7 features (7 real-valued, 0 ordinal and 0 nominal), and 3 classes.}
\vspace{6pt}
\small
\begin{tabular}{l@{\hspace{6pt}}c@{\hspace{6pt}}c@{\hspace{6pt}}c@{\hspace{6pt}}c@{\hspace{6pt}}c@{\hspace{6pt}}cl@{\hspace{6pt}}c@{\hspace{6pt}}c@{\hspace{6pt}}c@{\hspace{6pt}}c@{\hspace{6pt}}c@{\hspace{6pt}}cl@{\hspace{6pt}}c@{\hspace{6pt}}c@{\hspace{6pt}}c@{\hspace{6pt}}c@{\hspace{6pt}}c@{\hspace{6pt}}cl@{\hspace{6pt}}c@{\hspace{6pt}}c@{\hspace{6pt}}c@{\hspace{6pt}}c@{\hspace{6pt}}c@{\hspace{6pt}}cl@{\hspace{6pt}}c@{\hspace{6pt}}c@{\hspace{6pt}}c@{\hspace{6pt}}c@{\hspace{6pt}}c@{\hspace{6pt}}cl@{\hspace{6pt}}c@{\hspace{6pt}}c@{\hspace{6pt}}c@{\hspace{6pt}}c@{\hspace{6pt}}c@{\hspace{6pt}}cl@{\hspace{6pt}}c@{\hspace{6pt}}c@{\hspace{6pt}}c@{\hspace{6pt}}c@{\hspace{6pt}}c@{\hspace{6pt}}c}
\toprule
 & OG & CC & RE & KM & Ours & Oracle\\
\cmidrule{2-7}
Train acc. & $1.000 (0.000)$ & $0.936 (0.086)$ & $0.918 (0.037)$ & $0.987 (0.013)$ & $0.968 (0.008)$ & $0.934 (0.018)$\\
Val. acc. & NA & NA & $0.968 (0.039)$ & NA & NA & NA\\
Test acc. & $\mathbf{0.911 (0.048)}$ & $0.877 (0.091)$ & $0.905 (0.041)$ & $0.909 (0.046)$ & $\mathbf{0.911 (0.046)}$ & $0.949 (0.031)$\\
Leaves & $13.7 (2.1)$ & $5.2 (2.2)$ & $4.1 (1.1)$ & $10.1 (3.0)$ & $5.6 (0.7)$ & $4.2 (0.8)$\\
Height & $6.4 (0.6)$ & $4.0 (1.8)$ & $3.0 (0.9)$ & $5.9 (1.0)$ & $4.4 (0.7)$ & $3.0 (0.7)$\\
Time $[s]$ & $0.000 (0.000)$ & $0.709 (0.051)$ & $0.033 (0.014)$ & $0.578 (0.044)$ & $0.063 (0.019)$ & $0.040 (0.018)$\\
Bound & NA & NA & NA & NA & $2.206 (0.156)$ & NA\\
\bottomrule
\end{tabular}
\end{table}
\begin{table}[h!]
\centering
\caption{Statistics for the Spambase data set . There are 4601 examples, 57 features (57 real-valued, 0 ordinal and 0 nominal), and 2 classes.}
\vspace{6pt}
\small
\begin{tabular}{l@{\hspace{6pt}}c@{\hspace{6pt}}c@{\hspace{6pt}}c@{\hspace{6pt}}c@{\hspace{6pt}}c@{\hspace{6pt}}cl@{\hspace{6pt}}c@{\hspace{6pt}}c@{\hspace{6pt}}c@{\hspace{6pt}}c@{\hspace{6pt}}c@{\hspace{6pt}}cl@{\hspace{6pt}}c@{\hspace{6pt}}c@{\hspace{6pt}}c@{\hspace{6pt}}c@{\hspace{6pt}}c@{\hspace{6pt}}cl@{\hspace{6pt}}c@{\hspace{6pt}}c@{\hspace{6pt}}c@{\hspace{6pt}}c@{\hspace{6pt}}c@{\hspace{6pt}}cl@{\hspace{6pt}}c@{\hspace{6pt}}c@{\hspace{6pt}}c@{\hspace{6pt}}c@{\hspace{6pt}}c@{\hspace{6pt}}cl@{\hspace{6pt}}c@{\hspace{6pt}}c@{\hspace{6pt}}c@{\hspace{6pt}}c@{\hspace{6pt}}c@{\hspace{6pt}}cl@{\hspace{6pt}}c@{\hspace{6pt}}c@{\hspace{6pt}}c@{\hspace{6pt}}c@{\hspace{6pt}}c@{\hspace{6pt}}c}
\toprule
 & OG & CC & RE & KM & Ours & Oracle\\
\cmidrule{2-7}
Train acc. & $0.884 (0.034)$ & $0.856 (0.042)$ & $0.885 (0.017)$ & $0.878 (0.034)$ & $0.872 (0.034)$ & $0.880 (0.024)$\\
Val. acc. & NA & NA & $0.889 (0.020)$ & NA & NA & NA\\
Test acc. & $0.858 (0.036)$ & $0.847 (0.045)$ & $\mathbf{0.882 (0.018)}$ & $0.859 (0.036)$ & $0.861 (0.036)$ & $0.885 (0.022)$\\
Leaves & $75.0 (0.0)$ & $7.8 (6.3)$ & $14.6 (4.1)$ & $42.4 (22.1)$ & $11.5 (4.2)$ & $13.2 (4.2)$\\
Height & $23.6 (4.1)$ & $4.0 (2.5)$ & $7.8 (2.8)$ & $16.7 (7.5)$ & $5.1 (1.3)$ & $7.4 (2.2)$\\
Time $[s]$ & $0.000 (0.000)$ & $76.087 (7.525)$ & $8.851 (4.972)$ & $39.681 (3.144)$ & $7.480 (2.428)$ & $8.196 (5.506)$\\
Bound & NA & NA & NA & NA & $4.151 (1.016)$ & NA\\
\bottomrule
\end{tabular}
\end{table}
\begin{table}[h!]
\centering
\caption{Statistics for the Statlog German data set \citep{hofmann94statloggerman}. There are 1000 examples, 20 features (6 real-valued, 2 ordinal and 12 nominal), and 2 classes.}
\vspace{6pt}
\small
\begin{tabular}{l@{\hspace{6pt}}c@{\hspace{6pt}}c@{\hspace{6pt}}c@{\hspace{6pt}}c@{\hspace{6pt}}c@{\hspace{6pt}}cl@{\hspace{6pt}}c@{\hspace{6pt}}c@{\hspace{6pt}}c@{\hspace{6pt}}c@{\hspace{6pt}}c@{\hspace{6pt}}cl@{\hspace{6pt}}c@{\hspace{6pt}}c@{\hspace{6pt}}c@{\hspace{6pt}}c@{\hspace{6pt}}c@{\hspace{6pt}}cl@{\hspace{6pt}}c@{\hspace{6pt}}c@{\hspace{6pt}}c@{\hspace{6pt}}c@{\hspace{6pt}}c@{\hspace{6pt}}cl@{\hspace{6pt}}c@{\hspace{6pt}}c@{\hspace{6pt}}c@{\hspace{6pt}}c@{\hspace{6pt}}c@{\hspace{6pt}}cl@{\hspace{6pt}}c@{\hspace{6pt}}c@{\hspace{6pt}}c@{\hspace{6pt}}c@{\hspace{6pt}}c@{\hspace{6pt}}cl@{\hspace{6pt}}c@{\hspace{6pt}}c@{\hspace{6pt}}c@{\hspace{6pt}}c@{\hspace{6pt}}c@{\hspace{6pt}}c}
\toprule
 & OG & CC & RE & KM & Ours & Oracle\\
\cmidrule{2-7}
Train acc. & $0.812 (0.014)$ & $0.700 (0.006)$ & $0.712 (0.023)$ & $0.705 (0.027)$ & $0.741 (0.021)$ & $0.707 (0.019)$\\
Val. acc. & NA & NA & $0.758 (0.044)$ & NA & NA & NA\\
Test acc. & $0.655 (0.044)$ & $\mathbf{0.702 (0.032)}$ & $0.680 (0.040)$ & $\mathbf{0.702 (0.031)}$ & $0.697 (0.029)$ & $0.740 (0.032)$\\
Leaves & $75.0 (0.0)$ & $1.0 (0.0)$ & $8.6 (5.4)$ & $2.8 (8.8)$ & $6.2 (3.5)$ & $5.1 (5.1)$\\
Height & $14.1 (1.4)$ & $0.0 (0.0)$ & $4.8 (3.2)$ & $0.4 (2.2)$ & $4.4 (2.6)$ & $2.8 (3.1)$\\
Time $[s]$ & $0.000 (0.000)$ & $8.820 (0.284)$ & $1.175 (1.292)$ & $8.484 (0.746)$ & $1.441 (0.524)$ & $0.625 (0.858)$\\
Bound & NA & NA & NA & NA & $8.395 (0.489)$ & NA\\
\bottomrule
\end{tabular}
\end{table}
\begin{table}[h!]
\centering
\caption{Statistics for the Vertebral Column3 C data set \citep{berthonnaud2005analysis}. There are 310 examples, 6 features (6 real-valued, 0 ordinal and 0 nominal), and 3 classes.}
\vspace{6pt}
\small
\begin{tabular}{l@{\hspace{6pt}}c@{\hspace{6pt}}c@{\hspace{6pt}}c@{\hspace{6pt}}c@{\hspace{6pt}}c@{\hspace{6pt}}cl@{\hspace{6pt}}c@{\hspace{6pt}}c@{\hspace{6pt}}c@{\hspace{6pt}}c@{\hspace{6pt}}c@{\hspace{6pt}}cl@{\hspace{6pt}}c@{\hspace{6pt}}c@{\hspace{6pt}}c@{\hspace{6pt}}c@{\hspace{6pt}}c@{\hspace{6pt}}cl@{\hspace{6pt}}c@{\hspace{6pt}}c@{\hspace{6pt}}c@{\hspace{6pt}}c@{\hspace{6pt}}c@{\hspace{6pt}}cl@{\hspace{6pt}}c@{\hspace{6pt}}c@{\hspace{6pt}}c@{\hspace{6pt}}c@{\hspace{6pt}}c@{\hspace{6pt}}cl@{\hspace{6pt}}c@{\hspace{6pt}}c@{\hspace{6pt}}c@{\hspace{6pt}}c@{\hspace{6pt}}c@{\hspace{6pt}}cl@{\hspace{6pt}}c@{\hspace{6pt}}c@{\hspace{6pt}}c@{\hspace{6pt}}c@{\hspace{6pt}}c@{\hspace{6pt}}c}
\toprule
 & OG & CC & RE & KM & Ours & Oracle\\
\cmidrule{2-7}
Train acc. & $1.000 (0.000)$ & $0.779 (0.167)$ & $0.843 (0.046)$ & $0.987 (0.009)$ & $0.923 (0.016)$ & $0.820 (0.038)$\\
Val. acc. & NA & NA & $0.893 (0.051)$ & NA & NA & NA\\
Test acc. & $0.775 (0.077)$ & $0.730 (0.137)$ & $0.793 (0.082)$ & $0.783 (0.076)$ & $\mathbf{0.806 (0.066)}$ & $0.878 (0.047)$\\
Leaves & $35.2 (3.5)$ & $5.8 (7.4)$ & $5.9 (2.3)$ & $29.0 (4.3)$ & $10.2 (2.3)$ & $5.2 (2.1)$\\
Height & $10.0 (1.5)$ & $2.8 (3.2)$ & $4.0 (1.6)$ & $9.9 (1.6)$ & $6.0 (1.2)$ & $3.5 (1.3)$\\
Time $[s]$ & $0.000 (0.000)$ & $1.624 (0.129)$ & $0.211 (0.118)$ & $1.650 (0.191)$ & $0.640 (0.265)$ & $0.280 (0.192)$\\
Bound & NA & NA & NA & NA & $3.949 (0.259)$ & NA\\
\bottomrule
\end{tabular}
\end{table}
\begin{table}[h!]
\centering
\caption{Statistics for the Wall Following Robot24 data set \citep{freire2009short}. There are 5456 examples, 24 features (24 real-valued, 0 ordinal and 0 nominal), and 4 classes.}
\vspace{6pt}
\small
\begin{tabular}{l@{\hspace{6pt}}c@{\hspace{6pt}}c@{\hspace{6pt}}c@{\hspace{6pt}}c@{\hspace{6pt}}c@{\hspace{6pt}}cl@{\hspace{6pt}}c@{\hspace{6pt}}c@{\hspace{6pt}}c@{\hspace{6pt}}c@{\hspace{6pt}}c@{\hspace{6pt}}cl@{\hspace{6pt}}c@{\hspace{6pt}}c@{\hspace{6pt}}c@{\hspace{6pt}}c@{\hspace{6pt}}c@{\hspace{6pt}}cl@{\hspace{6pt}}c@{\hspace{6pt}}c@{\hspace{6pt}}c@{\hspace{6pt}}c@{\hspace{6pt}}c@{\hspace{6pt}}cl@{\hspace{6pt}}c@{\hspace{6pt}}c@{\hspace{6pt}}c@{\hspace{6pt}}c@{\hspace{6pt}}c@{\hspace{6pt}}cl@{\hspace{6pt}}c@{\hspace{6pt}}c@{\hspace{6pt}}c@{\hspace{6pt}}c@{\hspace{6pt}}c@{\hspace{6pt}}cl@{\hspace{6pt}}c@{\hspace{6pt}}c@{\hspace{6pt}}c@{\hspace{6pt}}c@{\hspace{6pt}}c@{\hspace{6pt}}c}
\toprule
 & OG & CC & RE & KM & Ours & Oracle\\
\cmidrule{2-7}
Train acc. & $1.000 (0.000)$ & $0.999 (0.001)$ & $0.996 (0.002)$ & $1.000 (0.000)$ & $0.997 (0.001)$ & $0.996 (0.002)$\\
Val. acc. & NA & NA & $0.996 (0.002)$ & NA & NA & NA\\
Test acc. & $\mathbf{0.995 (0.003)}$ & $\mathbf{0.995 (0.003)}$ & $0.993 (0.005)$ & $\mathbf{0.995 (0.003)}$ & $0.994 (0.003)$ & $0.996 (0.002)$\\
Leaves & $29.8 (2.6)$ & $24.6 (4.0)$ & $17.2 (2.1)$ & $29.5 (2.5)$ & $18.0 (1.5)$ & $17.6 (2.3)$\\
Height & $9.5 (0.9)$ & $9.3 (1.0)$ & $8.0 (0.8)$ & $9.5 (0.9)$ & $8.9 (0.9)$ & $8.2 (1.0)$\\
Time $[s]$ & $0.000 (0.000)$ & $35.002 (1.172)$ & $1.352 (0.461)$ & $16.744 (0.516)$ & $6.697 (1.565)$ & $1.568 (0.420)$\\
Bound & NA & NA & NA & NA & $0.307 (0.020)$ & NA\\
\bottomrule
\end{tabular}
\end{table}
\begin{table}[h!]
\centering
\caption{Statistics for the Wine data set \citep{aeberhard1994comparative}. There are 178 examples, 13 features (13 real-valued, 0 ordinal and 0 nominal), and 3 classes.}
\vspace{6pt}
\small
\begin{tabular}{l@{\hspace{6pt}}c@{\hspace{6pt}}c@{\hspace{6pt}}c@{\hspace{6pt}}c@{\hspace{6pt}}c@{\hspace{6pt}}cl@{\hspace{6pt}}c@{\hspace{6pt}}c@{\hspace{6pt}}c@{\hspace{6pt}}c@{\hspace{6pt}}c@{\hspace{6pt}}cl@{\hspace{6pt}}c@{\hspace{6pt}}c@{\hspace{6pt}}c@{\hspace{6pt}}c@{\hspace{6pt}}c@{\hspace{6pt}}cl@{\hspace{6pt}}c@{\hspace{6pt}}c@{\hspace{6pt}}c@{\hspace{6pt}}c@{\hspace{6pt}}c@{\hspace{6pt}}cl@{\hspace{6pt}}c@{\hspace{6pt}}c@{\hspace{6pt}}c@{\hspace{6pt}}c@{\hspace{6pt}}c@{\hspace{6pt}}cl@{\hspace{6pt}}c@{\hspace{6pt}}c@{\hspace{6pt}}c@{\hspace{6pt}}c@{\hspace{6pt}}c@{\hspace{6pt}}cl@{\hspace{6pt}}c@{\hspace{6pt}}c@{\hspace{6pt}}c@{\hspace{6pt}}c@{\hspace{6pt}}c@{\hspace{6pt}}c}
\toprule
 & OG & CC & RE & KM & Ours & Oracle\\
\cmidrule{2-7}
Train acc. & $1.000 (0.000)$ & $0.977 (0.012)$ & $0.933 (0.035)$ & $1.000 (0.000)$ & $0.977 (0.009)$ & $0.922 (0.045)$\\
Val. acc. & NA & NA & $0.910 (0.067)$ & NA & NA & NA\\
Test acc. & $\mathbf{0.924 (0.039)}$ & $0.904 (0.040)$ & $0.859 (0.084)$ & $\mathbf{0.924 (0.039)}$ & $0.911 (0.035)$ & $0.945 (0.030)$\\
Leaves & $9.9 (1.9)$ & $6.3 (1.5)$ & $4.2 (1.0)$ & $9.9 (1.9)$ & $6.0 (1.1)$ & $4.8 (1.0)$\\
Height & $4.8 (0.7)$ & $3.2 (0.5)$ & $2.7 (0.8)$ & $4.8 (0.7)$ & $3.1 (0.3)$ & $2.8 (0.7)$\\
Time $[s]$ & $0.000 (0.000)$ & $0.567 (0.048)$ & $0.013 (0.008)$ & $0.405 (0.036)$ & $0.030 (0.011)$ & $0.018 (0.007)$\\
Bound & NA & NA & NA & NA & $2.264 (0.382)$ & NA\\
\bottomrule
\end{tabular}
\end{table}
\begin{table}[h!]
\centering
\caption{Statistics for the Yeast data set \citep{horton1996probabilistic}. There are 1484 examples, 8 features (8 real-valued, 0 ordinal and 0 nominal), and 10 classes.}
\vspace{6pt}
\small
\begin{tabular}{l@{\hspace{6pt}}c@{\hspace{6pt}}c@{\hspace{6pt}}c@{\hspace{6pt}}c@{\hspace{6pt}}c@{\hspace{6pt}}cl@{\hspace{6pt}}c@{\hspace{6pt}}c@{\hspace{6pt}}c@{\hspace{6pt}}c@{\hspace{6pt}}c@{\hspace{6pt}}cl@{\hspace{6pt}}c@{\hspace{6pt}}c@{\hspace{6pt}}c@{\hspace{6pt}}c@{\hspace{6pt}}c@{\hspace{6pt}}cl@{\hspace{6pt}}c@{\hspace{6pt}}c@{\hspace{6pt}}c@{\hspace{6pt}}c@{\hspace{6pt}}c@{\hspace{6pt}}cl@{\hspace{6pt}}c@{\hspace{6pt}}c@{\hspace{6pt}}c@{\hspace{6pt}}c@{\hspace{6pt}}c@{\hspace{6pt}}cl@{\hspace{6pt}}c@{\hspace{6pt}}c@{\hspace{6pt}}c@{\hspace{6pt}}c@{\hspace{6pt}}c@{\hspace{6pt}}cl@{\hspace{6pt}}c@{\hspace{6pt}}c@{\hspace{6pt}}c@{\hspace{6pt}}c@{\hspace{6pt}}c@{\hspace{6pt}}c}
\toprule
 & OG & CC & RE & KM & Ours & Oracle\\
\cmidrule{2-7}
Train acc. & $0.546 (0.031)$ & $0.360 (0.075)$ & $0.512 (0.012)$ & $0.525 (0.029)$ & $0.500 (0.034)$ & $0.491 (0.033)$\\
Val. acc. & NA & NA & $0.506 (0.040)$ & NA & NA & NA\\
Test acc. & $0.457 (0.043)$ & $0.352 (0.069)$ & $\mathbf{0.490 (0.033)}$ & $0.469 (0.043)$ & $0.478 (0.040)$ & $0.497 (0.042)$\\
Leaves & $75.0 (0.0)$ & $1.9 (1.7)$ & $9.2 (3.2)$ & $36.2 (9.6)$ & $8.0 (2.7)$ & $10.4 (3.7)$\\
Height & $13.8 (1.3)$ & $0.8 (1.3)$ & $4.8 (1.2)$ & $11.2 (1.8)$ & $4.4 (1.4)$ & $5.3 (1.2)$\\
Time $[s]$ & $0.000 (0.000)$ & $8.115 (0.376)$ & $1.601 (1.199)$ & $14.049 (0.875)$ & $17.647 (2.780)$ & $2.036 (1.432)$\\
Bound & NA & NA & NA & NA & $15.869 (0.967)$ & NA\\
\bottomrule
\end{tabular}
\end{table}
\begin{table}[h!]
\centering
\caption{Statistics for the Zoo data set \citep{forsyth90zoo}. There are 101 examples, 16 features (0 real-valued, 1 ordinal and 15 nominal), and 7 classes.}
\vspace{6pt}
\small
\begin{tabular}{l@{\hspace{6pt}}c@{\hspace{6pt}}c@{\hspace{6pt}}c@{\hspace{6pt}}c@{\hspace{6pt}}c@{\hspace{6pt}}cl@{\hspace{6pt}}c@{\hspace{6pt}}c@{\hspace{6pt}}c@{\hspace{6pt}}c@{\hspace{6pt}}c@{\hspace{6pt}}cl@{\hspace{6pt}}c@{\hspace{6pt}}c@{\hspace{6pt}}c@{\hspace{6pt}}c@{\hspace{6pt}}c@{\hspace{6pt}}cl@{\hspace{6pt}}c@{\hspace{6pt}}c@{\hspace{6pt}}c@{\hspace{6pt}}c@{\hspace{6pt}}c@{\hspace{6pt}}cl@{\hspace{6pt}}c@{\hspace{6pt}}c@{\hspace{6pt}}c@{\hspace{6pt}}c@{\hspace{6pt}}c@{\hspace{6pt}}cl@{\hspace{6pt}}c@{\hspace{6pt}}c@{\hspace{6pt}}c@{\hspace{6pt}}c@{\hspace{6pt}}c@{\hspace{6pt}}cl@{\hspace{6pt}}c@{\hspace{6pt}}c@{\hspace{6pt}}c@{\hspace{6pt}}c@{\hspace{6pt}}c@{\hspace{6pt}}c}
\toprule
 & OG & CC & RE & KM & Ours & Oracle\\
\cmidrule{2-7}
Train acc. & $1.000 (0.000)$ & $0.981 (0.009)$ & $0.897 (0.054)$ & $1.000 (0.000)$ & $0.977 (0.008)$ & $0.856 (0.077)$\\
Val. acc. & NA & NA & $0.955 (0.057)$ & NA & NA & NA\\
Test acc. & $\mathbf{0.949 (0.048)}$ & $0.928 (0.051)$ & $0.864 (0.073)$ & $\mathbf{0.949 (0.048)}$ & $0.928 (0.051)$ & $0.965 (0.044)$\\
Leaves & $9.5 (0.7)$ & $7.5 (0.6)$ & $5.6 (1.0)$ & $9.5 (0.7)$ & $7.2 (0.4)$ & $5.5 (0.9)$\\
Height & $6.8 (0.4)$ & $5.7 (0.7)$ & $4.4 (0.8)$ & $6.8 (0.4)$ & $5.4 (0.5)$ & $4.3 (0.7)$\\
Time $[s]$ & $0.000 (0.000)$ & $0.374 (0.015)$ & $0.011 (0.003)$ & $0.412 (0.064)$ & $0.022 (0.007)$ & $0.015 (0.003)$\\
Bound & NA & NA & NA & NA & $4.086 (0.169)$ & NA\\
\bottomrule
\end{tabular}
\end{table}

\clearpage

\bibliography{bibliography}

\end{document}